\documentclass[11pt]{article}

\usepackage{amsmath,amsfonts,bm}

\def\1{\bm{1}}

\DeclareMathAlphabet{\mathsfit}{\encodingdefault}{\sfdefault}{m}{sl}
\SetMathAlphabet{\mathsfit}{bold}{\encodingdefault}{\sfdefault}{bx}{n}

\def\gD{{\mathcal{D}}}

\def\gF{{\mathcal{F}}}

\def\gL{{\mathcal{L}}}

\def\gP{{\mathcal{P}}}

\def\gR{{\mathcal{R}}}

\def\gY{{\mathcal{Y}}}

\newcommand{\E}{\mathbb{E}}

\newcommand{\R}{\mathbb{R}}

\newcommand{\Var}{\mathrm{Var}}

\newcommand{\Cov}{\mathrm{Cov}}

\DeclareMathOperator*{\argmin}{arg\,min}

\def\P {\mathbb{P}}
\def\E {\mathbb{E}}

\usepackage{a4wide}
\usepackage[a4paper,top=3cm,bottom=3cm,left=3cm,right=3cm,marginparwidth=1.75cm]{geometry}
\usepackage[utf8]{inputenc} 
\usepackage[T1]{fontenc}    
\usepackage{babel}
\usepackage{url}            
\usepackage{booktabs}       
\usepackage{amsfonts}       
\usepackage{nicefrac}       
\usepackage{lipsum}		    
\usepackage{graphicx}
\usepackage{doi}
\usepackage{physics}
\usepackage{color}
\usepackage{soul}
\usepackage{multirow}
\usepackage{algorithm}
\usepackage{algorithmicx}
\usepackage{algpseudocode}
\usepackage{caption}
\usepackage{amsmath,amssymb,amsthm}
\usepackage{bbm}
\usepackage{xr}
\usepackage{blindtext}
\usepackage{caption}
\usepackage{hyperref}       
\usepackage[title]{appendix}
\usepackage{enumerate}
\usepackage{array}
\usepackage{authblk}
\usepackage{microtype}      

\newtheorem{thm}{Theorem}[section]
\newtheorem{prop}{Proposition}[section]
\newtheorem{asm}{Assumption}[section]
\newtheorem{lem}{Lemma}[section]
\newtheorem{cor}{Corollary}[section]
\theoremstyle{definition}
\newtheorem{dfn}{Definition}[section]
\theoremstyle{remark}
\newtheorem{rem}{Remark}[section] 
\numberwithin{equation}{section}
\newcommand{\method}{PCAL}

\def\P{\mathbb{P}}
\def\E{\mathbb{E}}
\def\R{\mathbb{R}}

\def\0{{(0)}}
\def\1{{(1)}}
\def\2{{(2)}}
\def\3{{(3)}}
\def\k{{(k)}}
\def\sp{\mathrm{sp}}
\newcommand{\indep}{\perp \!\!\! \perp}
\def\sup{{\textnormal{sup}}}

\renewcommand{\hat}{\widehat}
\renewcommand{\tilde}{\widetilde}

\usepackage{xcolor}

\usepackage[round]{natbib}

\title{Labels or Preferences? Budget-Constrained Learning with Human Judgments over AI-Generated Outputs}

\makeatletter
\patchcmd{\AB@output}{\par\vspace{1em}}{}{}{}
\renewcommand\AB@affilsepx{\quad}  
\makeatother
\author[1]{Zihan Dong}
\author[2]{Xiaotian Hou}
\author[3]{Ruijia Wu\thanks{Co-corresponding author}}
\author[1]{Linjun Zhang\protect\footnotemark[1]}

\affil[1]{Rutgers University}
\affil[2]{University of Pennsylvania}
\affil[3]{Shanghai Jiao Tong University}

\begin{document}
\maketitle
\begin{abstract}
The increasing reliance on human preference feedback to judge AI-generated pseudo labels has created a pressing need for principled, budget-conscious data acquisition strategies. We address the crucial question of how to optimally allocate a fixed annotation budget between ground-truth labels and pairwise preferences in AI. Our solution, grounded in semi-parametric inference, casts the budget allocation problem as a monotone missing data framework. Building on this formulation, we introduce Preference-Calibrated Active Learning (\method), a novel method that learns the optimal data acquisition strategy and develops a statistically efficient estimator for functionals of the data distribution. Theoretically, we prove the asymptotic optimality of our \method~estimator and establish a key robustness guarantee that ensures robust performance even with poorly estimated nuisance models.
Our flexible framework applies to a general class of problems, by directly optimizing the estimator's variance instead of requiring a closed-form solution.
This work provides a principled and statistically efficient approach for budget-constrained learning in modern AI. Simulations and real-data analysis demonstrate the practical benefits and superior performance of our proposed method.
\footnote{Code available at: \url{https://github.com/zihandong02/SD_auditor}}
\end{abstract}



\section{Introduction}
Recent advances in Artificial Intelligence (AI), particularly in training Large Language Models (LLMs), increasingly rely on combining two complementary sources of supervision: ground-truth labels and human preference data \citep{azar2024general,rafailov2023direct, hong2024orpo, ethayarajh2024kto, meng2024simpo}. While classical supervised learning relies on data with definitive ground-truth labels, state-of-the-art models are often refined using human preferences. This preference data is gathered by presenting human annotators with two candidate outputs (often generated by AI models) and asking them which one is better \citep{christiano2017deep}.

Preference-labeled data arise naturally in many real-world settings. In particular, one may have access to pseudo-labels, cheap proxy labels that approximate the truth but can be systematically biased—and the annotator is asked to compare two pseudo-labels and identify which one is closer to the ground truth. While acquiring high-quality ground-truth labels can be expensive \citep{zajac2023ground}, preference-labeled data collected through pairwise comparisons provides a more accessible and efficient alternative. This approach is particularly useful for recalibrating when initial pseudo-labels exhibit bias \citep{bender2021dangers, gehman2020realtoxicityprompts}, as it provides a direct way to align model outputs with human judgment \citep{leike2018scalable}. {A particularly prominent application} of pairwise comparison is in fine-tuning large language models(LLMs), where it is used to collect human feedback to align models with user preferences \citep{ouyang2022traininglanguagemodelsfollow, christiano2017deep,rafailov2023direct}.  Beyond LLMs, pairwise comparison are also widely used in the other fields, including ranking search results \citep{burges2005learning, lee2024methods}, evaluating text generation quality \citep{stiennon2020learning}. 
Another compelling application lies in clinical decision support, particularly in cases where obtaining precise, direct assessments is often challenging. For a clinician, it is often simpler to make a comparative judgment, such as identifying which of two potential diagnoses is more accurate or which treatment is preferable, than to assign a quantitative score to a single option \citep{verbeeck2023generalized}. {Pairwise comparison also plays a crucial role in algorithmic fairness and metric learning, where pairwise judgments enable learning of task-specific similarity metrics for individual fairness \citep{ilvento2019metric} and provide a framework for pairwise fairness constraints in ranking and regression \citep{narasimhan2020pairwise}.
}

Viewed through this lens, pairwise preference labels provide a natural middle tier between pseudo-labels and ground truth: they avoid the cost of collecting exact labels while conveying substantially more information than pseudo-labels alone. When multiple supervision types are available under a fixed annotation budget, a central question is how to allocate resources across data types in order to optimize statistical efficiency:
\begin{center}
\textit{Given data types with distinct costs, what is the optimal budget allocation, and what is the most efficient estimator for the resulting data mix?}
\end{center}

Our answer reveals that a principled mixture of both data types is optimal. On the one hand, ground-truth labels are essential: without them (or when their proportion is negligible), the estimator becomes inconsistent. On the other hand, whether to include preference data depends on comparing its information gain relative to its cost (see Corollary \ref{cor: need preference data} for the formal condition). Our analysis yields an automatic and principled recommendation for the most cost-effective data acquisition strategy.

We address this problem from a semi-parametric perspective, formulating the goal as the estimation of a functional of the underlying data distribution. In this framework, efficiency is characterized through asymptotic variance, and the goal is to construct estimators and data collection strategies that minimize this variance subject to the budget constraint. Our contributions are as follows: 
\begin{itemize}
    \item To the best of our knowledge, we are the first to provide a theoretical model to study the problem of learning from pairwise preference labels and pseudo-labels under a budget constraint, and provide the optimal analysis. Our formulation models this problem as a general monotone missing data problem under the Missing At Random (MAR) case, providing a rigorous theoretical foundation for analysis.
    \item Secondly, within this framework, we propose Preference-Calibrated Active Learning (\method), which jointly designs an optimal active sampling policy for acquiring labels under a budget constraint and constructs an efficient estimator of the target parameter using the resulting mixed supervision. 
    \item Third, we derive the semiparametric efficiency bound for our framework and show that the \method~estimator is asymptotically normal and achieves this efficiency bound under standard regularity conditions. We further establish a robustness guarantee: \method~is never worse than a natural baseline strategy, even when nuisance components are not consistently estimated. 
    \item Fourth, we develop new technical tools in semiparametric inference to enable optimal budget allocation under mixed supervision. 
    This setting, combining ground-truth labels, pairwise preference labels, and pseudo-labels, lies beyond existing tools developed for specialized settings such as average treatment effect (ATE) estimation. We develop a novel and general variance-minimization approach based on the efficient influence function (EIF), tailored to a multi-dimensional propensity structure, and optimize the resulting asymptotic variance functional numerically under the budget constraint. The key technical steps include an EIF characterization for multi-type monotone missingness and an optimization-based derivation of the optimal data acquisition policy. 
    {Our key technical innovation is threefold: First, we derive an EIF structure for the multi-type missing data problem (Proposition~\ref{prop: EIF MCAR monotone}). Second, rather than seeking closed-form solutions for optimal propensity scores, we directly express the asymptotic variance of our estimator as a functional of the multi-dimensional propensity scores. Third, we optimize this variance functional numerically with respect to the propensity scores under the budget constraint, yielding an optimal data acquisition strategy.} These tools may be of independent interest and can be readily adapted to other semiparametric problems involving heterogeneous supervision and budget-constrained data acquisition. 
    \item Finally, we demonstrate the practical utility of our framework through a real-world application on evaluating the politeness of online requests, the cleaned Stanford Politeness Corpus (sourced from Stack Exchange \& Wikipedia) \citep{DanescuNiculescuMizil2013}. In this study, we aim to estimate the relationship between linguistic hedging and politeness scores using a dataset of online requests. We leverage two AI models, GPT-4o-mini and DeepSeeker-V3.1, to generate pseudo-labels for politeness. By treating human-annotated scores as ground truth and constructing preference labels based on the proximity of model predictions to these human scores, we show how \method~can effectively allocate a limited budget to achieve precise estimation of the target parameter, significantly outperforming baseline strategies.
\end{itemize}


\subsection{Related Work}
Our work connects to three research areas. First, pairwise comparison data is widely used in ranking \citep{burges2005learning}, text generation evaluation \citep{stiennon2020learning}, and has become central to aligning LLMs with human values through Reinforcement Learning from Human Feedback (RLHF) \citep{ouyang2022traininglanguagemodelsfollow, christiano2017deep, rafailov2023direct}. 

Second, we build on semiparametric inference, which provides a framework for improving statistical efficiency by leveraging auxiliary information in missing data settings \citep{robins1994estimation, chernozhukov2018double}. This approach has seen a modern resurgence through Prediction-Powered Inference (PPI) \citep{angelopoulos2023prediction, angelopoulos2023ppi++}, where machine learning predictions enhance inference when true outcomes are scarce. Recent work has developed decorrelation techniques to ensure robustness \citep{miao2023assumption, ji2025predictionssurrogatesrevisitingsurrogate}.

Third, our work relates to experiment design and active learning, both addressing optimal data acquisition \citep{settles2009active, hahn2011adaptive}. While recent studies explore budget-constrained data selection \citep{zrnic2024active, angelopoulos2025costoptimalactiveaimodel}, they lack formal efficiency guarantees. Methods with efficiency guarantees \citep{ao2024predictionguidedactiveexperiments, zhang2025efficient} are limited to two data types and require closed-form solutions. Our framework overcomes these limitations by handling multiple data types and directly minimizing asymptotic variance without closed-form requirements. A comprehensive literature review is provided in Appendix~\ref{sec: app additional related work}.

\subsection{Notation}
For any dataset $\gD$ and function $f$, we follow the convention and denote the sample average of $f$ on $\gD$ and its expectation as $ \check{\P}_{\gD} [f] = \frac{1}{\abs{\gD}} \sum_{x_i \in \gD} f(x_i), \check{\P}[f] = \E \qty[f(X)]$.
For a matrix $A$, denote $\norm{A}_{\mathrm{sp}}$ as the spectral norm, $\norm{A}_{\infty, \infty} := \max_{1 \le i \le m, \, 1 \le j \le n} |A_{ij}|$ as the maximum absolute entry of $A$, $\sigma_{\min}(A) := \min_{\|x\|_2=1} \|Ax\|_2$ as the minimum singular value of $A$ and $\Tr(A)$ as the trace of $A$.
We define $[a] + b = \{i + b \mid i \in [a]\} = \{b + 1, b + 2, \dots, b + a\}$.
$X_n \xrightarrow{P} X$ denotes convergence in probability, defined as $\lim_{n \to \infty} P(|X_n - X| > \epsilon) = 0, \quad \forall \epsilon > 0$.
$\min\{a,b\}$ denotes the minimum of $a$ and $b$.
We use $C(\cdot)$ to denote some constants depending on the parameters specified in the argument that may vary from place to place. 
\subsection{Organization}
The remainder of this paper is organized as follows. In Section \ref{sec: problem description}, we formally introduce the problem formulation and our modeling framework. We then present our proposed method, \method, in Section \ref{sec: method} and establish its theoretical properties in Section \ref{sec: theory}. Section \ref{sec: experiment} provides an empirical evaluation with both simulation study and a real-data analysis. Finally, we conclude with a discussion of future research directions in Section \ref{sec: discussion}.

\section{Problem Formulation}
\label{sec: problem description}
{We consider the problem of estimating a target parameter $\theta(P_{X,Y}) \in \R^d$, defined as a functional of the joint distribution $P_{X, Y}$, where $X$ are the covariates and $Y$ is the true outcome. The available data comprises covariates $X$, along with pseudo-outcomes $W_1$ and $W_2$ generated by two distinct AI models. However, naive estimates constructed by treating these pseudo-outcomes as the true labels may exhibit inaccuracies and biases. To address this, we acquire additional data, either preference outcome $V$ indicating which diagnosis is more plausible, or to obtain the true outcome $Y$, potentially observed at different prices. 

{To make this setup more concrete, consider an AI-assisted content moderation task \citep{gorwa2020algorithmic}. For a user post $x$, two models propose risk assessments $w_1$ and $w_2$. A moderator can either provide a low-cost preference $v$ (based on a quick check), or conduct an expensive investigation to obtain the ground truth $y$. Optimizing this trade-off is crucial for scalability, as obtaining $y$ incurs significantly higher costs than $v$ \citep{stanford2024llm, arxiv2508.05527}.}

Given this framework, we address two fundamental questions: what is the most efficient method for estimating the target parameter $\theta(P_{X,Y})$? Furthermore, under a limited budget, how should the available resources be optimally allocated across these two types of outcome data collection?}

\subsection{Data Setup}
We assume access to a large dataset containing covariates and two pseudo outcomes $\{(x_i, w_{1i}, w_{2i}) | i \in [n_{\text{tot}}]\}$, where $n_{\text{tot}}$ is the total number of samples. We take pseudo-outcomes as cost-free. Given a fixed budget $B$, we consider the cost structure associated with outcome acquisition: the cost of obtaining a preference outcome $v_i$ alone is normalized to $1$, and the cost of obtaining both the true outcome $y_i$ and the preference outcome $v_i$ is denoted as $\rho$ with $\rho>1$. 

This setup results in three types of data:
\begin{enumerate}
    \item $\{(x_i, y_i, w_{1i}, w_{2i}, v_i) | i \in [n_{\text{lab}}]\}$: covariates, true outcome, two pseudo-outcomes and the preference outcome.
    
    \item $\{(x_i, w_{1i}, w_{2i}, v_i) | i \in [n_{\text{pre}}] + n_{\text{lab}}\}$: covariates, two pseudo-outcomes and the preference outcome. 
    
    \item $\{(x_i, w_{1i}, w_{2i}) \mid i \in [n_{\text{unl}}] + n_{\text{lab}} + n_{\text{pre}}\}$: covariates, two pseudo-outcomes.
\end{enumerate}
Here, $v_i = 1$ indicates that $w_{1i}$ is preferred, whereas $v_i = 0$ indicates that $w_{2i}$ is preferred. Correspondingly, the dataset consists of $n_{\text{lab}}$ fully labeled samples containing both the true outcome and the preference outcome, $n_{\text{pre}}$ samples with only the preference outcome, and $n_{\text{unl}}$ unlabeled samples. The sample sizes are related by $n_{\text{tot}} = n_{\text{lab}} + n_{\text{pre}} + n_{\text{unl}}$, and the total acquisition cost is constrained by the budget
\begin{equation}
\label{eq: budget constraint no random case}
\rho\cdot n_{\text{lab}} + n_{\text{pre}} \leq B.
\end{equation}
\begin{rem}
    In this paper, we consider acquiring either the preference outcome $v_i$ or the true outcome along with the corresponding preference outcome $(y_i,v_i)$. Specifically, if the true outcome $y_i$ is available, the associated preference outcome $v_i$ is observed automatically, i.e. $v_i$ is a \textbf{known deterministic function} of $(y_i, w_{1i}, w_{2i})$.
\end{rem}

\subsection{Missing Data Framework}
The data can be formalized within a missing data framework, where each observation is represented as $((X, Y, W_1, W_2, V)^R, R)$. Here, $X$ denotes the covariates, $Y$ is the true outcome, $W_1$ and $W_2$ are pseudo-outcomes based on AI models, and $V$ represents the preference outcome. Additionally, the superscript $^{R}$ indicates the type of missing pattern, such that $r_1^* = \{1,2,3,4,5\}, r_2^* = \{1,3,4,5\}, r_3^* = \{1,3,4\}$ represent the indices that can be observed in each type of dataset. 

This structure corresponds to a monotone missing data mechanism, where the occurrence of a variable in an earlier group implies its presence in subsequent groups, or vice versa. Moreover, we assume that $V$ is a known
deterministic function of $(Y, W_1, W_2)$. For instance,
one may take $V = \mathbbm{1}\qty{\norm{W_1 - Y} \leq \norm{W_2 - Y}}$. Under this assumption, once the true outcome $Y$ is observed, the corresponding preference outcome $V$ is determined automatically.

\paragraph{Covariate-aware strategy}

To improve labeling efficiency, we explore a covariate-aware strategy, which selects data points for labeling based on their observed features. This setting is captured by the Missing at Random (MAR) framework, under which the probability of data selection depends on the observed covariates. Formally, the MAR assumption states that the observation pattern is conditionally independent of the potential unobserved response given the observed covariates, i.e. $\P(R = r\mid X, Y, W_1, W_2, V) = \P(R = r\mid X, W_1, W_2)$. This framework is particularly relevant as it allows the labeling strategy to leverage the always-observed variables $(X, W_1, W_2)$ to identify and prioritize more informative samples. We therefore use the propensity scores
$\alpha_j(x, w_1, w_2) := dP_{R\mid X, W_1, W_2}(r_j^* \mid x, w_1, w_2), j \in [3]$
to guide the selection of data points for labeling. The induced distribution of the observed data will be
\begin{multline*}
    dP_{(X, Y, W_1, W_2, V)^R, R}((x, y, w_1, w_2, v)^r, r) = \\\sum_{j = 1}^3 \mathbbm{1}(r = r_j^*) \alpha_j(x, w_1, w_2) dP_{(X, Y, W_1, W_2, V)^{r_j^*}}((x, y, w_1, w_2, v)^{r_j^*}),
\end{multline*}
where $dP$ denotes the probability measure.

\section{Method}
\label{sec: method}

The budget allocation problem can be viewed as assigning labels to data points under the constraint \eqref{eq: budget constraint no random case}. Under the MAR framework defined above, it can be formulated as determining the propensity scores $\alpha_j(x, w_1, w_2)$, which specify the probability that an observation is assigned to label type $j$, $j\in[3]$. Thus, each data point is randomized into one of the three label type according to $\alpha_j(x, w_1, w_2)$. Consequently, the deterministic budget \eqref{eq: budget constraint no random case} is reformulated as the constraint on the expected total labeling cost.
\begin{equation}
\label{eq: budget constraint MAR inital}
\begin{cases}
\sum_{j = 1}^3 \alpha_j(x, w_1, w_2) = 1 \\
n_{\text{tot}} \check{\P}[\rho \alpha_1 + \alpha_2] \leq B
\end{cases},    
\end{equation}
where $\check{\P}[\rho \alpha_1 + \alpha_2]$ is the expected cost per sample. Define $\tau = \frac{B}{n_{\text{tot}}}$, we have
\begin{equation}
\label{eq: budget constraint MAR}
\begin{cases}
\sum_{j = 1}^3 \alpha_j(x, w_1, w_2) = 1 \\
\check{\P}[\rho \alpha_1 + \alpha_2] \leq \tau
\end{cases}.
\end{equation}

Now the task is to determine the propensity scores $\alpha_j(x, w_1, w_2), j \in [3]$, which govern how labels are acquired for an unlabeled dataset, and then use the resulting dataset to construct an efficient estimator. Accordingly, our procedure consists of two steps:
\begin{enumerate}
    \item Determine $\alpha_j(x, w_1, w_2), j \in [3]$. Subsequently, for each data point, we sample a missingness pattern according to $\alpha_j(x, w_1, w_2), j \in [3]$.
    \item Estimate $\theta(P_{X,Y})$ based on the dataset obtained from the first step.
\end{enumerate}
We assume that an initial labeled dataset $\gD_{r_1^*}^\0$ is available to provide a preliminary estimate of the propensity scores. We then allocate the budget based solely on the unlabeled dataset. {In the following, we use $n_1$ to denote the size of the total initial unlabeled dataset, which corresponds to $n_{\text{tot}}$ used above, and $n_0$ to denote the size of the initial labeled dataset used to obtain our allocation strategy.}
The overall pipeline is as follows.

\begin{algorithm}[!ht]
\caption{Preference-Calibrated Active Learning (\method)}
\label{alg: method}
\begin{algorithmic}[1]
\State \textbf{Input:} Total budget $B$; Initial labeled dataset $\gD_{r_1^*}^\0 = \{(x_i, y_i, w_{1i}, w_{2i}, v_i) \mid i \in [n_0]\}$; Unlabeled dataset $\gD_{\text{unl}} = \{(x_i, w_{1i}, w_{2i}) \mid i \in [n_1]\}$. 
\State \textbf{Output:} Estimator $\hat{\theta}^{\text{\method}}$.
\State Run Algorithm~\ref{alg: obtain alpha1} on the initial dataset $\gD_{r_1^*}^\0$ to estimate $\hat{\alpha}^{\text{\method}}_j(x, w_1, w_2), j \in [3]$.
\State Use $\hat{\alpha}^{\text{\method}}_j(x, w_1, w_2), j \in [3]$ and the budget $B$ to selectively label data from the unlabeled set and construct the final dataset $\gD = \{\gD_{r_j^*}, j \in [3]\}$.
\State Use $\gD = \{\gD_{r_j^*}, j \in [3]\}$ to run Algorithm~\ref{alg: estimator} to get our final estimator $\hat{\theta}^{\text{\method}}$.
\end{algorithmic}
\end{algorithm}

Algorithm \ref{alg: method} provides an overall pipeline with two key components: the allocation rule (Algorithm \ref{alg: obtain alpha1}) and the final estimator (Algorithm \ref{alg: estimator}). 

The main idea is summarized as follows.
Our goal is to construct the most efficient estimator of the target parameter $\theta$ under the constraint \eqref{eq: budget constraint MAR}, where ``the most efficient'' refers to minimizing the asymptotic variance of the proposed estimator. 
To achieve this, we first consider the form of the estimator for fixed propensity scores $\alpha_j(x, w_1, w_2), j \in [3]$. 
The key step is to construct an efficient estimator of $\theta$ based on the dataset $\gD$, which highly relies on the form of EIF.
Subsequently, its asymptotic variance then can be expressed as a function of propensity scores $\alpha_j(x, w_1, w_2), j \in [3]$, which naturally leads to an optimization problem: the optimal allocation is obtained by minimizing the variance with respect to $\alpha_j(x, w_1, w_2), j \in [3]$, thereby yielding the most efficient estimator. See Algorithm~\ref{alg: estimator} for the detailed procedure. 

To facilitate understanding of the remainder of the paper, we first provide a brief summary of essential concepts and results from semiparametric efficiency theory in Section \ref{sec: overview of semiparametric} and derive the EIF in Section \ref{sec: eif}. We then present our detailed algorithms in Section \ref{sec: algorithms}. Lastly we consider a special case, the covariate-agnostic strategy, in Section \ref{sec: mcar case}.

\subsection{Overview of semiparametric efficiency theory}
\label{sec: overview of semiparametric}

Let $\{X_i\}_{i=1}^n$ be independent and identically distributed with common distribution $P^* \in \gP$, and suppose there exists a dominating measure $\mu$ such that each element of $\gP$ can be represented by its corresponding density with respect to $\mu$. We are interested in a functional $\theta(\cdot): \gP \rightarrow \Theta \subset \R^k$, and write $\theta^* := \theta(P^*)$. Consider a one-dimensional regular parametric sub-model $\gP_T = \{ P_t : t \in T \subset \R \} \subset \gP$ such that $P_{t^*} = P^*$ for some $t^* \in T$. Each $\gP_T$ defines a score function at $P^*$. The set of score functions at $P^*$ from all such one-dimensional regular parametric sub-models of $\gP$ forms the \textit{tangent set} at $P^*$ relative to $\gP$, denoted as $\dot{\gP}_P$,  and the closed linear span of the tangent set is the \textit{tangent space} at $P^*$ relative to $\gP$, denoted as $\overline{\operatorname{lin}} \dot{\gP}_P$.

We define $\theta : \gP \mapsto \mathbb{R}^k$ to be \textit{differentiable at $P^*$} relative to a given tangent set $\dot{\gP}_P$ if there exists a continuous linear map $\dot{\theta}_P : L_2(P) \mapsto \mathbb{R}^k$ such that for every $g \in \dot{\gP}_P$ and a submodel $t \mapsto P_t$ with score function $g$, as $t \rightarrow 0$,
$\frac{\theta(P_t) - \theta(P^*)}{t} \longrightarrow \dot{\theta}_{P^*} g$.
This requires that the derivative of the map $t \mapsto \theta(P_t)$ exists in the ordinary sense, and also that it has a special representation.
By the Riesz representation theorem for Hilbert spaces, the map $\dot{\theta}_P$ can always be written in the form of an inner product with a fixed vector-valued, measurable function $\tilde{\theta}_P: \mathcal{X} \mapsto \mathbb{R}^k$, $\dot{\theta}_P g = \langle \tilde{\theta}_P, g \rangle_P = \int \tilde{\theta}_P g \, dP$.

Here the function $\tilde{\theta}_P$ is not uniquely defined by the functional $\theta$ and the model $\gP$, because only inner products of $\tilde{\theta}_P$ with elements of the tangent set are specified, and the tangent set does not span all of $L_2(P)$. However, it is always possible to find a candidate $\tilde{\theta}_P$ whose coordinate functions are contained in $\overline{\operatorname{lin}} \dot{\gP}_P$, the closure of the linear span of the tangent set. This function is unique and is called the \textit{efficient influence function}. It can be found as the projection of any other “influence function” onto the closed linear span of the tangent set.

An estimator $\hat{\theta}_n = \hat{\theta}_n(X_1, \ldots, X_n)$ of $\theta^*$ is \textit{regular} relative to $\gP$ if, for any regular parametric sub-model $\gP_T = \{ P_t : t \in T \subset \R \} \subset \gP$ such that $P^* = P_{t^*}$ for some $t^* \in T$,
$\sqrt{n} \qty[ \hat{\theta}_n - \theta\qty( P_{t^* + \frac{h}{\sqrt{n}}} ) ] \xrightarrow{d} L_{P^*}$
for all $h \in \R$ under $P_{t^* + \frac{h}{\sqrt{n}}}^n$, where $L_{P^*}$ is a random variable whose distribution depends on $P^*$ but not on $h$.




Semiparametric efficiency theory aims to provide a lower bound for the asymptotic variances of regular estimators.
A more comprehensive introduction can be found in Chapter 25 of \citet{van2000asymptotic}.

\subsection{Efficient Influence Function}
\label{sec: eif}
Given the propensity scores $\alpha_j(x, w_1, w_2)$,
it is necessary to derive the efficient influence function (EIF) for $\theta$, which provides the lowest-variance unbiased representation of $\theta$ in the semiparametric model. Based on this EIF, we then apply the decorrelation strategy to construct the estimator.

Let $\psi_{P_{X,Y}}(X, Y)$ denote the EIF of \( \theta(P_{X,Y}) \) in the nonparametric model space
$\mathcal{P}_{X,Y} = \{ \text{All distributions of } (X, Y) \}$.

\begin{asm}
\label{asm: model distribution}
The two AI-generated pseudo-outcomes, $W_1$ and $W_2$, are treated as unknown, meaning that the joint distribution of $(W_1, W_2)$ is unrestricted. The variable $V$ is assumed to be a deterministic function of $(Y, W_1, W_2)$, i.e., $V$ is fully determined given $(Y, W_1, W_2)$.
\end{asm}
Then, the EIF $\psi$ of $\theta(P_{X,Y})$ in the nonparametric observed model space $\gP_{(X,Y,W_1, W_2, V)^R, R} = \{ \text{All distributions of } ((X,Y,W_1, W_2, V)^R, R) \text{ satisfying Assumption~\ref{asm: model distribution}} \}$
has a simplified form based on $\psi_{P_{X,Y}}(X, Y)$, as stated in Proposition \ref{prop: EIF MCAR monotone} below.
\begin{prop}
\label{prop: EIF MCAR monotone}
Consider the model space defined above. When the missingness mechanism satisfies MAR, the EIF of $\theta$ is given by:
\begin{multline}
\label{eq: EIF MAR monotone}
   \psi\qty((X, Y, W_1, W_2, V)^{R},R) 
   = \mathbbm{1}\{R = r_1^*\} \frac{1}{\alpha_1(X, W_1, W_2)} 
     \psi_1(X, Y, W_1, W_2, V) \\
   + \sum_{j =1}^3 \mathbbm{1}\{R = r_j^*\} 
     \phi_j\qty((X, Y, W_1, W_2, V)^{r_j^*}).
\end{multline}
where
\begin{equation}
\label{eq: phi defn}
\begin{aligned}
&\phi_1(X, Y, W_1, W_2, V) 
= - \frac{\alpha_2(X, W_1, W_2)\psi_2(X, W_1, W_2, V)}{\gamma_2(X, W_1, W_2)\gamma_1(X, W_1, W_2)}  - \frac{\alpha_3(X, W_1, W_2)\psi_3(X, W_1, W_2)}{\gamma_2(X, W_1, W_2)} , \\
&\phi_2(X, W_1, W_2, V)  
= \frac{\psi_2(X, W_1, W_2, V)}{\gamma_2(X, W_1, W_2)} -\frac{\alpha_3(X, W_1, W_2) \psi_3(X, W_1, W_2)}{\gamma_3(X, W_1, W_2)\gamma_2(X, W_1, W_2)} 
    , \\
&\phi_3(X, W_1, W_2) 
= \psi_3(X, W_1, W_2).
\end{aligned}
\end{equation}
and $\psi_j\qty((X, Y, W_1, W_2, V)^{r_j^*}) = \E\qty[\psi_{P_{X,Y}}(X, Y) \mid (X, Y, W_1, W_2, V)^{r_j^*}]$, $\gamma_j(X, W_1, W_2) = \sum_{l = 1}^j \alpha_l(X, W_1, W_2)$.
\end{prop}
\begin{rem}
A notable feature of this derivation is the inductive structure of the EIF. This structure not only simplifies the proof of efficiency (minimum variance) but also ensures our method extends naturally to scenarios with more than three data types. Furthermore, we can verify this result by demonstrating that the derived EIF lies within the model's tangent space. This procedure provides both a validation check and an alternative derivation. 
\end{rem}

\subsection{Algorithms}
\label{sec: algorithms}
As Algorithm \ref{alg: method} provides the overall procedure, in this section, we describe the details of Algorithm \ref{alg: estimator} and Algorithm \ref{alg: obtain alpha1}, which build on the EIF formulation derived in Section \ref{sec: eif}.

We begin by presenting the estimation algorithm. The main idea is to construct an efficient estimator using the labeled data, then apply the EIF derived above \eqref{eq: EIF MAR monotone} to decorrelate the estimator. The details are provided in  Algorithm \ref{alg: estimator} below.
\begin{algorithm}[!ht]
\caption{\method~EIF-Based Estimator}
\label{alg: estimator}
\begin{algorithmic}[1]
\State \textbf{Input:} Dataset $\gD = \{\gD_{r_j^*}, j \in [3]\}$; The propensity scores $\alpha_j(x, w_1, w_2), j \in [3]\}$.
\State \textbf{Output:} Estimator $\hat{\theta}^{\text{\method}}$.

\State \textbf{Step 1:} Randomly split the complete samples $\gD$ into three parts: $\gD_1$, $\gD_2$, and $\gD_3$.

\State \textbf{Step 2:} On $\gD_3 \cap \gD_{r_1^*}$, construct estimators 
$\hat{\psi}^\1_j\qty((x, y, w_1, w_2, v)^{r_j^*}),\ j \in [3]$. 
Then compute $\hat{\phi}^\1_j\qty((x, y, w_1, w_2, v)^{r_j^*}),\ j \in [3]$, 
according to Equation~\eqref{eq: phi defn}.

\State \textbf{Step 3:} On $\gD_2$, given the estimators in Step 2, estimate
\begin{equation*}
M^\1 = -\check{\P}\qty[\mathbbm{1}\{R = r_1^*\}\frac{\hat{\psi}^\1_1\hat{\phi}^{\1\top}_1}{\alpha_1}]\qty( \check{\P}\qty[\sum_{j = 1}^3 \mathbbm{1}\{R = r_j^*\} \hat{\phi}_j\hat{\phi}_j^\top])^{-1}   
\end{equation*}
by $\hat{M}^\1 = -\check{\P}_{\gD_2 \cap \gD_{r_1^*}}\qty[\mathbbm{1}\{R = r_1^*\}\frac{\hat{\psi}^\1_1\hat{\phi}^{\1\top}_1}{\alpha_1}]\qty( \check{\P}_{\gD_2}\qty[\sum_{j = 1}^3 \mathbbm{1}\{R = r_j^*\} \hat{\phi}_j\hat{\phi}_j^\top])^{-1}$.
\State \textbf{Step 4:} On $\gD_1$, compute an efficient estimator $\tilde{\theta}^{(1)}$ through the \textbf{weighted} algorithm for $\theta$ in $\check{\P}_{X,Y}$ and set
\begin{equation*}
    \hat{\theta}^{(1)} = \tilde{\theta}^{(1)} +
    \hat{M}^\1 \cdot \frac{1}{|\gD_1|} \qty( \sum_{i \in \gD_1} \sum_{j=1}^3 \mathbbm{1}\{r_i = r_j^*\} \hat{\phi}_j^\1\qty((x_i, y_i, w_{1i}, w_{2i}, v_i)^{r_j^*}) ).  
\end{equation*}
\State \textbf{Step 5:} Repeat Steps 2–4 with fold rotations: $(\gD_2, \gD_3, \gD_1)$ and $(\gD_3, \gD_1, \gD_2)$ to obtain $\hat{\theta}^{(2)}$ and $\hat{\theta}^{(3)}$

\State \textbf{Step 6:} Set the final estimator as the weighted average:
\[
\hat{\theta}^{\text{\method}} = \sum_{k = 1}^3 \frac{|\gD_k|}{n_1} \hat{\theta}^{(k)}.
\]
\end{algorithmic}
\end{algorithm}
\begin{rem}
    The matrix $M^\1$ serves as a decorrelation matrix, guaranteeing that our estimator, $\hat{\theta}^{(1)}$ is at least as efficient as the initial estimator, $\tilde{\theta}^{(1)}$, even when the nuisance functins  $\phi_j, j \in [3]$ are poorly estimated. This technique is similar to methods used in \citet{miao2023assumption, ji2025predictionssurrogatesrevisitingsurrogate}.
    The weighted algorithm addresses the distribution shift between the complete-case distribution (the conditional distribution given $R = r_1^*$)  and the full data distribution (the marginal distribution of $(X, Y, W_1, W_2, V)$). By reweighting each complete-case datapoint using the inverse propensity score $1/\alpha_1(X, W_1, W_2)$, we can recover unbiased expectations under the full data distribution. For example, for $Z$ estimation, if the initial estimation equation under the full data distribution is $\E\qty[f(X, Y, W_1, W_2, V)] = 0$, we construct the weighted estimation equation as $\E\qty[\mathbbm{1}\qty{R =  r_1^*} f(X, Y, W_1, W_2, V)/\alpha_1(X, W_1, W_2)] = 0$.
\end{rem}

Once the form of the efficient estimator is established, we can estimate its asymptotic variance and subsequently minimize it with respect to the propensity scores $\alpha_j(x, w_1, w_2), j \in [3]$. 

To establish the asymptotic normality of $\hat{\theta}^{\text{\method}}$, we state two natural assumptions.
\begin{asm}[Stability of the Black-Box Algorithm]
\label{asm: stability of psi}
There exists $\tilde{\phi}_j$ with $\check{\P}\qty[\norm{\tilde{\phi}_j}_2^2] < \infty$ such that as $n_1 \to \infty$,
$\E\qty[\check{\P}\qty[\norm{ \hat{\phi}_j^\k - \tilde{\phi}_j }_2^2 ]]\to 0$ for $k = 1,2,3$.
Moreover, $\E\qty[\check{\P}\qty[\norm{ \hat{\psi}_1^\k - \psi_1 }_2^2 ]]\to 0$ as $n_1 \to \infty$, i.e., $\psi_1$ is consistently estimable. Here, $\check{\P}$ denotes expectation over $(X, Y, W_1, W_2, V)$, and the outer expectation is over $\gD$ since the estimators are constructed from it.
\end{asm}
\begin{asm}
\label{asm: consistency of M}
    We assume $\norm{ \hat{M}^\1 - M }_\sp \xrightarrow{P} 0$.
\end{asm}
\begin{rem}
    Assumption \ref{asm: stability of psi} requires that the black box algorithms used to fit the nuisance components $\psi_1, \phi_j, j \in [3]$ are stable,implying they converge to fixed limits as $n_1$ goes to infinity. Crucially, consistency is not required for $\phi_j, j \in [3]$. For $\psi_1$, we require consistency, which is readily satisfied since we can already obtain an efficient estimator based on the labeled data. {Assumption \ref{asm: consistency of M} states that the empirical matrix $\hat{M}^\1$ is a consistent estimator of the true decorrelation matrix $M$, as $M^\1$ will converge to $M$.} 
\end{rem}
 The following theorem establishes that, assuming the propensity scores $\alpha_j(x, w_1, w_2)$ are known, the estimator is asymptotically normal.
\begin{thm}[Asymptotic Normality]
\label{thm: asymptotic normality}
Suppose Assumptions \ref{asm: model distribution} - \ref{asm: consistency of M} hold, we have
\begin{equation*}
    \sqrt{n_1}(\hat{\theta}^{\text{\method}} - \theta) = \sqrt{n_1} \check{\P}_{\gD} \qty[
\mathbbm{1}\{R = r_1^*\} \frac{\psi_1}{\alpha_1}
+ M \sum_{j=1}^3 \mathbbm{1}\{R = r_j^*\} \tilde{\phi}_j
] + o_p(1),
\end{equation*}
where
$M = -\check{\P}\qty[\mathbbm{1}\{R = r_1^*\}\frac{\psi_1\tilde{\phi}_1^\top}{\alpha_1}] \qty( \check{\P}\qty[\sum_{j = 1}^3 \mathbbm{1}\{R = r_j^*\} \tilde{\phi}_j\tilde{\phi}_j^\top])^{-1}$.
\end{thm}
After establishing Theorem~\ref{thm: asymptotic normality}, we next derive the asymptotic variance of the proposed estimator $\hat\theta$, which is stated as follows.
\begin{align}
&\Cov\qty(
\mathbbm{1}\{R = r_1^*\} \frac{\psi_1(X, Y, W_1, W_2, V)}{\alpha_1(X, W_1, W_2)}
+ M \sum_{j=1}^3 \mathbbm{1}\{R = r_j^*\} \tilde{\phi}_j\qty(\qty(X, Y, W_1, W_2, V)^{r_j^*})
)\nonumber\\
=& \Cov\qty(
\mathbbm{1}\{R = r_1^*\} \frac{\psi_1(X, Y, W_1, W_2, V)}{\alpha_1(X, W_1, W_2)}) - \Cov\qty( M \sum_{j=1}^3 \mathbbm{1}\{R = r_j^*\} \tilde{\phi}_j\qty(\qty(X, Y, W_1, W_2, V)^{r_j^*}))\nonumber\\
=& \check{\P}\qty[\mathbbm{1}\{R = r_1^*\}\frac{\psi_1\psi_1^\top}{\alpha_1^2}]- \check{\P}\qty[\mathbbm{1}\{R = r_1^*\}\frac{\psi_1\tilde{\phi}_1^\top}{\alpha_1}] \qty( \check{\P}\qty[\sum_{j = 1}^3 \mathbbm{1}\{R = r_j^*\} \tilde{\phi}_j\tilde{\phi}_j^\top])^{-1} \check{\P}\qty[\mathbbm{1}\{R = r_1^*\}\frac{\tilde{\phi}_1\psi_1^\top}{\alpha_1}]
\label{eq: variance population}
\end{align}
Based on the Equation \eqref{eq: variance population}, we can estimate the asymptotic variance and then minimize it using the first batch dataset $\gD_{r_1^*}^\0$. The resulting procedure is summarized in Algorithm~\ref{alg: obtain alpha1}.
\begin{algorithm}[!ht]
\caption{\method~Propensity Score Computation}
\label{alg: obtain alpha1}
\begin{algorithmic}[1]
\State \textbf{Input:} Initial Dataset $\gD_{r_1^*}^\0$.
\State \textbf{Output:} $\hat{\alpha}^{\text{\method}}_j(x, w_1, w_2), j \in [3]$.

\State \textbf{Step 1:} Randomly split the complete samples $\gD_{r_1^*}^\0$ into two equal parts $\gD_{r_1^*, 1}^\0$ and $\gD_{r_1^*, 2}^\0$.

\State \textbf{Step 2:} On $\gD_{r_1^*, 1}^\0$, construct the estimators 
$\hat{\psi}_j\qty((x, y, w_1, w_2, v)^{r_j^*}),\ j \in [3]$. 
Then obtain the estimators 
$\hat{\phi}_j\qty((x, y, w_1, w_2, v)^{r_j^*}),\ j \in [3]$ according to Equation~\eqref{eq: phi defn}.
\State \textbf{Step 3:} On $\gD_{r_1^*, 2}^\0$, construct the estimators $\check{\P}_{\gD_{r_1^*, 2}^\0}\qty[\frac{\hat{\psi}_1\hat{\psi}_1^\top}{\alpha_1}]$, $\check{\P}_{\gD_{r_1^*, 2}^\0}\qty[\alpha_j \hat{\phi}_j\hat{\phi}_j^\top]$ for $j \in [3]$, $\check{\P}_{\gD_{r_1^*, 2}^\0}\qty[\hat{\psi}_1\hat{\phi}_1^\top]$ and $\check{\P}_{\gD_{r_1^*, 2}^\0}\qty[\rho\alpha_1 + \alpha_2]$.

\State \textbf{Step 4:} Compute the estimator of the variance of our expected estimator:
\begin{multline*}
    \hat{\Cov}(\hat{\theta}^{\text{\method}})= 
    \check{\P}_{\gD_{r_1^*, 2}^\0}\qty[\frac{\hat{\psi}_1\hat{\psi}_1^\top}{\alpha_1}]-\\ \check{\P}_{\gD_{r_1^*, 2}^\0}\qty[\hat{\psi}_1\hat{\phi}_1^\top]\qty(\check{\P}_{\gD_{r_1^*, 2}^\0}\qty[\sum_{j = 1}^3\alpha_j \hat{\phi}_j\hat{\phi}_j^\top])^{-1}\check{\P}_{\gD_{r_1^*, 2}^\0}\qty[\hat{\phi}_1\hat{\psi}_1^\top].
\end{multline*}


\State \textbf{Step 5:} Minimize $\Tr\qty(\hat{\Cov}(\hat{\theta}^{\text{\method}}))$ w.r.t. $\alpha_j(x, w_1, w_2), j \in [3]$ under constraints $\sum_{j = 1}^3 \alpha_j(x, w_1, w_2) = 1, \alpha_1(x, w_1, w_2) \geq \underline{\alpha}$ and $\check{\P}_{\gD_{r_1^*, 2}^\0}\qty[\rho\alpha_1 + \alpha_2] \leq \tau + \delta$.
\end{algorithmic}
\end{algorithm}

\begin{rem}
    We discuss two modifications to the constraints in Step 5. First, the lower bound $\alpha_1(x, w_1, w_2) \geq \underline{\alpha}$ prevents division by zero in the inverse propensity weighting and ensures numerical stability during optimization. In practice, $\underline{\alpha}$ is set to a small constant such as $10^{-6}$. Second, we introduce a slack parameter $\delta$ in the budget constraint $\check{\P}_{\gD_{r_1^*, 2}^\0}\qty[\rho\alpha_1 + \alpha_2] \leq \tau + \delta$. While the true constraint is $\check{\P}\qty[\rho\alpha_1 + \alpha_2] \leq \tau$, the slack $\delta$ (typically very small) accounts for the approximation error when replacing the true expectation with its empirical counterpart, which facilitates both practical optimization and theoretical analysis.
\end{rem}

\subsection{Covariate-agnostic strategy}
\label{sec: mcar case}
{In many practical scenarios, the primary goal is not to identify specific data points for annotation, but rather to determine the optimal budget allocation across different label types (e.g., ground-truth labels versus preference labels). We thus consider a simplified covariate-agnostic strategy where the budget is allocated to each label type without an instance-specific selection rule. This approach is often easier to implement and formally corresponds to a Missing Completely at Random (MCAR) mechanism, as the probability of acquiring a particular label type is independent of any underlying data characteristics.} Mathematically speaking, $R \indep (X, Y, W_1, W_2, V)$, and hence the observed-data distribution can be written as
\begin{multline*}
    dP_{(X, Y, W_1, W_2, V)^R, R}((x, y, w_1, w_2, v)^r, r) = \\\sum_{j = 1}^3 \mathbbm{1}(r = r_j^*) \alpha_j dP_{(X, Y, W_1, W_2, V)^{r_j^*}}((x, y, w_1, w_2, v)^{r_j^*}),
\end{multline*}
where $r_j^*, j \in [3]$ are the same as MAR case and $\alpha_j = dP_R(r_j^*); j \in [3]$ represent the probability of $R= r_j^*$.

The budget constraint \eqref{eq: budget constraint MAR} then turns out to be
\begin{equation}
\label{eq: budget constraint}
\begin{cases}
\alpha_1 + \alpha_2 + \alpha_3 = 1 \\
\rho\alpha_1 + \alpha_2 = \tau
\end{cases},    
\end{equation}
which will be equivalent to
\begin{equation}
\label{eq: budget constraint simplified}
\begin{cases}
\alpha_2 = \tau - \rho \alpha_1 \\
\alpha_3 = 1 + (\rho - 1)\alpha_1 - \tau
\end{cases}.    
\end{equation}
Hence, the allocation rule is fully determined by the single scalar parameter $\alpha_1$. 

Under MCAR the EIF simplifies to
\begin{multline}
\label{eq: EIF MCAR monotone}
   \psi\qty((X, Y, W_1, W_2, V)^{R},R) 
   = \mathbbm{1}\{R = r_1^*\} \frac{1}{\alpha_1} 
     \psi_1(X, Y, W_1, W_2, V) \\
   + \sum_{j =1}^3 \mathbbm{1}\{R = r_j^*\} 
     \phi_j\qty(\qty(X, Y, W_1, W_2, V)^{r_j^*}).
\end{multline}

Therefore, given the constraint~\eqref{eq: budget constraint simplified}, we can express $\alpha_2$ and $\alpha_3$ in terms of $\alpha_1$, simplifying the budget allociation problem to a univariate optimization problem in $\alpha_1$.

\begin{algorithm}[!ht]
\caption{\method~Propensity Score Computation (covariate-agnostic)}
\label{alg: obtain alpha1 mcar}
\begin{algorithmic}[1]
\State \textbf{Input:} Initial Dataset  $\gD_{r_1^*}^\0$.
\State \textbf{Output:} $\hat{\alpha}^{\text{\method-CA}}_j, j \in [3]$.

\State \textbf{Step 1:} Randomly split the complete samples $\gD_{r_1^*}^\0$ into two equal parts $\gD_{r_1^*, 1}^\0$ and $\gD_{r_1^*, 2}^\0$.

\State \textbf{Step 2:} On $\gD_{r_1^*, 1}^\0$, construct the estimators 
$\hat{\psi}_j\qty((x, y, w_1, w_2, v)^{r_j^*}),\ j \in [3]$. 
Then obtain the estimators 
$\hat{\phi}_j\qty((x, y, w_1, w_2, v)^{r_j^*}),\ j \in [3]$ according to Equation~\eqref{eq: phi defn}.
\State \textbf{Step 3:} On $\gD_{r_1^*, 2}^\0$, construct the estimators $\check{\P}_{\gD_{r_1^*, 2}^\0}\qty[\hat{\psi}_1\hat{\psi}_1^\top]$, $\check{\P}_{\gD_{r_1^*, 2}^\0}\qty[ \hat{\phi}_j\hat{\phi}_j^\top]$ for $j \in [3]$, $\check{\P}_{\gD_{r_1^*, 2}^\0}\qty[\hat{\psi}_1\hat{\phi}_1^\top]$.

\State \textbf{Step 4:} Compute the estimator of the variance of our expected estimator:
\begin{equation*}
    \hat{\Cov}(\hat{\theta}^{\text{\method-CA}})= 
    \frac{\check{\P}_{\gD_{r_1^*, 2}^\0}\qty[\hat{\psi}_1\hat{\psi}_1^\top]}{\alpha_1} - \check{\P}_{\gD_{r_1^*, 2}^\0}\qty[\hat{\psi}_1\hat{\phi}_1^\top]\qty(\check{\P}_{\gD_{r_1^*, 2}^\0}\qty[\sum_{j = 1}^3 \alpha_j \hat{\phi}_j\hat{\phi}_j^\top])^{-1}\check{\P}_{\gD_{r_1^*, 2}^\0}\qty[\hat{\phi}_1\hat{\psi}_1^\top].    
\end{equation*}

\State \textbf{Step 5:} Minimize $\Tr\qty(\hat{\Cov}(\hat{\theta}^{\text{\method-CA}}))$ w.r.t. $\alpha_1$ under constraint $\alpha_1 \geq \underline{\alpha}$ and then obtain $\alpha_2, \alpha_3$ by the constraint~\eqref{eq: budget constraint simplified}.
\end{algorithmic}
\end{algorithm}
\begin{rem}
    As mentioned above, the key difference between Algorithm \ref{alg: obtain alpha1 mcar} and Algorithm \ref{alg: obtain alpha1} is that the function class under MCAR is reduced. In particular, the budget constraints imply that the problem can be expressed in terms of a single scalar parameter, so that the resulting optimization problem becomes one-dimensional. As a consequence, the constraint no longer needs to be estimated from the data.
    The corresponding estimator, denoted $\hat{\theta}^{\text{\method-CA}}$ is obtained from Algorithm \ref{alg: estimator mcar} below.
\end{rem}
After determining $\hat{\alpha}_j$, we use them to allocate the budget $B$, thereby generating new labeled and preferred data $\gD$.
\begin{algorithm}[!ht]
\caption{\method~EIF-Based Estimator (covariate-agnostic)}
\label{alg: estimator mcar}
\begin{algorithmic}[1]
\State \textbf{Input:} Dataset $\gD = \{\gD_{r_j^*}, j \in [3]\}$; The proensity scores $\alpha_j$.
\State \textbf{Output:} Estimator $\hat{\theta}^{\text{\method-CA}}$.

\State \textbf{Step 1:} Randomly split the complete samples $\gD$ into three parts: $\gD_1$, $\gD_2$, and $\gD_3$.

\State \textbf{Step 2:} On $\gD_3 \cap \gD_{r_1^*}$, construct estimators 
$\hat{\psi}^\1_j\qty((x, y, w_1, w_2, v)^{r_j^*}),\ j \in [3]$. 
Then compute $\hat{\phi}^\1_j\qty((x, y, w_1, w_2, v)^{r_j^*}),\ j \in [3]$, 
according to Equation~\eqref{eq: phi defn}.

\State \textbf{Step 3:} On $\gD_2$, given the estimators in Step 2, estimate
\[
M^\1 = -\check{\P}\qty[\hat{\psi}^{\1}_1\hat{\phi}^{\1\top}_1] \qty(\sum_{j = 1}^3 \alpha_j\check{\P}\qty[ \hat{\phi}_j^\1\hat{\phi}_j^{\1\top}])^{-1}
\]
by
\begin{equation}
\label{eq: m matrix estimation mcar}
\hat{M}^\1 = -\check{\P}_{\gD_2 \cap \gD_{r_1^*}}\qty[\hat{\psi}^{\1}_1\hat{\phi}^{\1\top}_1] \qty(\sum_{j = 1}^3 \alpha_j\check{\P}_{\gD_2 \cap \gD_{r_j^*}}\qty[ \hat{\phi}_j^\1\hat{\phi}_j^{\1\top}])^{-1}.
\end{equation}
\State \textbf{Step 4:} On $\gD_1$, compute an efficient estimator $\tilde{\theta}^{(1)}$ for $\theta$ in $P_{X,Y}$ and set
\begin{equation*}
    \hat{\theta}^{(1)} = \tilde{\theta}^{(1)} +
    \hat{M}^\1 \cdot \frac{1}{|\gD_1|} \qty( \sum_{i \in \gD} \sum_{j=1}^3 \mathbbm{1}\{r_i = r_j^*\} \hat{\phi}_j^\1\qty((x_i, y_i, w_{1i}, w_{2i}, v_i)^{r_j^*}) ).  
\end{equation*}
\State \textbf{Step 5:} Repeat Steps 2–4 with fold rotations: $(\gD_2, \gD_3, \gD_1)$ and $(\gD_3, \gD_1, \gD_2)$ to obtain $\hat{\theta}^{(2)}$ and $\hat{\theta}^{(3)}$

\State \textbf{Step 6:} Set the final estimator as the weighted average:
\[
\hat{\theta}^{\text{\method-CA}} = \sum_{k = 1}^3 \frac{|\gD_k|}{n_1} \hat{\theta}^{(k)}.
\]

\end{algorithmic}
\end{algorithm}
\begin{rem}
    This covariate-agnostic version algorithm is a straightforward modification of Algorithm \ref{alg: estimator}. It simplifies in two respects. First, we can directly use $\alpha_j$ when estimating $M^\1$ in Equation \eqref{eq: m matrix estimation mcar}. Second, the weighted algorithm is no longer required when computing the efficient estimator $\tilde{\theta}^{(1)}$, since there is no distribution shift under MCAR.
\end{rem}

\section{Theoretical Analysis}
\label{sec: theory}
The theoretical performance of \method~is analyzed in this section. In Section \ref{sec: analysis for algorithm obtain alpha1}, we formally prove that Algorithm \ref{alg: obtain alpha1} produces an optimal score, $\alpha_j(x, w_1, w_2), j \in [3]$, in the sense that it minimizes the estimator's variance. We then proceed to analyze the key statistical properties of the final estimator, constructed via Algorithm \ref{alg: estimator}, which are detailed in Section \ref{sec: estimators performace}.

\subsection{Theoretical Guarantees of Propensity Score Estimator}
\label{sec: analysis for algorithm obtain alpha1}
We first analyze the performance of Algorithm \ref{alg: obtain alpha1} by establishing an upper bound on its excess risk, where the risk is defined as the variance of the final estimator produced by Algorithm \ref{alg: estimator}.

We represent $\alpha(x, w_1, w_2) = \qty(\alpha_1(x, w_1, w_2), \alpha_2(x, w_1, w_2), \alpha_3(x, w_1, w_2))^\top$.
Define the loss functional $\gL(\alpha(x, w_1, w_2))$ as
\begin{equation}
\label{eq: loss functional}
\gL(\alpha(x, w_1, w_2)) := 
\Tr\qty( \check{\P}\qty[\frac{\psi_1 \psi_1^\top}{\alpha_1}] - \check{\P}\qty[\psi_1 \tilde{\phi}_1^\top] \, \qty(\check{\P}\qty[ \sum_{j=1}^3 \alpha_j \tilde{\phi}_j \tilde{\phi}_j^\top])^{-1} \check{\P}\qty[\phi_1 \tilde{\psi}_1^\top] ).
\end{equation}
This population-version loss functional, $\gL$, represents the asymptotic variance of the proposed estimator.
Given $\alpha_1(x, w_1, w_2) \geq \underline{\alpha}$ and the budget constraint $\check{\P}\qty[\rho \alpha_1 + \alpha_2] \leq \tau$, we define the feasible set $\gF_1$ as:
\begin{equation*}
\label{eq: feasible set F1}
\gF_1
= \Bigg\{ \alpha(x, w_1, w_2) \;\Big|\; 
   \sum_{j=1}^3 \alpha_j(x, w_1, w_2) = 1,\ 
   \alpha_1(x, w_1, w_2) \geq \underline{\alpha},\ 
   \check{\P}\qty[\rho \alpha_1 + \alpha_2] \leq \tau \Bigg\}.
\end{equation*}
The optimal design is defined as
\begin{equation}
\label{eq: optimization problem for alpha}
\alpha^*(x, w_1, w_2) 
= \argmin_{\alpha(x, w_1, w_2) \in \gF_1} 
\gL(\alpha(x, w_1, w_2)).
\end{equation}
Moreover, denote $\hat{\alpha}^{\text{\method}}(x, w_1, w_2)$ as the output of Algorithm~\ref{alg: obtain alpha1}.
We now introduce two assumptions before establishing the corresponding error bound.
\begin{asm}[Stability of the Black-Box Algorithm]
\label{asm: stronger stability}
Assume the following stability and boundedness conditions hold:
\begin{enumerate}
    \item[(i)] \textbf{Stability:} As $n_0 \to \infty$, 
    $\E\qty[\check{\P}\qty[\norm{ \hat{\psi}_1 \hat{\psi}_1^\top - \psi_1 \psi_1^\top}_{\mathrm{sp}}]] \leq \epsilon_{n_0} \to 0$, and similarly for $\hat{\phi}_j \hat{\phi}_j^\top$ ($j \in [3]$) and $\hat{\psi}_1 \hat{\phi}_1^\top$.
    
    \item[(ii)] \textbf{Boundedness:} There exists $K_1 > 0$ such that $\check{\P}\qty[\norm{ \hat{\psi}_1 \hat{\psi}_1^\top - \psi_1 \psi_1^\top }_{\mathrm{sp}}] \leq K_1$, and similarly for $\hat{\phi}_j \hat{\phi}_j^\top$ ($j \in [3]$) and $\hat{\psi}_1 \hat{\phi}_1^\top$.
\end{enumerate}
Here, the inner $\check{\P}$ denotes expectation over $(X, Y, W_1, W_2, V)$, and the outer expectation is over $\gD_{r_1^*}^\0$.
\end{asm}

\begin{asm}[Non-degeneracy]
\label{asm: nondegeneracy}
There exist constants $K_2, K_3 > 0$ such that $\sigma_{\min}\qty(\check{\P}\qty[\tilde{\phi}_1 \tilde{\phi}_1^\top]) \geq K_2$ and $\norm{\check{\P}\qty[\psi_1 \tilde{\phi}_1^\top]}_{\mathrm{sp}} \leq K_3$.
\end{asm}
\begin{dfn}[Sub-Weibull Random Variable]
Let $\kappa > 0$. Define the Orlicz function $\psi_\kappa(x) := \exp(x^\kappa) - 1$ for $x \ge 0$. For a real-valued random variable $X$, its $\psi_\kappa$-Orlicz norm is $\norm{X}_{\psi_\kappa} := \inf \qty{ \eta > 0 : \E\qty[\psi_\kappa\qty(\tfrac{\abs{X}}{\eta})] \le 1 }$. We say $X$ is \emph{sub-Weibull of order $\kappa$}, denoted sub-Weibull($\kappa$), if $\norm{X}_{\psi_\kappa} < \infty$.
\end{dfn}
\begin{asm}
\label{asm: subweibull for phi and psi}
Let $\beta>0$ and $K_4>0$. Suppose $\psi_1$ and $\phi_j, j \in [3]$ are entrywise sub-Weibull($\beta$) with Orlicz norms bounded by $K_4$; that is, for every coordinate index $k$ and each $j \in [3]$, $\norm{\psi_{1k}}_{\psi_\beta} \le K_4$, $\norm{\tilde{\phi}_{jk}}_{\psi_\beta} \le K_4$.
\end{asm}
\begin{rem}
    The sub-Weibull assumption is not restrictive. For example, in the case of linear regression, the EIF (defined in Equation \eqref{eq: eif for linear regression}) easily satisfies this condition.
\end{rem}
Recall $\theta \in \R^d$ and $\abs{\gD_{r_1^*}^{(0)}} = n_0$.
Since the dataset is divided into two equal parts, we have
$\abs{\gD_{r_1^*,1}^{(0)}} = \abs{\gD_{r_1^*,2}^{(0)}} = \frac{n_0}{2}$.
\begin{thm}
\label{thm: excess risk upper bound}
Suppose Assumptions \ref{asm: stronger stability} - \ref{asm: subweibull for phi and psi} hold, we have: with probability at least $1 - 12 d^2 \exp\qty( - C\qty(\rho, \underline{\alpha}, \beta, K_1, K_4)\min \{ n_0^{\frac{\beta}{8}}, \sqrt{n_0} \delta^2\}) - 5 \frac{\epsilon_{n_0}}{\delta}$, 
\begin{multline*}
    \gL\qty(\hat{\alpha}^{\text{\method}}(x, w_1, w_2)) - \gL\qty(\alpha^*(x, w_1, w_2)) \leq\\
    C\qty( \underline{\alpha}, \beta, K_2, K_3) d \qty( \delta + \gR_{\frac{n_0}{2}}(\gF_2) + n_0^{\frac{1}{4}} \exp\qty(- C\qty(K_4, \beta) n_0^{\frac{\beta}{8}})),
\end{multline*}
where $C\qty(\rho, \underline{\alpha}, \beta, K_1, K_4), C\qty( \underline{\alpha}, \beta, K_2, K_3)$ and $C\qty(K_4, \beta)$ are some constants and $\gF_2$ is defined as
\begin{equation*}
\gF_2 = \qty{ \alpha(x, w_1, w_2) \middle| \,
   \sum_{j=1}^3 \alpha_j(x, w_1, w_2) = 1,
   \alpha_1(x, w_1, w_2) \geq \underline{\alpha} }.
\end{equation*}
\end{thm}
\begin{rem}
This theorem states that, with high probability, the asymptotic variance of estimator obtained through \method {(i.e. the population loss $\gL$)} will attain the minimum with certain tolerance, 
which means our algorithm can indeed obtain an optimal estimator. Given $\gF_1 \subseteq \gF_2$ and $\hat{\gF}_1 \subseteq \gF_2$, where $\hat{\gF}_1$ is the empirical constraint set defined in Equation \eqref{eq: feasible set F1 empirical}, we select the larger function set $\gF_2$ to control the class complexity. Once the Rademacher complexity $\gR_{\frac{n_0}{2}}(\gF_2)$ is $o(1)$, then the tolerance is small, with $\epsilon_{n_0}$ within certain rate, we can pick $n_0$ large enough, the tolerance will be small.
\end{rem}
\begin{rem}
    The key insight underlying our method is that the optimal allocation strategy is directly obtained by minimizing the asymptotic variance of the estimator with respect to the propensity scores $\alpha_j(x, w_1, w_2)$. Specifically, as shown in Algorithm \ref{alg: obtain alpha1}, we estimate the asymptotic variance (via the empirical loss $\hat{\gL}$ in Equation \eqref{eq: empirical loss functional}) and then solve the optimization problem in Step 5 to find the propensity scores that minimize this variance subject to the budget constraint.
\end{rem}
\begin{rem}
    Our theoretical analysis establishes an excess risk bound using tools from statistical learning theory and introduces two key technical innovations. First, we introduce a $\delta$-tolerance, a crucial step that guarantees the optimal solution lies within the empirical constraint set $\hat{\gF}_1$, allowing us to control the final risk bound. Second, to handle cases with unbounded loss functions, we employ a truncation technique, making our analysis applicable to a broader class of problems. 
\end{rem}

\subsection{Performance of the Estimator}
\label{sec: estimators performace}
Building on the previous results together with Theorem~\ref{thm: asymptotic normality}, we now characterize the performance of our proposed estimator. Our estimator performs at least as efficiently as the baseline estimator that relies solely on fully labeled data. Specifically, under the budget constraint, our estimator remains efficient  when constructed from the second batch of data.

\subsubsection{Efficiency}

If $\tilde{\phi}_j = \phi_j$, our estimator is efficient, which means that our estimator $\hat{\theta}$ attains the minimum variance when we only use the datasets $\gD_{r_j^*}, j \in [3]$ to construct the estimator, which is formally stated in the corollary below.
\begin{cor}
\label{cor: consistency get efficiency}
Suppose Assumptions \ref{asm: model distribution} - \ref{asm: consistency of M} and Assumptions \ref{asm: stronger stability} - \ref{asm: subweibull for phi and psi} hold, if $\tilde{\phi}_j = \phi_j$, then for any regular estimator $\hat{\theta}^{\text{any}}$ constructed by the unlabeled dataset $\gD_{\text{unl}}$ under the constraint \ref{eq: budget constraint MAR}, with probability at least $1 - 12 d^2 \exp\qty( - C\qty(\rho, \underline{\alpha}, \beta, K_1, K_4)\min \{ n_0^{\frac{\beta}{8}}, \sqrt{n_0} \delta^2\}) - 5 \frac{\epsilon_{n_0}}{\delta}$, we have
\begin{multline*}
    \lim_{n_1 \to \infty} \Tr(\Cov(\hat{\theta}^{\text{\method}})) - \lim_{n_1 \to \infty} \Tr(\Cov(\hat{\theta}^{\text{any}})) \leq \\
    C\qty( \underline{\alpha}, \beta, K_2, K_3) d \qty( \delta + \gR_{\frac{n_0}{2}}(\gF_2) + n_0^{\frac{1}{4}} \exp\qty(- C\qty(K_4, \beta) n_0^{\frac{\beta}{8}})),
\end{multline*}
where the high probability is taken over the dataset $\gD_{r_1^*}^\0$ since we consider the asymptotic variance using the dataset $\gD_{\text{unl}}$.
\end{cor}
\subsubsection{Robustness to Misspecification}
Even if $\tilde{\phi}_j \neq \phi_j$, we would like to show our estimator still has good properties. {This robustness guarantee is critical in practice: machine learning algorithms are often unreliable, and their predictions may exhibit systematic bias that we cannot fully characterize. By ensuring our estimator maintains good performance even when the nuisance parameter estimates $\hat{\phi}_j$ differ from their true values $\phi_j$, we provide a practical robustness guarantee that does not require us to know whether the machine learning model is accurate.} 
Specifically, we demonstrate \method's superiority over the baseline policy that uses the entire budget to only acquire true labels in the covariate-uncorrelated manner. We consider the following two estimators derived based on this baseline policy:
a label-only estimator ($\hat{\theta}^{\text{label-only}}$) using only the newly acquired labeled data, and a semi-supervised estimator ($\hat{\theta}^{\text{label-unlabel}}$) that combines both the labeled and original unlabeled data per Algorithm \ref{alg: estimator mcar}.
The latter baseline is conceptually similar to the setting in Recalibrated Prediction-Powered Inference (PPI) \citep{ji2025predictionssurrogatesrevisitingsurrogate}; however, our framework is more general because it focuses on the broader task of estimating a statistical functional.

When we purchase only true labels in the covariate-uncorrelated manner, i.e., we set $\alpha_2(x, w_1, w_2) = 0$ and $\alpha_1(x, w_1, w_2)$ is a constant. In this setting, we only observe labeled and unlabeled data. Since our method solves the corresponding optimization problem, we can guarantee that our estimator performs at least as well as in this restricted setting.
The following corollary formalizes this guarantee.
\begin{cor}
\label{cor: keep safe}
Suppose Assumptions \ref{asm: model distribution} - \ref{asm: consistency of M} and Assumptions \ref{asm: stronger stability} - \ref{asm: subweibull for phi and psi} hold, with probability at least $1 - 12 d^2 \exp\qty( - C\qty(\rho, \underline{\alpha}, \beta, K_1, K_4)\min \{ n_0^{\frac{\beta}{8}}, \sqrt{n_0} \delta^2\}) - 5 \frac{\epsilon_{n_0}}{\delta}$, we have
\begin{multline*}
    \lim_{n_1 \to \infty} \Tr(\Cov(\hat{\theta}^{\text{\method}})) - \lim_{n_1 \to \infty} \Tr(\Cov(\hat{\theta}^{\text{label-unlabel}})) \leq \\
    C\qty( \underline{\alpha}, \beta, K_2, K_3) d \qty( \delta + \gR_{\frac{n_0}{2}}(\gF_2) + n_0^{\frac{1}{4}} \exp\qty(- C\qty(K_4, \beta) n_0^{\frac{\beta}{8}})),
\end{multline*}
and
\begin{equation*}
    \lim_{n_1 \to \infty} \Tr(\Cov(\hat{\theta}^{\text{label-unlabel}}))  \leq  \lim_{n_1 \to \infty} \Tr(\Cov(\hat{\theta}^{\text{label-only}})).
\end{equation*}
\end{cor}

A critical question is whether preference data, which is often cheaper to acquire, can serve as a complete substitute for labeled data. We now demonstrate that some quantity of ground-truth labeled data is indispensable. To formalize this argument, we analyze the behavior of the estimator's asymptotic variance as the contribution of the labeled data approaches zero. For simplicity, this analysis proceeds under $\tilde{\phi}_j = \phi_j$.
\begin{cor}
\label{cor: need label data}
Suppose Assumptions \ref{asm: model distribution} - \ref{asm: consistency of M} hold. If $\tilde{\phi}_j = \phi_j$ and $\psi_1 \neq \psi_2$, then as $\alpha_1(x, w_1, w_2) \rightarrow 0$, we have $\gL\qty(\alpha(x, w_1, w_2)) \rightarrow \infty$.
\end{cor}
\begin{rem}
   Corollary \ref{cor: need label data} states that when the labeled data contains unique information ($\psi_1 \neq \psi_2$), eliminating its contribution ($\alpha_1 \rightarrow 0$) causes the estimator's asymptotic variance (as represented by the risk $\gL$) to diverge. Consequently, any approach relying solely on preference data in this setting will be inconsistent.
\end{rem}

\subsection{The covariate-agnostic case}
We next establish analogous results for the covariate-agnostic setting. For notational convenience, we write $\alpha = \qty(\alpha_1, \alpha_2, \alpha_3)^\top$.
Define the loss function $\gL(\alpha)$ as
\begin{equation}
\label{eq: loss function}
\gL(\alpha) := 
\Tr\qty( \check{\P}\qty[\frac{\psi_1 \psi_1^\top}{\alpha_1}] - \check{\P}\qty[\psi_1 \tilde{\phi}_1^\top] \, \qty(\check{\P}\qty[ \sum_{j=1}^3 \alpha_j \tilde{\phi}_j \tilde{\phi}_j^\top])^{-1} \check{\P}\qty[\phi_1 \tilde{\psi}_1^\top] ).
\end{equation}
According to the constraint \eqref{eq: budget constraint simplified} and $\alpha \geq \underline{\alpha}$, we define
\begin{equation*}
\gF_3 = \qty{ \alpha \middle| \,
   \alpha_j = \alpha_j, j \in [3],
   \alpha_1 + \alpha_2 + \alpha_3 = 1,  \rho\alpha_1 + \alpha_2 = \tau, \alpha_1 \geq \underline{\alpha} }.
\end{equation*}
Then the optimal design under covariate-agnostic case is defined as
$\alpha^* 
= \argmin_{\alpha \in \gF_3} 
\gL(\alpha)$.
Denote $\hat{\alpha}^{\text{\method-CA}}$ as the output of Algorithm~\ref{alg: obtain alpha1 mcar}.
\begin{thm}
\label{thm: mcar excess risk upper bound}
Suppose Assumptions \ref{asm: stronger stability} - \ref{asm: subweibull for phi and psi} hold, we have: with probability at least $1 - 12 d^2 \exp\qty( - C\qty(\rho, \underline{\alpha}, \beta, K_1, K_4)\min \{ n_0^{\frac{\beta}{8}}, \sqrt{n_0} \delta^2\}) - 5 \frac{\epsilon_{n_0}}{\delta}$, 
\begin{multline*}
    \gL\qty(\hat{\alpha}^{\text{\method-CA}}) - \gL\qty(\alpha^*) \leq\\
    C\qty( \underline{\alpha}, \beta, K_2, K_3) d \qty( \delta + \gR_{\frac{n_0}{2}}(\gF_3) + n_0^{\frac{1}{4}} \exp\qty(- C\qty(K_4, \beta) n_0^{\frac{\beta}{8}})),
\end{multline*}
where $C\qty(\rho, \underline{\alpha}, \beta, K_1, K_4), C\qty( \underline{\alpha}, \beta, K_2, K_3)$ and $C\qty(K_4, \beta)$ are some constants.
\end{thm}
The covariate-agnostic version of our method also ensures efficiency when $\tilde{\phi}_j = \phi_j$.
\begin{cor}
\label{cor: mcar consistency get efficiency}
Suppose Assumptions \ref{asm: model distribution} - \ref{asm: consistency of M} and Assumptions \ref{asm: stronger stability} - \ref{asm: subweibull for phi and psi} hold. If $\tilde{\phi}_j = \phi_j$, then for any regular estimator $\hat{\theta}^{\text{any}}$ constructed by the unlabeled dataset $\gD_{\text{unl}}$ under the constraint \ref{eq: budget constraint}, with probability at least $1 - 12 d^2 \exp\qty( - C\qty(\rho, \underline{\alpha}, \beta, K_1, K_4)\min \{ n_0^{\frac{\beta}{8}}, \sqrt{n_0} \delta^2\}) - 5 \frac{\epsilon_{n_0}}{\delta}$, we have
\begin{multline*}
    \lim_{n_1 \to \infty} \Tr(\Cov(\hat{\theta}^{\text{\method-CA}})) - \lim_{n_1 \to \infty} \Tr(\Cov(\hat{\theta}^{\text{any}})) \leq \\
    C\qty( \underline{\alpha}, \beta, K_2, K_3) d \qty( \delta + \gR_{\frac{n_0}{2}}(\gF_3) + n_0^{\frac{1}{4}} \exp\qty(- C\qty(K_4, \beta) n_0^{\frac{\beta}{8}})),
\end{multline*}
where the high probability is taken over the dataset $\gD_{r_1^*}^\0$ since we consider the asymptotic variance using the dataset $\gD_{\text{unl}}$.
\end{cor}

Importantly, our estimator remains at least as efficient as the baseline even when $\tilde{\phi}_j \neq \phi_j$.
\begin{cor}
\label{cor: mcar keep safe}
Suppose Assumptions \ref{asm: model distribution} - \ref{asm: consistency of M} and Assumptions \ref{asm: stronger stability} - \ref{asm: subweibull for phi and psi} hold. With probability at least $1 - 12 d^2 \exp\qty( - C\qty(\rho, \underline{\alpha}, \beta, K_1, K_4)\min \{ n_0^{\frac{\beta}{8}}, \sqrt{n_0} \delta^2\}) - 5 \frac{\epsilon_{n_0}}{\delta}$, we have
\begin{multline*}
    \lim_{n_1 \to \infty} \Tr(\Cov(\hat{\theta}^{\text{\method-CA}})) - \lim_{n_1 \to \infty} \Tr(\Cov(\hat{\theta}^{\text{label-unlabel}})) \leq \\
    C\qty( \underline{\alpha}, \beta, K_2, K_3) d \qty( \delta + \gR_{\frac{n_0}{2}}(\gF_3) + n_0^{\frac{1}{4}} \exp\qty(- C\qty(K_4, \beta) n_0^{\frac{\beta}{8}})).
\end{multline*}
\end{cor}

{We now analyze the necessity of preference data, even in the covariate-agnostic case. While one might assume preference data is redundant given sufficient labeled data, we demonstrate that under the budget constraint (Equation \ref{eq: budget constraint}), it is often essential for achieving an optimal result. Specifically, we show that the decision to incorporate preference data depends on a direct comparison of its relative cost and its relative information contribution. For this analysis, we again use the simplification $\tilde{\phi}_j = \phi_j$.
\begin{cor}
\label{cor: need preference data}
Suppose Assumptions \ref{asm: model distribution} - \ref{asm: consistency of M} hold and $\tilde{\phi}_j = \phi_j$. Under the budget constraint \ref{eq: budget constraint}, if $\rho > \frac{e_1 - e_3}{e_2 -e_3}$, then the optimal allocation requires $\alpha_2^* \neq 0$, where $e_j = \Tr\qty(\check{\P}\qty[\psi_j \psi_j^\top]), j \in [3]$.
\end{cor}
\begin{rem}
Corollary \ref{cor: need preference data} provides the formal condition for including preference data in the optimal solution. The terms $(e_1 - e_3)$ and $(e_2 - e_3)$ quantify the additional information gain from labeled data and preference data, respectively. The condition $\alpha_2^* \neq 0$ signifies that preference data is given a non-zero weight. This occurs when the price ratio $\rho$ is greater than the information-gain ratio, $\frac{e_1 - e_3}{e_2 - e_3}$. In short, preference data is incorporated whenever its `information-per-cost' is sufficiently advantageous compared to that of labeled data.
\end{rem}}

\section{Numerical Experiments}
\label{sec: experiment}
In this section, we evaluate the performance of our proposed method through simulations and a real data analysis, comparing it against two baselines.
\subsection{Linear Regression}
For M-estimation with $\theta^* = \arg\min_{\theta \in \Theta} \E[\ell_\theta(X, Y)]$ and convex loss $\ell_\theta$, the EIF takes the form $\psi_{P_{X,Y}}(X, Y) = -\qty( \E [\nabla^2 \ell_\theta(X, Y)] )^{-1} \nabla \ell_\theta(X, Y)$.

We specialize to linear regression $Y = X^\top \theta^* + \epsilon$, where $X \sim \mathcal{N}(0, \Sigma_X)$ and $\epsilon \sim \mathcal{N}(0, \sigma^2_\epsilon)$ are independent. Two AI models produce pseudo-outcomes
\begin{equation*}
    W_1 = X^\top \tilde{\theta}_1 + \eta_1\epsilon + \epsilon_1, \quad
    W_2 = X^\top \tilde{\theta}_2 + \eta_2\epsilon + \epsilon_2,
\end{equation*}
where $\tilde{\theta}_1, \tilde{\theta}_2$, $\eta_1, \eta_2$ are model-specific parameters, and $\epsilon_1, \epsilon_2$ are independent noise terms. The preference label is $V = \mathbbm{1}\qty{\norm{W_1 - Y} \leq \norm{W_2 - Y}}$.
For the least-squares loss $\theta^* = \arg\min_{\theta} \E[(Y - X^\top \theta)^2]$, the EIF reduces to
\begin{equation}
\label{eq: eif for linear regression}
    \qty(\E[XX^\top])^{-1}X(Y - X^\top\theta^*).
\end{equation}

For our simulation study, we configured the data-generating process with the following parameters. The dimension of $X$ is set to $d = 5$ and covariance matrix $\Sigma_X = I_d$. The ground-truth parameter vector is $\theta^* = (0.2,0.4,0.6,0.8,1.0)$ and we set $\sigma_\epsilon = 16$. Two AI models are generated as $\tilde{\theta}_1 = \theta^* + N(0, 0.2I_d)$ with $\tilde{\theta}_2 = \theta^* + N(0, 0.1I_d)$, $\eta_1 = \eta_2 = 0.075$, $\epsilon_1 \sim N(0, 16), \epsilon_2 \sim N(0, 36)$.
We generate $n_0 = 2000$ unlabeled observations and $n_1 = 20000$ labeled candidates.
To evaluate the system under different scenarios, we vary the price $c \in \{5, 10, 20\}$. The budget $\tau$ is adjusted in correspondence with each price level.

We focus on the inference of the first coordinate $\theta_1^*$ of the parameter $\theta$ and implement Algorithm \ref{alg: method} to assess its performance in both covariate-aware and covariate-agnostic settings. The nuisance functions $\psi_1, \phi_j, j \in [3]$ are estimated using the three-layer neural network. In the covariate-aware setting, propensity scores are estimated with a two-layer neural network. For each design, we construct 90\% confidence intervals and summarize interval lengths and empirical coverage over repeated simulations.

\begin{figure}[!ht]
  \centering
  \includegraphics[width=\linewidth]{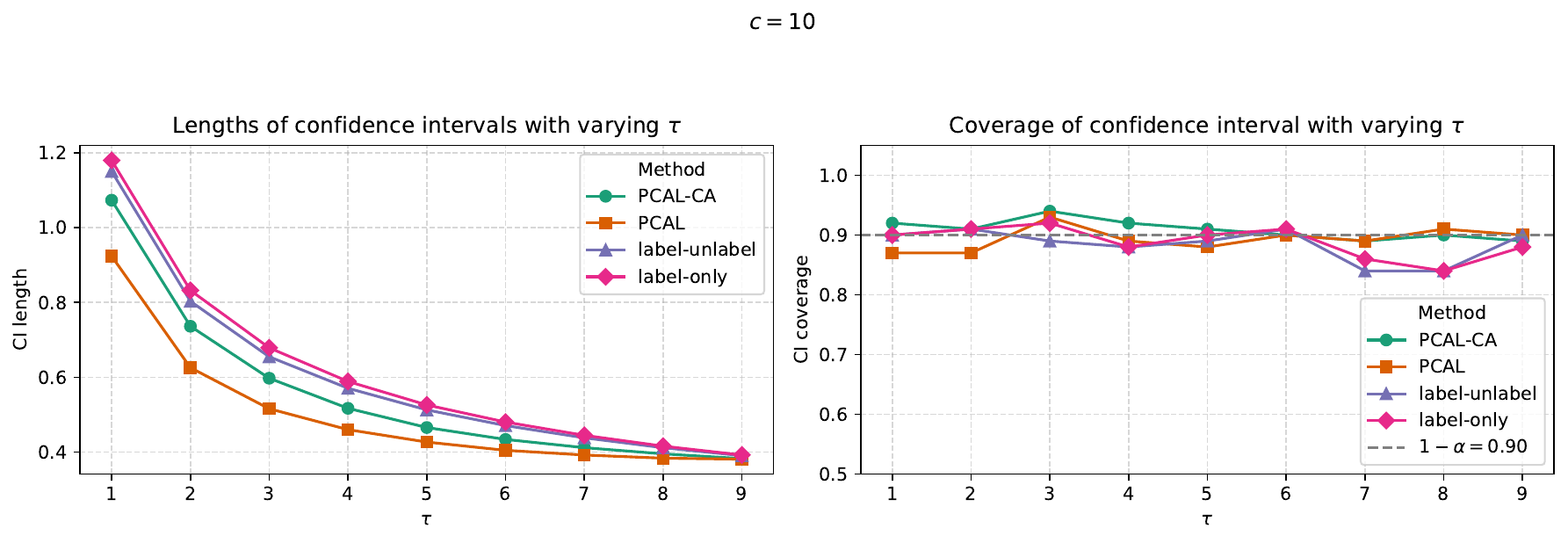}
  \caption{Confidence interval results of $\theta_1^*$ with $c=10$. Results for $c \in \{5, 20\}$ are provided in Appendix~\ref{sec: app additional simulation figures}.}
  \label{fig:lm_simulation_main}
\end{figure}

Figure~\ref{fig:lm_simulation_main} compares the performance of our proposed methods against baseline approaches across different budget levels $\tau$ for $c=10$. We evaluate four methods: (1) \method~(our covariate-aware approach), (2) \method-CA~(our covariate-agnostic variant), (3) label-unlabel (a semi-supervised baseline combining newly acquired labeled data with the original unlabeled data, conceptually similar to Prediction-Powered Inference), and (4) label-only (using only newly acquired labeled data).
The left panel displays the average confidence interval length for $\theta_1^*$. Both \method~and \method-CA~produce substantially shorter confidence intervals than the baseline methods, demonstrating that our optimal allocation strategy significantly enhances statistical efficiency. Notably, the covariate-aware \method~achieves slightly shorter intervals than \method-CA, indicating that leveraging covariate information provides additional efficiency gains.
The right panel reports the empirical coverage rates. All four methods achieve coverage close to the nominal 90\% level, confirming that the efficiency gains do not come at the expense of statistical validity. Results for other cost settings ($c \in \{5, 20\}$) show qualitatively similar patterns and are provided in Appendix~\ref{sec: app additional simulation figures}.
\subsection{Real Data Analysis: Politeness of Online Requests}
The proposed framework is also illustrated using a real dataset \citep{danescu2013computational}.
Hedges such as ``maybe'', ``I think'', ``possibly'' typically soften statements, reduce directness, or signal uncertainty. 
Motivated by \cite{ji2025predictionssurrogatesrevisitingsurrogate} and \cite{gligoric2025can}, we examine the relationship between  politeness and linguist device of hedging using our proposed \method.

Following \cite{ji2025predictionssurrogatesrevisitingsurrogate}, we used their cleaned dataset containing 5512 online requests from Stack Exchange and Wikipedia. Each request includes (i) a human-assessed politeness score (true outcome $Y$) averaged from five evaluators on a 1–25 scale, (ii) two hedging-related features (features 3 and 10) as covariates $X$, and (iii) AI-predicted politeness scores $W_1$ and $W_2$ generated by GPT-4o-mini and DeepSeeker-V3.1, respectively. The preference indicator is defined as $V = \mathbbm{1}\qty{\norm{W_1 - Y} \leq \norm{W_2 - Y}}$. For coverage evaluation, we use the estimand computed on the full labeled dataset as a proxy for $\theta^*$.

For model training, we randomly select 1000 instances for Stage 1 to compute the optimal $\alpha$ values. Subsequently, we randomly sample 2000 observations from the remaining 3000 for Stage 2. Each experiment is repeated 100 times, and we compare the performance of \method, \method-CA, label-unlabel, and label-only under varying settings of $c$ and $\tau$.

\begin{figure}[!ht]
  \centering
  \includegraphics[width=\linewidth]{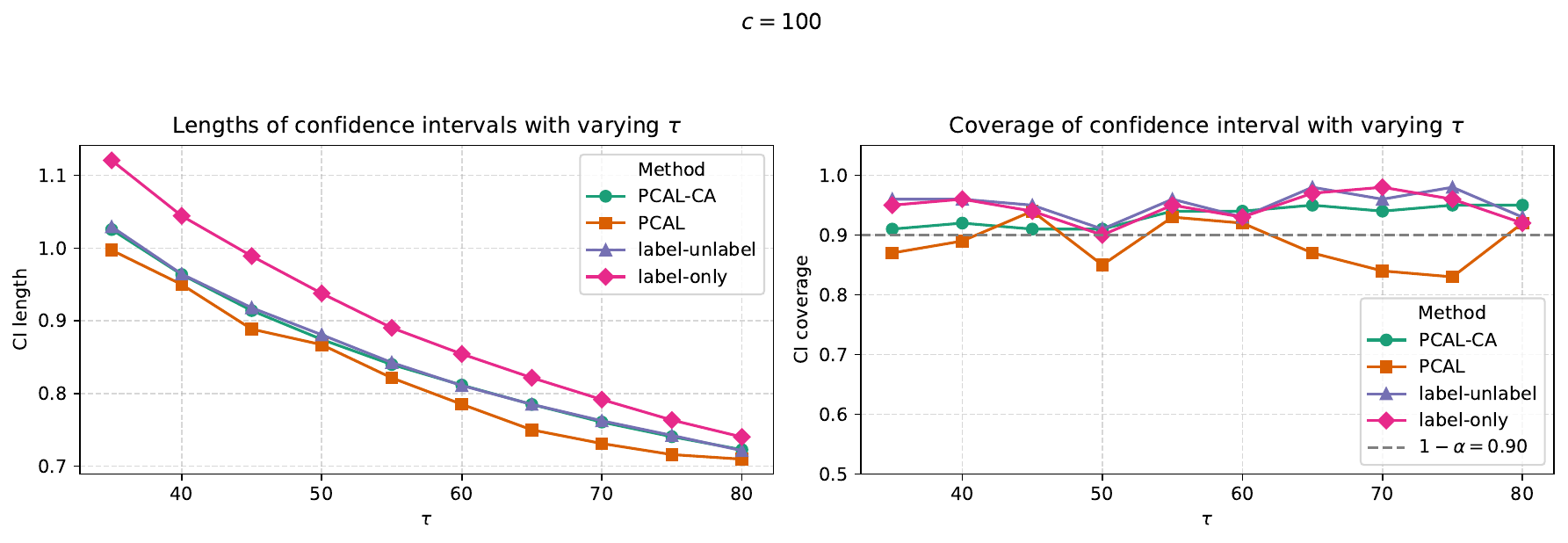}
  \caption{Confidence interval results of feature 3 with $c=100$. Results for $c \in \{20, 50, 200\}$ are provided in Appendix~\ref{sec: app additional realdata figures}.}
  \label{fig:politeness_main}
\end{figure}

Figure~\ref{fig:politeness_main} presents the results on real data for $c=100$. The left panel displays the average confidence interval length as a function of $\tau$, while the right panel reports the empirical coverage rates. The results reveal a clear efficiency ranking: \method~(covariate-aware) consistently achieves the shortest confidence intervals, followed by \method-CA~(covariate-agnostic), with both substantially outperforming the baselines label-unlabel and label-only. This demonstrates the effectiveness of our optimal, covariate-aware allocation strategy. Critically, all four methods maintain empirical coverage rates close to the nominal 90\% level, validating that the efficiency improvements do not compromise statistical validity. Results for other cost settings ($c \in \{20, 50, 200\}$) are qualitatively similar and provided in Appendix~\ref{sec: app additional realdata figures}.

\section{Conclusion and Discussion}
\label{sec: discussion}
{In this paper, we addressed the critical challenge of optimally allocating a fixed budget among multiple data types, ground-truth labels, pairwise preference labels, and pseudo-labels, under resource constraints. We formulated this problem within a semiparametric inference framework, modeling it as a monotone missing data problem under the Missing At Random (MAR) mechanism. Our proposed method, \method, provides a principled approach that simultaneously learns the optimal data acquisition strategy and constructs a statistically efficient estimator for the target functional. }

{Theoretically, we derived the semiparametric efficiency bound for this problem and proved that our \method~estimator achieves this bound under standard regularity conditions, establishing its asymptotic optimality. Crucially, we also provided a robustness guarantee: even when nuisance parameters are misspecified, our estimator is guaranteed to perform no worse than baseline methods that allocate budget exclusively to ground-truth labels. This robustness property is particularly valuable in practice, where machine learning models used to estimate nuisance functions may exhibit systematic biases. Our framework's flexibility stems from its variance-based optimization approach, which directly minimizes the asymptotic variance with respect to multi-dimensional propensity scores, eliminating the need for closed-form solutions and naturally extending to settings with more than three data types.}

While our framework provides strong theoretical guarantees, practical performance remains sensitive to the quality of nuisance function estimates ($\psi_j$ and $\phi_j$). Thus, despite our robustness to misspecification, refining these estimates is critical.

{Several promising directions for future research emerge from this work. First, our current two-stage procedure does not use the initial labeled batch in the final estimation, leading to potential information loss. A natural extension would employ adaptive sequential design, where data acquisition and estimation are interleaved: the allocation strategy is continuously updated as new data arrives, all collected data contributes to the final estimate, and decisions are made in real-time for streaming environments. This would require developing computationally efficient algorithms for on-the-fly propensity score updates and addressing theoretical challenges related to data-dependent stopping times. Second, incorporating heterogeneous and instance-dependent label costs would enhance practical applicability, as annotation costs often vary based on task difficulty, annotator expertise, and availability. Third, the generality of our framework allows us to consider extending the approach to handle preference labels beyond binary values, such as ``cannot tell'' or ``both are bad''.
In practice, preference feedback often comes from diverse sources—expert annotators, crowdworkers, automated quality metrics—each with different costs and noise characteristics. Optimally combining these heterogeneous sources while accounting for their reliability would broaden the framework's applicability.}

{In conclusion, this work establishes a rigorous statistical foundation for optimal budget allocation across heterogeneous data types. The theoretical guarantees, coupled with the framework's flexibility and practical applicability, make it a valuable tool for efficient data collection in resource-constrained settings. As AI systems increasingly rely on diverse forms of human feedback, principled methods for optimally combining different data types will become ever more critical.}

\bibliographystyle{apalike}
\bibliography{main_reference}

@misc{ji2025predictionssurrogatesrevisitingsurrogate,
      title={Predictions as Surrogates: Revisiting Surrogate Outcomes in the Age of AI}, 
      author={Wenlong Ji and Lihua Lei and Tijana Zrnic},
      year={2025},
      eprint={2501.09731},
      archivePrefix={arXiv},
      primaryClass={stat.ML},
      url={https://arxiv.org/abs/2501.09731}, 
}

@article{miao2023assumption,
  title={Assumption-lean and data-adaptive post-prediction inference},
  author={Miao, Jiacheng and Miao, Xinran and Wu, Yixuan and Zhao, Jiwei and Lu, Qiongshi},
  journal={arXiv preprint arXiv:2311.14220},
  year={2023}
}

@article{gan2024prediction,
  title={Prediction de-correlated inference: A safe approach for post-prediction inference},
  author={Gan, Feng and Liang, Wanfeng and Zou, Changliang},
  journal={Australian \& New Zealand Journal of Statistics},
  volume={66},
  number={4},
  pages={417--440},
  year={2024},
  publisher={Wiley Online Library}
}

@misc{angelopoulos2025costoptimalactiveaimodel,
      title={Cost-Optimal Active AI Model Evaluation}, 
      author={Anastasios N. Angelopoulos and Jacob Eisenstein and Jonathan Berant and Alekh Agarwal and Adam Fisch},
      year={2025},
      eprint={2506.07949},
      archivePrefix={arXiv},
      primaryClass={cs.LG},
      url={https://arxiv.org/abs/2506.07949}, 
}

@misc{ao2024predictionguidedactiveexperiments,
      title={Prediction-Guided Active Experiments}, 
      author={Ruicheng Ao and Hongyu Chen and David Simchi-Levi},
      year={2024},
      eprint={2411.12036},
      archivePrefix={arXiv},
      primaryClass={stat.ML},
      url={https://arxiv.org/abs/2411.12036}, 
}

@misc{wang2024maximinoptimalapproachsampling,
      title={A maximin optimal approach for sampling designs in two-phase studies}, 
      author={Ruoyu Wang and Qihua Wang and Wang Miao},
      year={2024},
      eprint={2312.10596},
      archivePrefix={arXiv},
      primaryClass={stat.ME},
      url={https://arxiv.org/abs/2312.10596}, 
}

@misc{li2024optimaladaptiveexperimentaldesign,
      title={Optimal Adaptive Experimental Design for Estimating Treatment Effect}, 
      author={Jiachun Li and David Simchi-Levi and Yunxiao Zhao},
      year={2024},
      eprint={2410.05552},
      archivePrefix={arXiv},
      primaryClass={stat.ML},
      url={https://arxiv.org/abs/2410.05552}, 
}

@misc{pinelis2020exactlowerupperbounds,
      title={Exact lower and upper bounds on the incomplete gamma function}, 
      author={Iosif Pinelis},
      year={2020},
      eprint={2005.06384},
      archivePrefix={arXiv},
      primaryClass={math.CA},
      url={https://arxiv.org/abs/2005.06384}, 
}

@article{robins1994estimation,
  title={Estimation of regression coefficients when some regressors are not always observed},
  author={Robins, James M and Rotnitzky, Andrea and Zhao, Lue Ping},
  journal={Journal of the American statistical Association},
  volume={89},
  number={427},
  pages={846--866},
  year={1994},
  publisher={Taylor \& Francis}
}

@article{xu2025unified,
  title={A unified framework for semiparametrically efficient semi-supervised learning},
  author={Xu, Zichun and Witten, Daniela and Shojaie, Ali},
  journal={arXiv preprint arXiv:2502.17741},
  year={2025}
}

@book{wainwright2019high,
  title={High-dimensional statistics: A non-asymptotic viewpoint},
  author={Wainwright, Martin J},
  volume={48},
  year={2019},
  publisher={Cambridge university press}
}

@article{vladimirova2020sub,
  title={Sub-Weibull distributions: Generalizing sub-Gaussian and sub-Exponential properties to heavier tailed distributions},
  author={Vladimirova, Mariia and Girard, St{\'e}phane and Nguyen, Hien and Arbel, Julyan},
  journal={Stat},
  volume={9},
  number={1},
  pages={e318},
  year={2020},
  publisher={Wiley Online Library}
}

@misc{ouyang2022traininglanguagemodelsfollow,
      title={Training language models to follow instructions with human feedback}, 
      author={Long Ouyang and Jeff Wu and Xu Jiang and Diogo Almeida and Carroll L. Wainwright and Pamela Mishkin and Chong Zhang and Sandhini Agarwal and Katarina Slama and Alex Ray and John Schulman and Jacob Hilton and Fraser Kelton and Luke Miller and Maddie Simens and Amanda Askell and Peter Welinder and Paul Christiano and Jan Leike and Ryan Lowe},
      year={2022},
      eprint={2203.02155},
      archivePrefix={arXiv},
      primaryClass={cs.CL},
      url={https://arxiv.org/abs/2203.02155}, 
}

@article{christiano2017deep,
  title={Deep reinforcement learning from human preferences},
  author={Christiano, Paul F and Leike, Jan and Brown, Tom and Martic, Miljan and Legg, Shane and Amodei, Dario},
  journal={Advances in neural information processing systems},
  volume={30},
  year={2017}
}

@article{rafailov2023direct,
  title={Direct preference optimization: Your language model is secretly a reward model},
  author={Rafailov, Rafael and Sharma, Archit and Mitchell, Eric and Manning, Christopher D and Ermon, Stefano and Finn, Chelsea},
  journal={Advances in neural information processing systems},
  volume={36},
  pages={53728--53741},
  year={2023}
}

@inproceedings{burges2005learning,
  title={Learning to rank using gradient descent},
  author={Burges, Chris and Shaked, Tal and Renshaw, Erin and Lazier, Ari and Deeds, Matt and Hamilton, Nicole and Hullender, Greg},
  booktitle={Proceedings of the 22nd international conference on Machine learning},
  pages={89--96},
  year={2005}
}

@inproceedings{lee2024methods,
  title={Methods, Applications, and Directions of Learning-to-Rank in NLP Research},
  author={Lee, Justin and Bernier-Colborne, Gabriel and Maharaj, Tegan and Vajjala, Sowmya},
  booktitle={Findings of the Association for Computational Linguistics: NAACL 2024},
  pages={1900--1917},
  year={2024}
}

@article{stiennon2020learning,
  title={Learning to summarize with human feedback},
  author={Stiennon, Nisan and Ouyang, Long and Wu, Jeffrey and Ziegler, Daniel and Lowe, Ryan and Voss, Chelsea and Radford, Alec and Amodei, Dario and Christiano, Paul F},
  journal={Advances in neural information processing systems},
  volume={33},
  pages={3008--3021},
  year={2020}
}

@article{verbeeck2023generalized,
  title={Generalized pairwise comparisons to assess treatment effects: JACC review topic of the week},
  author={Verbeeck, Johan and De Backer, Mickael and Verwerft, Jan and Salvaggio, Samuel and Valgimigli, Marco and Vranckx, Pascal and Buyse, Marc and Brunner, Edgar},
  journal={Journal of the American College of Cardiology},
  volume={82},
  number={13},
  pages={1360--1372},
  year={2023},
  publisher={American College of Cardiology Foundation Washington DC}
}

@article{angelopoulos2023prediction,
  title={Prediction-powered inference},
  author={Angelopoulos, Anastasios N and Bates, Stephen and Fannjiang, Clara and Jordan, Michael I and Zrnic, Tijana},
  journal={Science},
  volume={382},
  number={6671},
  pages={669--674},
  year={2023},
  publisher={American Association for the Advancement of Science}
}

@article{angelopoulos2023ppi++,
  title={Ppi++: Efficient prediction-powered inference},
  author={Angelopoulos, Anastasios N and Duchi, John C and Zrnic, Tijana},
  journal={arXiv preprint arXiv:2311.01453},
  year={2023}
}

@article{zrnic2024active,
  title={Active statistical inference},
  author={Zrnic, Tijana and Cand{\`e}s, Emmanuel J},
  journal={arXiv preprint arXiv:2403.03208},
  year={2024}
}

@misc{chernozhukov2018double,
  title={Double/debiased machine learning for treatment and structural parameters},
  author={Chernozhukov, Victor and Chetverikov, Denis and Demirer, Mert and Duflo, Esther and Hansen, Christian and Newey, Whitney and Robins, James},
  year={2018},
  publisher={Oxford University Press Oxford, UK}
}

@article{robins1995semiparametric,
  title={Semiparametric efficiency in multivariate regression models with missing data},
  author={Robins, James M and Rotnitzky, Andrea},
  journal={Journal of the American Statistical Association},
  volume={90},
  number={429},
  pages={122--129},
  year={1995},
  publisher={Taylor \& Francis}
}

@article{rubin1976inference,
  title={Inference and missing data},
  author={Rubin, Donald B},
  journal={Biometrika},
  volume={63},
  number={3},
  pages={581--592},
  year={1976},
  publisher={Oxford University Press}
}

@article{rubin1996multiple,
  title={Multiple imputation after 18+ years},
  author={Rubin, Donald B},
  journal={Journal of the American statistical Association},
  volume={91},
  number={434},
  pages={473--489},
  year={1996},
  publisher={Taylor \& Francis}
}

@article{lai1985asymptotically,
  title={Asymptotically efficient adaptive allocation rules},
  author={Lai, Tze Leung and Robbins, Herbert},
  journal={Advances in applied mathematics},
  volume={6},
  number={1},
  pages={4--22},
  year={1985},
  publisher={Academic Press}
}

@article{hahn2011adaptive,
  title={Adaptive experimental design using the propensity score},
  author={Hahn, Jinyong and Hirano, Keisuke and Karlan, Dean},
  journal={Journal of Business \& Economic Statistics},
  volume={29},
  number={1},
  pages={96--108},
  year={2011},
  publisher={Taylor \& Francis}
}

@article{ash2019deep,
  title={Deep batch active learning by diverse, uncertain gradient lower bounds},
  author={Ash, Jordan T and Zhang, Chicheng and Krishnamurthy, Akshay and Langford, John and Agarwal, Alekh},
  journal={arXiv preprint arXiv:1906.03671},
  year={2019}
}

@article{settles2009active,
  title={Active learning literature survey},
  author={Settles, Burr},
  year={2009},
  publisher={University of Wisconsin-Madison Department of Computer Sciences}
}

@inproceedings{gal2017deep,
  title={Deep bayesian active learning with image data},
  author={Gal, Yarin and Islam, Riashat and Ghahramani, Zoubin},
  booktitle={International conference on machine learning},
  pages={1183--1192},
  year={2017},
  organization={PMLR}
}

@inproceedings{bender2021dangers,
  title={On the dangers of stochastic parrots: Can language models be too big?},
  author={Bender, Emily M and Gebru, Timnit and McMillan-Major, Angelina and Shmitchell, Shmargaret},
  booktitle={Proceedings of the 2021 ACM conference on fairness, accountability, and transparency},
  pages={610--623},
  year={2021}
}

@article{gehman2020realtoxicityprompts,
  title={Realtoxicityprompts: Evaluating neural toxic degeneration in language models},
  author={Gehman, Samuel and Gururangan, Suchin and Sap, Maarten and Choi, Yejin and Smith, Noah A},
  journal={arXiv preprint arXiv:2009.11462},
  year={2020}
}

@article{leike2018scalable,
  title={Scalable agent alignment via reward modeling: a research direction},
  author={Leike, Jan and Krueger, David and Everitt, Tom and Martic, Miljan and Maini, Vishal and Legg, Shane},
  journal={arXiv preprint arXiv:1811.07871},
  year={2018}
}

@article{kenton2021alignment,
  title={Alignment of language agents},
  author={Kenton, Zachary and Everitt, Tom and Weidinger, Laura and Gabriel, Iason and Mikulik, Vladimir and Irving, Geoffrey},
  journal={arXiv preprint arXiv:2103.14659},
  year={2021}
}

@book{van2000asymptotic,
  title={Asymptotic statistics},
  author={Van der Vaart, Aad W},
  volume={3},
  year={2000},
  publisher={Cambridge university press}
}

@article{hong2024orpo,
  title={Orpo: Monolithic preference optimization without reference model},
  author={Hong, Jiwoo and Lee, Noah and Thorne, James},
  journal={arXiv preprint arXiv:2403.07691},
  year={2024}
}

@inproceedings{azar2024general,
  title={A general theoretical paradigm to understand learning from human preferences},
  author={Azar, Mohammad Gheshlaghi and Guo, Zhaohan Daniel and Piot, Bilal and Munos, Remi and Rowland, Mark and Valko, Michal and Calandriello, Daniele},
  booktitle={International Conference on Artificial Intelligence and Statistics},
  pages={4447--4455},
  year={2024},
  organization={PMLR}
}

@article{ethayarajh2024kto,
  title={Kto: Model alignment as prospect theoretic optimization},
  author={Ethayarajh, Kawin and Xu, Winnie and Muennighoff, Niklas and Jurafsky, Dan and Kiela, Douwe},
  journal={arXiv preprint arXiv:2402.01306},
  year={2024}
}

@article{meng2024simpo,
  title={Simpo: Simple preference optimization with a reference-free reward},
  author={Meng, Yu and Xia, Mengzhou and Chen, Danqi},
  journal={Advances in Neural Information Processing Systems},
  volume={37},
  pages={124198--124235},
  year={2024}
}

@article{zajac2023ground,
  title={Ground Truth Or Dare: Factors Affecting The Creation Of Medical Datasets For Training AI},
  author={Zajac, Hubert Dariusz and Avlona, Natalia Rozalia and Andersen, Tariq O and Kensing, Finn and Shklovski, Irina},
  journal={CoRR},
  year={2023}
}

@article{zhang2025efficient,
  title={Efficient Statistical Estimation for Sequential Adaptive Experiments with Implications for Adaptive Designs},
  author={Zhang, Wenxin and van der Laan, Mark},
  journal={arXiv preprint arXiv:2508.09135},
  year={2025}
}

@article{imberg2025active,
  title={Active sampling: A machine-learning-assisted framework for finite population inference with optimal subsamples},
  author={Imberg, Henrik and Yang, Xiaomi and Flannagan, Carol and B{\"a}rgman, Jonas},
  journal={Technometrics},
  volume={67},
  number={1},
  pages={46--57},
  year={2025},
  publisher={Taylor \& Francis}
}

@inproceedings{gligoric2025can,
  title={Can unconfident llm annotations be used for confident conclusions?},
  author={Gligori{\'c}, Kristina and Zrnic, Tijana and Lee, Cinoo and Candes, Emmanuel and Jurafsky, Dan},
  booktitle={Proceedings of the 2025 Conference of the Nations of the Americas Chapter of the Association for Computational Linguistics: Human Language Technologies (Volume 1: Long Papers)},
  pages={3514--3533},
  year={2025}
}

@inproceedings{danescu2013computational,
  title={A Computational Approach to Politeness with Application to Social Factors},
  author={Danescu-Niculescu-Mizil, Cristian and Sudhof, Moritz and Jurafsky, Dan and Leskovec, Jure and Potts, Christopher},
  booktitle={51st Annual Meeting of the Association for Computational Linguistics},
  pages={250--259},
  year={2013},
  organization={ACL}
}

@inproceedings{ilvento2019metric,
  title={Metric learning for individual fairness},
  author={Ilvento, Christina},
  booktitle={1st Symposium on Foundations of Responsible Computing (FORC 2020)},
  year={2020},
  organization={Schloss Dagstuhl-Leibniz-Zentrum f{\"u}r Informatik}
}

@inproceedings{narasimhan2020pairwise,
  title={Pairwise fairness for ranking and regression},
  author={Narasimhan, Harikrishna and Cotter, Andrew and Gupta, Maya R and Wang, Serena},
  booktitle={Proceedings of the AAAI Conference on Artificial Intelligence},
  volume={34},
  number={04},
  pages={5248--5255},
  year={2020}
}

@article{morato2025conbatch,
  title={ConBatch-BAL: Batch Bayesian Active Learning under Budget Constraints},
  author={Morato, Pablo G and Andriotis, Charalampos P and Khademi, Seyran},
  journal={arXiv preprint arXiv:2507.04929},
  year={2025}
}

@article{goebel2025budgeted,
  title={Budgeted Online Active Learning with Expert Advice and Episodic Priors},
  author={Goebel, Kristen and Solow, William and Pesantez-Cabrera, Paola and Keller, Markus and Fern, Alan},
  journal={arXiv preprint arXiv:2506.03307},
  year={2025}
}

@article{li2025robust,
  title={Robust Sampling for Active Statistical Inference},
  author={Li, Puheng and Zrnic, Tijana and Cand{\`e}s, Emmanuel},
  journal={arXiv preprint arXiv:2511.08991},
  year={2025}
}

@article{gorwa2020algorithmic,
  title={Algorithmic content moderation: Technical and political challenges in the automation of platform governance},
  author={Gorwa, Robert and Binns, Reuben and Katzenbach, Christian},
  journal={Big Data \& Society},
  volume={7},
  number={1},
  year={2020},
  publisher={SAGE Publications}
}

@article{arxiv2508.05527,
  title={AI vs. Human Moderators: A Comparative Evaluation of Multimodal {LLM}s in Content Moderation},
  author={Doe, John and Smith, Jane}, 
  journal={arXiv preprint arXiv:2508.05527},
  year={2025},
  note={Accepted to ICCV 2025 Workshop on Computer Vision in Advertising and Marketing}
}

@techreport{stanford2024llm,
  title={Large Language Models and the Economics of Content Moderation},
  author={{Stanford Cyber Policy Center}},
  institution={Stanford University},
  year={2024},
  type={Tech Report}
}

@inproceedings{DanescuNiculescuMizil2013,
  author    = {Danescu-Niculescu-Mizil, Cristian and Sudhof, Moritz and Jurafsky, Dan and Leskovec, Jure and Potts, Christopher},
  title     = {A computational approach to politeness with application to social factors},
  booktitle = {Proceedings of the 51st Annual Meeting of the Association for Computational Linguistics (Volume 1: Long Papers)},
  month     = {August},
  year      = {2013},
  address   = {Sofia, Bulgaria},
  publisher = {Association for Computational Linguistics},
  pages     = {250--259},
  url       = {https://aclanthology.org/P13-1025}
}
\newpage
\appendix

\clearpage
\begin{center}
  {\Large \bfseries Supplementary Appendix}
\end{center}
\vspace{1em}

\baselineskip=22pt

\begin{abstract}

\smallskip
\noindent
In this Appendix, Section \ref{sec:appendix-lemmas} collects auxiliary lemmas that facilitate our proofs. Sections \ref{sec: app proof covariate aware} and \ref{sec: app proof covariate agnostic} provide the proofs for the covariate-aware and covariate-agnostic cases, respectively. Section \ref{sec: app additional related work} presents additional related work, and Section \ref{sec: app additional figures} provides supplementary experimental results.

\end{abstract}
\section{Auxiliary Lemmas}
\label{sec:appendix-lemmas}
\begin{lem}[Proposition 2.7 in \citet{pinelis2020exactlowerupperbounds}]
\label{lem: incomplete gamma function upper bound}
Define
\begin{equation*}
\Gamma(s, x) = \int_{x}^{\infty} t^{s-1} e^{-t} \, dt.
\end{equation*}
Take any real $s \geq 1$. Then $\Gamma(s, x) \geq x^{s-1} e^{-x}$ for all real $x > 0$, and $\Gamma(s, x) \leq \dfrac{x^{s-1} e^{-x}}{1 - (s - 1)/x}$ for all real $x > s - 1$.
\end{lem}

\begin{lem}[Theorem 4.10 in \citet{wainwright2019high}]
\label{lem: wainwright2019high uniform bound}
For any $b$-uniformly bounded class of functions $\gF$, any positive integer $n \geq 1$, 
and any scalar $\delta \geq 0$, we have
\begin{equation*}
\P\qty(\operatorname*{sup}_{f \in \gF}\abs{ \qty(\P_n - \P)[f] } \leq 2 \gR_n(\gF) + \delta) \geq 1 - \exp\qty(-\frac{n \delta^2}{2 b^2}).
\end{equation*}
A function class $\gF$ is said to be \emph{$b$-uniformly bounded} if
$\norm{f}_{\infty} \leq b, \forall f \in \gF$.
\end{lem}

\begin{lem}[Corollary 3.1 in \citet{vladimirova2020sub}]
\label{lem: vladimirova2020sub sum of iid subweibull}
Let $X_1, \ldots, X_n$ be identically distributed sub-Weibull($\theta$) random variables. Then, for all $x \geq n K_\theta$, we have
\begin{equation*}
\P\qty( \abs{\sum_{i=1}^n X_i} \geq x )
   \leq \exp\qty( -\qty(\frac{x}{nK_\theta})^\theta ),
\end{equation*}
for some constant $K_\theta$ dependent on $\theta$.
\end{lem}

\begin{lem}[Theorem 25.20 (Convolution) in \citet{van2000asymptotic}]
\label{lem: van2000asymptotic convolution}
Let the function $\theta : \gP \mapsto \R^k$ be differentiable at $P$ relative to the tangent cone $\dot{\gP}_P$ with efficient influence function $\tilde{\theta}_P$. Then the asymptotic covariance matrix of every regular sequence of estimators is bounded below by $\check{\P}\qty[ \tilde{\theta}_P \tilde{\theta}_P^{\top}]$. Furthermore, if $\dot{\gP}_P$ is a convex cone, then every limit distribution $L$ of a regular sequence of estimators can be written $L = \mathcal{N}(0, \check{\P}\qty[ \tilde{\theta}_P \tilde{\theta}_P^{\top}]) * M$ for some probability distribution $M$.
\end{lem}
\begin{lem}
\label{lem: converge to 0 in prob}
Let $\{X_n\}_{n\ge1}$ be real–valued random variables and
$\{Y_n\}_{n\ge1}$ be $\R^d$-valued random vectors ($d\ge1$), endowed with the $\ell_2$ norm
$\lVert\cdot\rVert_2$.
Assume
\begin{equation*}
X_n \xrightarrow{P} 0,
\qquad
\text{$\{Y_n\}$ is \emph{tight}, i.e. }
\forall \eta>0\;\exists M<\infty:
      \sup_{n\ge1}\P\bigl(\lVert Y_n\rVert_2>M\bigr)<\eta .
\end{equation*}
Then
\begin{equation*}
X_nY_n \xrightarrow{P} 0 \in \R^d ,
\end{equation*}
that is,
\(
\P\bigl(\lVert X_nY_n\rVert_2>\epsilon\bigr)\to0
\)
for every $\epsilon>0$.
\end{lem}
\begin{proof}
Fix $\epsilon>0$ and pick any $\delta\in(0,1)$.

\paragraph{Tightness of $Y_n$.}
By tightness, there exists $M=M(\delta)$ such that
\begin{equation*}
\sup_{n\ge1}\P\bigl(\lVert Y_n\rVert_2>M\bigr)\le\frac{\delta}{2}.
\end{equation*}

\paragraph{Smallness of $X_n$.}
Since $X_n\xrightarrow{P}0$, there is $N=N(\epsilon,\delta)$ satisfying
\begin{equation*}
\P\bigl(|X_n|>\epsilon/M\bigr)<\frac{\delta}{2}
\quad\text{for all }n\ge N .
\end{equation*}

\paragraph{Combine.}
For every $n\ge N$,
\begin{equation*}
\begin{aligned}
\P\bigl(\lVert X_nY_n\rVert_2>\epsilon\bigr)
&\le
  \P\bigl(\lVert Y_n\rVert_2>M\bigr)
  +\P\bigl(|X_n|>\epsilon/M\bigr)\\
&\le
  \frac{\delta}{2}+\frac{\delta}{2}= \delta .
\end{aligned}
\end{equation*}
Because $\delta>0$ is arbitrary, we have
\(
\P\bigl(\lVert X_nY_n\rVert_2>\epsilon\bigr)\to0
\)
as $n\to\infty$ for every $\epsilon>0$, hence
$X_nY_n\xrightarrow{P}0$ in $\R^d$.
\end{proof}

\begin{lem}
\label{lem: ABI-op-relaxed}
Let $A\in\R^{d_1\times d_2}$ and $B\in\R^{d_2\times d_2}$ be symmetric positive definite.
Let $\hat{A},\hat{B}$ be matrices of the same sizes and define $\Delta_A:=\hat{A}-A$ and $\Delta_B:=\hat{B}-B$.
Then
\begin{multline*}
\norm{\hat{A}\,\hat{B}^{-1}\hat{A}-A B^{-1}A}_{\mathrm{sp}}
\leq
\frac{\norm{A}_{\mathrm{sp}}+\norm{\Delta_A}_{\mathrm{sp}}}{\sigma_{\min}\qty(\hat{B})}\,\norm{\Delta_A}_{\mathrm{sp}}
+\\ \frac{\norm{A}_{\mathrm{sp}}\qty(\norm{A}_{\mathrm{sp}}+\norm{\Delta_A}_{\mathrm{sp}})}{\sigma_{\min}\qty(B)\,\sigma_{\min}\qty(\hat{B})}\,\norm{\Delta_B}_{\mathrm{sp}}
+ \frac{\norm{A}_{\mathrm{sp}}}{\sigma_{\min}\qty(B)}\,\norm{\Delta_A}_{\mathrm{sp}}.
\end{multline*}
\end{lem}

\begin{proof}
Use the standard three-term decomposition
\begin{equation*}
\hat{A}\,\hat{B}^{-1}\hat{A}-A B^{-1}A
= \underbrace{\Delta_A\,\hat{B}^{-1}\hat{A}}_{T_1}
+ \underbrace{A\,\qty(\hat{B}^{-1}-B^{-1})\,\hat{A}}_{T_2}
+ \underbrace{A B^{-1}\Delta_A}_{T_3}.
\end{equation*}
By submultiplicativity of the spectral norm,
\begin{align*}
\norm{T_1}_{\mathrm{sp}} &\le \norm{\Delta_A}_{\mathrm{sp}}\,\norm{\hat{B}^{-1}}_{\mathrm{sp}}\,\norm{\hat{A}}_{\mathrm{sp}},\\
\norm{T_2}_{\mathrm{sp}} &\le \norm{A}_{\mathrm{sp}}\,\norm{\hat{B}^{-1}-B^{-1}}_{\mathrm{sp}}\,\norm{\hat{A}}_{\mathrm{sp}},\\
\norm{T_3}_{\mathrm{sp}} &\le \norm{A}_{\mathrm{sp}}\,\norm{B^{-1}}_{\mathrm{sp}}\,\norm{\Delta_A}_{\mathrm{sp}}.
\end{align*}

\textbf{Inverse bounds.}
\begin{equation*}
\norm{\hat{B}^{-1}}_{\mathrm{sp}}=\frac{1}{\sigma_{\min}\qty(\hat{B})},
\qquad
\norm{B^{-1}}_{\mathrm{sp}}=\frac{1}{\sigma_{\min}\qty(B)}.
\end{equation*}

\textbf{Resolvent difference.}
Using $\hat{B}^{-1}-B^{-1}=-\,B^{-1}\Delta_B\hat{B}^{-1}$, we obtain
\begin{equation*}
\norm{\hat{B}^{-1}-B^{-1}}_{\mathrm{sp}}
\le \norm{B^{-1}}_{\mathrm{sp}}\,\norm{\Delta_B}_{\mathrm{sp}}\,\norm{\hat{B}^{-1}}_{\mathrm{sp}}
= \frac{1}{\sigma_{\min}\qty(B)\,\sigma_{\min}\qty(\hat{B})}\,\norm{\Delta_B}_{\mathrm{sp}}.
\end{equation*}

\textbf{Assembling the pieces.}
Since $\norm{\hat{A}}_{\mathrm{sp}} \le \norm{A}_{\mathrm{sp}}+\norm{\Delta_A}_{\mathrm{sp}}$, the three terms satisfy
\begin{align*}
\norm{T_1}_{\mathrm{sp}}
&\le \frac{1}{\sigma_{\min}\qty(\hat{B})}\,\norm{\Delta_A}_{\mathrm{sp}}\qty(\norm{A}_{\mathrm{sp}}+\norm{\Delta_A}_{\mathrm{sp}}),\\
\norm{T_2}_{\mathrm{sp}}
&\le \frac{1}{\sigma_{\min}\qty(B)\,\sigma_{\min}\qty(\hat{B})}\,\norm{A}_{\mathrm{sp}}\,\norm{\Delta_B}_{\mathrm{sp}}\qty(\norm{A}_{\mathrm{sp}}+\norm{\Delta_A}_{\mathrm{sp}}),\\
\norm{T_3}_{\mathrm{sp}}
&= \frac{1}{\sigma_{\min}\qty(B)}\,\norm{A}_{\mathrm{sp}}\,\norm{\Delta_A}_{\mathrm{sp}}.
\end{align*}
Summing these bounds yields the displayed inequality.
\end{proof}

\section{Proofs for Covariate-Aware Case}
\label{sec: app proof covariate aware}
\subsection{Proof of Proposition~\ref{prop: EIF MCAR monotone}}
\paragraph{Model Space}

The data: $((X, Y, W_1, W_2, V)^R, R)$.
The distribution:
\begin{multline*}
    dP_{(X, Y, W_1, W_2, V)^R, R}((x, y, w_1, w_2, v)^r, r) = \\\sum_{j = 1}^3 \mathbbm{1}(r = r_j^*) \alpha_j(x, w_1, w_2) dP_{(X, Y, W_1, W_2, V)^{r_j^*}}((x, y, w_1, w_2, v)^{r_j^*}),
\end{multline*}
where $dP_{X, Y, W_1, W_2, V}$ will be $dP_{V|X, Y, W_1, W_2}dP_{X, Y, W_1, W_2}$.

\paragraph{Tangent Space} The space formed by score functions.
Notice that $V$ is deterministic, the score function of the complete data will be
\[
S_{X, Y, W_1, W_2, V}(X, Y, W_1, W_2, V) 
= S_{X, Y, W_1, W_2}(X, Y, W_1, W_2).
\]
Thus the score of the original distribution will take the form:
\begin{multline}
\label{eq: score form}
S_{(X, Y, W_1, W_2, V)^R, R}((X, Y, W_1, W_2, V)^R, R) \\= \sum_{j =1}^3\mathbbm{1}\{R = r_j^*\}
\E\qty[S_{X, Y, W_1, W_2}(X, Y, W_1, W_2) \mid (X, Y, W_1, W_2, V)^{r_j^*}].
\end{multline}

\paragraph{EIF}
After determining the tangent space, we can verify that the EIF is given by \eqref{eq: EIF MCAR monotone}. EIF is projection of IF in the tanget space. Therefore, the proof consists of two steps.

\textbf{Step 1: It is an influence function.}
Let $P_{X, Y}^*$ denote the true distribution.
Given $ \psi_{P_{X, Y}}(X, Y)$ is the EIF of $\theta(P_{X, Y})$ in $\gP_{X, Y}$ at $P_{X, Y}^*$, we have
\begin{equation*}
    \frac{d \theta(P_{X,Y}^t)}{dt} \bigg|_{t=0} = \E \qty[ \psi_{P_{X,Y}}(X, Y) S_{X,Y}(X, Y)].
\end{equation*}
Then
\begin{equation*}
\begin{aligned}
\frac{d \theta(P_{X,Y}^t)}{dt} \bigg|_{t=0} &=\E \qty[  \frac{\mathbbm{1}\{R = r_1^*\}}{\alpha_1(X, W_1, W_2)} \psi_{P_{X, Y}}(X, Y) S_{X, Y, W_1, W_2}(X, Y, W_1, W_2)]\\
&= \E \qty[  \frac{\mathbbm{1}\{R = r_1^*\}}{\alpha_1(X, W_1, W_2)} \psi_{P_{X, Y}}(X, Y)S_{(X, Y, W_1, W_2, V)^R, R}((X, Y, W_1, W_2, V)^R, R)]
\end{aligned}
\end{equation*}
where the second equality holds from the score equation form \eqref{eq: score form}.
Now we only need to show
\begin{equation*}
\E \qty[\qty(\sum_{j =1}^3 \mathbbm{1}\{R = r_j^*\} \phi_i) S_{(X, Y, W_1, W_2, V)^R, R}((X, Y, W_1, W_2, V)^R, R)] = 0.
\end{equation*}
We have:

\begin{align}
&\E \qty[\qty(\sum_{j =1}^3 \mathbbm{1}\{R = r_j^*\} \phi_i) S_{(X, Y, W_1, W_2, V)^R, R}((X, Y, W_1, W_2, V)^R, R)] \\
= &\sum_{j = 1}^3 \E\qty[\alpha_j(X, W_1, W_2)\phi_j \E\qty[S_{X, Y, W_1, W_2}(X, Y, W_1, W_2) \mid (X, Y, W_1, W_2, V)^{r_j^*}]]\\
= &\E\qty[\sum_{j = 1}^3\alpha_j(X, W_1, W_2)\phi_jS_{X, Y, W_1, W_2}(X, Y, W_1, W_2)]\label{eq: phi times score}
\end{align}
Notice that
\begin{equation*}
    \sum_{j = 1}^3\alpha_j(X, W_1, W_2)\phi_j= 0.
\end{equation*}
Plug in to Equation \eqref{eq: phi times score}, it will be equal to $0$.

\textbf{Step 2: It is in the tangent space.}
We directly construct the score as
\begin{equation*}
     S_{\psi}(X, Y, W_1, W_2) = \frac{1}{\alpha_1} \psi_1 
- \frac{\alpha_2}{(\alpha_1 + \alpha_2)\alpha_1} 
\psi_2  
- \frac{\alpha_3}{\alpha_1 + \alpha_2} 
\psi_3
\end{equation*}
It is sufficient to check that $\psi_1$ belongs to the tangent space because we do not impose any assumption on $P_{X,Y}$.
To verify that $\psi_1$ belongs to the tangent space of the complete data, recall that $\psi_1(X, Y, W_1, W_2, V) = \E[\psi_{P_{X,Y}}(X, Y) \mid (X, Y, W_1, W_2, V)^{r_1^*}]$ by definition. Since $(X, Y, W_1, W_2, V)^{r_1^*}$ represents the complete data pattern where all variables are observed, and $V$ is deterministic, $\psi_1$ is a function of $(X, Y, W_1, W_2)$ only. Therefore, $\psi_1$ can be written as $S_{X, Y, W_1, W_2}(X, Y, W_1, W_2)$ for an appropriate score function, which confirms that it belongs to the tangent space of the complete data as defined in \eqref{eq: score form}.

Next, we verify that
\begin{equation*}
    \begin{aligned}
        \E\qty[S_{\psi}(X, Y, W_1, W_2) \mid (X, Y, W_1, W_2, V)^{r_2^*}] = \phi_2\\
        \E\qty[S_{\psi}(X, Y, W_1, W_2) \mid (X, Y, W_1, W_2, V)^{r_3^*}] = \phi_3
    \end{aligned}
\end{equation*}
It is indeed in the tangent space. We complete our proof.
\subsection{Proof of Theorem~\ref{thm: asymptotic normality}}
\begin{proof}
\begin{equation*}
\sqrt{n_1} (\hat{\theta}^{\text{\method}} - \theta) 
= \sum_{k=1}^3 \frac{|\gD_k|}{\sqrt{n_1}} \hat{\theta}^\k -\sqrt{n_1} \theta 
= \sum_{k=1}^3 \frac{|\gD_k|}{\sqrt{n_1}} (\hat{\theta}^\k - \theta).
\end{equation*}
We consider $k = 1$ for the analysis; the remaining terms can be handled analogously.
\begin{multline}
\label{eq: normality first batch}
\frac{|\gD_1|}{\sqrt{n_1}} \qty( \hat{\theta}^{(1)} - \theta ) = \\
\underbrace{
\frac{|\gD_1|}{\sqrt{n_1}} \qty( \tilde{\theta}^{(1)} - \theta )
}_{I_1}
+
\underbrace{
\frac{1}{\sqrt{n_1}} \widehat{M}^{(1)} 
\qty[
\sum_{i \in \gD_1} \sum_{j=1}^{3} \mathbbm{1}\{ r_i = r_j^* \}
\hat{\phi}_j^{(1)}\qty( (x_i, y_i, w_{1i}, w_{2i}, v_i)^{r_j^*} )
]
}_{I_2}.
\end{multline}
For the first term ($I_1$), since we use a weighted estimator and $\tilde{\theta}^\1$ is an efficient weighted estimator, we have
\begin{equation*}
\tilde{\theta}^\1 - \theta 
= \frac{1}
       {|\gD_1 \cap \gD_{r_1^*}|} \sum\limits_{i \in \gD_1 \cap \gD_{r_1^*}} \frac{\P(R = r_1^*)}{\alpha_1(x_i, w_{1i}, w_{2i})} \psi_1(x_i, y_i, w_{1i}, w_{2i}, v_i) 
+  o_p\qty( \frac{1}{\sqrt{|\gD_1 \cap \gD_{r_1^*}|}} ).
\end{equation*}
Then
\begin{multline*}
\frac{|\gD_1|}{\sqrt{n_1}} \qty( \tilde{\theta}^\1 - \theta )
= \\\sum_{i \in \gD_1} \frac{ \mathbbm{1}\{ r_i = r_1^* \} \psi_1(x_i, y_i, w_{1i}, w_{2i}, v_i) }{ \alpha_1(x_i, w_{1i}, w_{2i}) \sqrt{n_1} }
\frac{|\gD_1| \P(R = r_1^*)}{|\gD_1 \cap \gD_{r_1^*}|}
+ o_p\qty( \frac{|\gD_1|}{\sqrt{n_1 |\gD_1 \cap \gD_{r_1^*}|}} ).
\end{multline*}
Notice that
\begin{multline*}
 \sup_{n_1\geq 1 } \frac{1}{n_1} \E \qty[ 
  \left\lVert 
    \sum_{i \in \gD_1} 
    \frac{ \mathbbm{1}\{ R_i = r_1^* \} \, \psi_1(X_i, Y_i, W_{1i}, W_{2i}, V) }{ \alpha_1(X_i, W_{1i}, W_{2i}) } 
  \right\rVert_2^2 
] \leq \\ \E\qty[ 
  \left\lVert
    \frac{ \mathbbm{1}\{ R= r_1^* \} \, \psi_1(X, Y, W_1, W_2, V) }{ \alpha_1(X, W_1, W_2) } 
  \right\rVert_2^2 
]   \leq C,
\end{multline*}
where $C$ is some constant. Hence the sequence
\[
  Y_n
  :=
  \frac{1}{\sqrt{n_1}}
  \sum_{i\in\gD_1}
  \frac{\mathbbm 1\{R_i=r_1^*\}\,
        \psi_1(X_i,Y_i,W_{1i},W_{2i},V)}
       {\alpha_1(X_i,W_{1i},W_{2i})}
\]
is tight.
Moreover,
\begin{equation*}
  X_n
  :=
  \frac{|\gD_1|\,\P(R=r_1^*)}{|\gD_1\cap\gD_{r_1^*}|}-1
  \;\xrightarrow{P}\;0 .
\end{equation*}
We directly apply Lemma~\ref{lem: converge to 0 in prob} and get
\begin{equation}
\label{eq: normality 1 part}
    \frac{|\gD_1|}{\sqrt{n_1}} \qty( \tilde{\theta}^\1 - \theta )
= \frac{1}{\sqrt{n_1}}\sum_{i \in \gD_1} \frac{ \mathbbm{1}\{ r_i = r_1^* \} \psi_1 (x_i, y_i, w_{1i}, w_{2i}, v_i)}{ \alpha_1(x_i, w_{1i}, w_{2i})} + o_p(1).
\end{equation}
For the second term ($I_2$), we only need to show
\begin{multline}
\label{eq: normality 2 part}
    \Bigg\|\frac{1}{\sqrt{n_1}} \hat{M}^\1 \sum_{i \in \gD_1} \sum_{j = 1}^3 \mathbbm{1}\{ R_i = r_j^* \} \hat{\phi}_j^\1\qty((X_i,Y_i,W_{1i},W_{2i},V)^{r_j^*})
-\\ \frac{1}{\sqrt{n_1}}M \sum_{i \in \gD_1} \sum_{j = 1}^3 \mathbbm{1}\{ R_i = r_j^* \} \tilde{\phi}_j\qty((X_i,Y_i,W_{1i},W_{2i},V)^{r_j^*}) \Bigg\|_2^2
\xrightarrow{P} 0
\end{multline}
We split it into two terms to analyze,
\begin{multline*}
 \frac{1}{\sqrt{n_1}} \hat{M}^\1 \sum_{i \in \gD_1} \sum_{j=1}^3 \mathbbm{1}\{ R_i = r_j^* \} \hat{\phi}_{ij}^\1 
- \frac{1}{\sqrt{n_1}} M \sum_{i \in \gD_1} \sum_{j=1}^3 \mathbbm{1}\{ R_i = r_j^* \} \tilde{\phi}_{ij} =\\
\frac{1}{\sqrt{n_1}}(\hat{M}^\1 - M) \sum_{i \in \gD_1} \sum_{j=1}^3 \mathbbm{1}\{ R_i = r_j^* \} \hat{\phi}_{ij}^\1
+ \frac{1}{\sqrt{n_1}}M \sum_{i \in \gD_1} \sum_{j=1}^3 \mathbbm{1}\{ R_i = r_j^* \} \qty( \hat{\phi}_{ij}^\1 - \tilde{\phi}_{ij} ),
\end{multline*}
where we abbreviate $\hat{\phi}_j^\1\qty((X_i,Y_i,W_{1i},W_{2i},V)^{r_j^*}), \tilde{\phi}_j\qty((X_i,Y_i,W_{1i},W_{2i},V)^{r_j^*})$ by $\hat{\phi}_{ij}^\1$, $\tilde{\phi}_{ij}$.
Notice that
\begin{multline*}
   \sup_{n_1\geq 1} \frac{1}{n_1} \E \qty[ \norm{ \sum_{i \in \gD_1} \sum_{j=1}^3 \mathbbm{1}\{R_i = r_j^*\} \qty( \hat{\phi}_{ij}^\1 - \tilde{\phi}_{ij} + \tilde{\phi}_{ij}) }_2^2 ]
\leq \\2 \E \qty[ \sum_{j=1}^3 \norm{ \tilde{\phi}_j \qty((X, Y, W_1, W_2, V)^{r_j^*}) }_2^2 + \norm{ \qty( \hat{\phi}_j^\1 - \tilde{\phi}_j )\qty((X, Y, W_1, W_2, V)^{r_j^*}) }_2^2].
\end{multline*}
By Assumptions~\ref{asm: stability of psi} and~\ref{asm: consistency of M}, and applying Lemma~\ref{lem: converge to 0 in prob}, we have
\begin{equation}
\label{eq: normality 2.1 part}
    \norm{\frac{1}{\sqrt{n_1}}(\hat{M}^\1 - M) \sum_{i \in \gD_1} \sum_{j=1}^3 \mathbbm{1}\{ R_i = r_j^* \} \hat{\phi}_{ij}^\1}_2^2 \xrightarrow{P} 0.
\end{equation}
Moreover, by Assumption~\ref{asm: stability of psi} and using the Markov inequality, it follows that
\begin{equation*}
    \frac{1}{n_1} \norm{ \sum_{i \in \gD} \sum_{j=1}^3 \mathbbm{1}\{ R_i = r_j^* \} \qty( \hat{\phi}_{ij}^\1 - \tilde{\phi}_{ij} ) }_2^2 \xrightarrow{P} 0.
\end{equation*}
Then, applying Lemma~\ref{lem: converge to 0 in prob} once again, we conclude that
\begin{equation}
\label{eq: normality 2.2 part}
    \norm{ \frac{1}{\sqrt{n_1}} M \sum_{i \in \gD_1} \sum_{j=1}^3 \mathbbm{1}\{ R_i = r_j^* \} \qty( \hat{\phi}_{ij}^\1 - \tilde{\phi}_{ij} ) }_2^2 \xrightarrow{P} 0.
\end{equation}
Combining~\eqref{eq: normality 2.1 part} and~\eqref{eq: normality 2.2 part}, we obtain Equation~\eqref{eq: normality 2 part}. Finally, substituting Equations~\eqref{eq: normality 1 part} and~\eqref{eq: normality 2 part} into Equation~\eqref{eq: normality first batch}, we have
\begin{equation*}
     \frac{|\gD_1|}{\sqrt{n_1}} (\hat{\theta}^\1 - \theta) =  \frac{|\gD_1|}{\sqrt{n_1}}\check{\P}_{\gD_1} \qty[
\mathbbm{1}\{R = r_1^*\} \frac{\psi_1}{\alpha_1}
+ M \sum_{j=1}^3 \mathbbm{1}\{R = r_j^*\} \tilde{\phi}_j] + o_p(1).
\end{equation*}
Finally, we have
\begin{equation*}
    \sqrt{n_1} (\hat{\theta}^{\text{\method}} - \theta) = \sqrt{n_1} \check{\P}_{\gD} \qty[
\mathbbm{1}\{R = r_1^*\} \frac{\psi_1}{\alpha_1}
+ M \sum_{j=1}^3 \mathbbm{1}\{R = r_j^*\} \tilde{\phi}_j] + o_p(1).
\end{equation*}
\end{proof}

\subsection{Proof of Theorem \ref{thm: excess risk upper bound}}
\label{sec: app proof excess risk}
\begin{proof}
We restate the optimization problem defined in Algorithm \ref{alg: obtain alpha1}.
Define the feasible set
\begin{align}
\label{eq: feasible set F1 empirical}
\hat{\gF}_1
= \Bigg\{ \alpha(x, w_1, w_2) \;\Big|\;
   \sum_{j=1}^3 \alpha_j(x, w_1, w_2) = 1,\ 
   &\alpha_1(x, w_1, w_2) \geq \underline{\alpha},\\ 
   &\check{\P}_{\gD_{r_1^*, 2}^\0}\qty[\rho \alpha_1 + \alpha_2] \leq \tau + \delta\Bigg\}.    
\end{align}
The empirical optimal design is defined as
\begin{equation}
\label{eq: optimization problem for alpha empirical}
\hat{\alpha}^{\text{\method}}(x, w_1, w_2) 
= \argmin_{\alpha(x, w_1, w_2) \in \hat{\gF}_1} 
\hat{\gL}(\alpha(x, w_1, w_2)).
\end{equation}
Here the empirical loss is
\begin{multline}
\label{eq: empirical loss functional}
\hat{\gL}(\alpha(x, w_1, w_2)) 
:= \Tr\left(
\check{\P}_{\gD_{r_1^*, 2}^\0}\qty[\frac{\hat{\psi}_1 \hat{\psi}_1^\top}{\alpha_1}] \right. \\
\left.
- \check{\P}_{\gD_{r_1^*, 2}^\0}\qty[\hat{\psi}_1 \hat{\phi}_1^\top]
  \qty(\check{\P}_{\gD_{r_1^*, 2}^\0}\qty[\sum_{j=1}^3 \alpha_j\, \hat{\phi}_j \hat{\phi}_j^\top])^{-1}
  \check{\P}_{\gD_{r_1^*, 2}^\0}\qty[\hat{\phi}_1 \hat{\psi}_1^\top]
\right).
\end{multline}
We split the excess risk as
\begin{multline*}
\gL(\hat{\alpha}^{\text{\method}}(x, w_1, w_2)) - \gL(\alpha^*(x, w_1, w_2))
= \gL(\hat{\alpha}^{\text{\method}}(x, w_1, w_2)) - \hat{\gL}(\hat{\alpha}^{\text{\method}}(x, w_1, w_2))\\
+ \hat{\gL}(\hat{\alpha}^{\text{\method}}(x, w_1, w_2)) - \hat{\gL}(\alpha^*(x, w_1, w_2))
+ \hat{\gL}(\alpha^*(x, w_1, w_2)) - \gL(\alpha^*(x, w_1, w_2)).
\end{multline*}
Define
\begin{equation*}
\Omega_{1} 
= \qty{
   \abs{ \qty( \check{\P}_{\gD_{r_1^*, 2}^\0} - \P ) \qty[\rho \alpha_1^* + \alpha_2^*] } 
   \leq \delta }.
\end{equation*}
Since it is a bounded random variable, Hoeffding's inequality implies
\begin{equation*}
\P\qty(\Omega_1) 
\geq 1 - 2 \exp\qty( - \frac{n_0 \delta^2}{2(\rho+1)^2} ).
\end{equation*}
Now suppose $\Omega_1$ happens, then $\alpha^*(x, w_1, w_2) \in \hat{\gF}_1$.  
So
\begin{equation*}
\hat{\gL}\qty(\hat{\alpha}^{\text{\method}}(x, w_1, w_2)) 
- \hat{\gL}\qty(\alpha^*(x, w_1, w_2)) \leq 0.
\end{equation*}
Recall
\begin{equation*}
\gF_2 = \qty{ \alpha(x, w_1, w_2) \middle| \,
   \sum_{j=1}^3 \alpha_j(x, w_1, w_2) = 1,
   \alpha_1(x, w_1, w_2) \geq \underline{\alpha} }.
\end{equation*}
Since $\gF_1, \hat{\gF}_1 \subseteq \gF_2$, we have
\begin{equation*}
\gL\qty(\hat{\alpha}^{\text{\method}}(x, w_1, w_2)) - \gL\qty(\alpha^*(x, w_1, w_2))
\leq 2 \operatorname*{sup}_{\alpha(x, w_1, w_2) \in \gF_2}
\abs{ \hat{\gL}\qty(\alpha(x, w_1, w_2)) - \gL\qty(\alpha(x, w_1, w_2)) }.
\end{equation*}
Now we only need to control $\operatorname*{sup}_{\alpha(x, w_1, w_2) \in \gF_2}
\abs{ \hat{\gL}\qty(\alpha(x, w_1, w_2)) - \gL\qty(\alpha(x, w_1, w_2)) }$.
Next, we use the triangular inequality to decompose it:
\begin{multline*}
\operatorname*{sup}_{\alpha(x, w_1, w_2) \in \gF_2}
\abs{ \hat{\gL}\qty(\alpha(x, w_1, w_2)) - \gL\qty(\alpha(x, w_1, w_2)) } \leq
\operatorname*{sup}_{\alpha(x, w_1, w_2) \in \gF_2}\\
\abs{ \tilde{\gL}\qty(\alpha(x, w_1, w_2)) - \gL\qty(\alpha(x, w_1, w_2)) }
+ \operatorname*{sup}_{\alpha(x, w_1, w_2) \in \gF_2} 
\abs{ \hat{\gL}\qty(\alpha(x, w_1, w_2)) - \tilde{\gL}\qty(\alpha(x, w_1, w_2)) },
\end{multline*}
where
\begin{multline*}
\tilde{\gL}(\alpha(x, w_1, w_2)) 
:= \Tr\Bigg(
\check{\P}_{\gD_{r_1^*, 2}^\0}\qty[\frac{\psi_1 \psi_1^\top}{\alpha_1}] \\
- \check{\P}_{\gD_{r_1^*, 2}^\0}\qty[\psi_1 \tilde{\phi}_1^\top]
  \qty(\check{\P}_{\gD_{r_1^*, 2}^\0}\qty[\sum_{j=1}^3 \alpha_j\, \tilde{\phi}_j \tilde{\phi}_j^\top])^{-1}
  \check{\P}_{\gD_{r_1^*, 2}^\0}\qty[\tilde{\phi}_1 \psi_1^\top]
\Bigg).
\end{multline*}

\textbf{Firstly we control $\operatorname*{sup}_{\alpha(x, w_1, w_2) \in \gF_2}
\abs{ \tilde{\gL}\qty(\alpha(x, w_1, w_2)) - \gL\qty(\alpha(x, w_1, w_2)) }$.}
\begin{multline}
\label{eq: defn of i1}
I_1 = \operatorname*{sup}_{\alpha(x, w_1, w_2) \in \gF_2}\abs{\Tr\qty(
  \qty(\check{\P}_{\gD_{r_1^*,2}^\0} - \check{\P})\qty[ \frac{\psi_1 \psi_1^\top}{\alpha_1} ]
) }
\leq\\ d \operatorname*{sup}_{\alpha(x, w_1, w_2) \in \gF_2} \norm{ \qty(\check{\P}_{\gD_{r_1^*,2}^\0} - \check{\P})\qty[ \frac{\psi_1 \psi_1^\top}{\alpha_1} ]}_{\infty, \infty}.   
\end{multline}
\begin{equation}
\begin{aligned}
\label{eq: defn of i2}
I_2 = &\operatorname*{sup}_{\alpha(x, w_1, w_2) \in \gF_2} \left|
\Tr\Bigg(
  \check{\P}_{\gD_{r_1^*,2}^\0}\left[\psi_1 \tilde{\phi}_1^\top\right]
  \left(\check{\P}_{\gD_{r_1^*,2}^\0}\left[\sum_{j=1}^3 
      \alpha_j \tilde{\phi}_j \tilde{\phi}_j^\top\right]\right)^{-1}
  \check{\P}_{\gD_{r_1^*,2}^\0}\left[\tilde{\phi}_1 \psi_1^\top\right]
  \right.
\\
&\hspace{7em}\left.
  -\P\left[\psi_1 \tilde{\phi}_1^\top\right]
  \left(\P\left[\sum_{j=1}^3 
      \alpha_j \tilde{\phi}_j \tilde{\phi}_j^\top\right]\right)^{-1}
  \P\left[\tilde{\phi}_1 \psi_1^\top\right]
\Bigg)
\right|
\\
&\leq d \operatorname*{sup}_{\alpha(x, w_1, w_2) \in \gF_2}
\left\|
  \check{\P}_{\gD_{r_1^*,2}^\0}\left[\psi_1 \tilde{\phi}_1^\top\right]
  \left(\check{\P}_{\gD_{r_1^*,2}^\0}\left[\sum_{j=1}^3 
      \alpha_j \tilde{\phi}_j \tilde{\phi}_j^\top\right]\right)^{-1}
  \check{\P}_{\gD_{r_1^*,2}^\0}\left[\tilde{\phi}_1 \psi_1^\top\right]
\right.
\\
&\hspace{7em}\left.
  -\P\left[\psi_1 \tilde{\phi}_1^\top\right]
  \left(\P\left[\sum_{j=1}^3 
      \alpha_j \tilde{\phi}_j \tilde{\phi}_j^\top\right]\right)^{-1}
  \P\left[\tilde{\phi}_1 \psi_1^\top\right]
\right\|_{\mathrm{sp}}.
\end{aligned}
\end{equation}
\textbf{For $I_1$,}
\begin{equation*}
\operatorname*{sup}_{\alpha(x, w_1, w_2) \in \gF_2}
\norm{ \qty(\check{\P}_{\gD_{r_1^*,2}^\0} - \check{\P})\qty[\frac{\psi_1 \psi_1^\top}{\alpha_1}] }_{\infty,\infty}
= \operatorname*{sup}_{1 \leq k,l \leq d} \operatorname*{sup}_{\alpha(x, w_1, w_2) \in \gF_2}
\abs{
  \qty(\check{\P}_{\gD_{r_1^*,2}^\0} - \check{\P})\qty[\frac{\psi_{1k}\psi_{1l}}{\alpha_1}]
}.
\end{equation*}
We now apply the truncation technique:
\begin{equation*}
F_{\psi\psi} = \psi_1 \psi_1^\top \cdot \mathbf{1}\qty{\norm{ \psi_1 \psi_1^\top}_{\infty,\infty} \leq M_{\psi\psi} }, 
\qquad 
G_{\psi\psi} = \psi_1 \psi_1^\top - F_{\psi\psi},  
\end{equation*}
for some $M_{\psi\psi}$ we will determine later.

We analyze it entry-wise:
\begin{multline*}
\operatorname*{sup}_{\alpha(x, w_1, w_2) \in \gF_2}
\abs{ 
  \qty(\check{\P}_{\gD_{r_1^*,2}^\0} - \check{\P})\qty[\frac{\psi_{1k}\psi_{1l}}{\alpha_1}]
} \leq \\
 \operatorname*{sup}_{\alpha(x, w_1, w_2) \in \gF_2}
  \abs{ \qty(\check{\P}_{\gD_{r_1^*,2}^\0} - \check{\P})\qty[\frac{\qty(F_{\psi\psi})_{kl}}{\alpha_1}] }
  +\operatorname*{sup}_{\alpha(x, w_1, w_2) \in \gF_2}
  \abs{ \qty(\check{\P}_{\gD_{r_1^*,2}^\0} - \check{\P})\qty[\frac{\qty(G_{\psi\psi})_{kl}}{\alpha_1}] }.
\end{multline*}
Notice that $\abs{\qty(F_{\psi\psi})_{kl}} \leq M_{\psi\psi}$ and $\alpha_1(x,w_1,w_2) \geq \underline{\alpha}$.
Therefore, we have
\begin{equation*}
\frac{\abs{\qty(F_{\psi\psi})_{kl}}}{\alpha_1(x,w_1,w_2)}
\leq \frac{M_{\psi\psi}}{\underline{\alpha}}.
\end{equation*}
Apply Lemma \ref{lem: wainwright2019high uniform bound},
\begin{equation*}
\P\qty(\operatorname*{sup}_{\alpha(x, w_1, w_2) \in \gF_2}\abs{ \qty(\check{\P}_{\gD_{r_1^*,2}^\0} - \check{\P})\qty[\frac{\qty(F_{\psi\psi})_{kl}}{\alpha_1}] } \leq 2 \gR_{\frac{n_0}{2}}(\gF_2) + \delta) \geq 1 - \exp\qty(-\frac{n_0 \underline{\alpha}^2\delta^2}{4 M_{\psi\psi}^2}).
\end{equation*}
Define
\begin{equation*}
    \Omega_{21} = \qty{\operatorname*{sup}_{\alpha(x, w_1, w_2) \in \gF_2}\norm{ \qty(\check{\P}_{\gD_{r_1^*,2}^\0} - \check{\P})\qty[\frac{F_{\psi\psi}}{\alpha_1}] }_{\infty,\infty} \leq 2 \gR_{\frac{n_0}{2}}(\gF_2) + \delta}.
\end{equation*}
We have
\begin{equation*}
\P\qty(\Omega_{21}) \geq 1 - d^2\exp\qty(-\frac{n_0 \underline{\alpha}^2\delta^2}{4 M_{\psi\psi}^2}).
\end{equation*}
Moreover,
\begin{multline*}
    \operatorname*{sup}_{\alpha(x, w_1, w_2) \in \gF_2}
  \abs{ \qty(\check{\P}_{\gD_{r_1^*,2}^\0} - \check{\P})\qty[\frac{\qty(G_{\psi\psi})_{kl}}{\alpha_1}] } \leq  \operatorname*{sup}_{\alpha(x, w_1, w_2) \in \gF_2}
  \abs{ \check{\P}_{\gD_{r_1^*,2}^\0}\qty[\frac{\qty(G_{\psi\psi})_{kl}}{\alpha_1}] }  +\\ \operatorname*{sup}_{\alpha(x, w_1, w_2) \in \gF_2}
  \abs{ \check{\P}\qty[\frac{\qty(G_{\psi\psi})_{kl}}{\alpha_1}] } \leq
  \frac{1}{\underline{\alpha}}\qty(\check{\P}_{\gD_{r_1^*,2}^\0}\qty[\abs{\qty(G_{\psi\psi})_{kl}}] + \check{\P}\qty[\abs{\qty(G_{\psi\psi})_{kl}}])
\end{multline*}
By Assumption \ref{asm: subweibull for phi and psi}, each entry of $\psi_1$ is sub-Weibull($\beta$) with Orlicz norm $K_4$. Hence, the product $\abs{\psi_{1k}\psi_{1l}}$ is sub-Weibull($\frac{\beta}{2}$) with Orlicz norm bounded by $C_{\beta, 1}K_4^2$. Applying Lemma \ref{lem: vladimirova2020sub sum of iid subweibull}, we obtain that $\check{\P}_{\gD_{r_1^*,2}^\0}\qty[\abs{\psi_{1k} \psi_{1l}}]$ is sub-Weibull($\frac{\beta}{2}$) with
\begin{equation*}
\P\qty( \check{\P}_{\gD_{r_1^*,2}^\0}\qty[\abs{\psi_{1k} \psi_{1l}}] \geq M_{\psi\psi} )
\leq 2 \exp\qty( -\qty(\frac{M_{\psi\psi}}{2C_{\beta, 2}K_4^2})^{\frac{\beta}{2}} )\,,   
\end{equation*}
where $C_{\beta, 1}, C_{\beta, 2}$ are some constants depending only on $\beta$. Consequently,
\begin{equation*}
\P\qty( \frac{1}{\underline{\alpha}} \check{\P}_{\gD_{r_1^*,2}^\0}\qty[\abs{\qty(G_{\psi\psi})_{kl}}] > 0 )
\leq 2 \exp\qty( -\qty(\frac{M_{\psi\psi}}{2C_{\beta, 2}K_4^2})^{\frac{\beta}{2}} ). 
\end{equation*}
Moreover, for $\frac{1}{\underline{\alpha}} \check{\P}\qty[\abs{\qty(G_{\psi\psi})_{kl}}]$, we have
\begin{equation*}
\begin{aligned}
\frac{1}{\underline{\alpha}}\check{\P}\qty[\abs{\qty(G_{\psi\psi})_{kl}}]
   & \leq \frac{1}{\underline{\alpha}}\qty(M_{\psi\psi}  \P\qty( \abs{\psi_{1k}\psi_{1l}} > M_{\psi\psi} ) + \int_{M_{\psi\psi}}^{\infty} \P\qty( \abs{\psi_{1k}\psi_{1l}} > t ) dt )\\
   & \leq \frac{1}{\underline{\alpha}}\qty(M_{\psi\psi}  \exp\qty(-\qty(\frac{M_{\psi\psi}}{C_{\beta, 1}K_4^2})^{\frac{\beta}{2}}) + 2\int_{M_{\psi\psi}}^{\infty} \exp\qty( -\qty(\frac{t}{C_{\beta, 1}K_4^2})^{\frac{\beta}{2}} ) dt).
\end{aligned}
\end{equation*}
Let $u = \qty(\frac{t}{C_{\beta, 1}K_4^2})^{\frac{\beta}{2}}$ and $t = C_{\beta, 1}K_4^2 u^{\frac{2}{\beta}}$, then $dt = \frac{2}{\beta}C_{\beta, 1}K_4^2 u^{\frac{2}{\beta}-1} du$, and
\begin{equation*}
\int_{M_{\psi\psi}}^{\infty} \exp\qty( -\qty(\frac{t}{C_{\beta, 1}K_4^2})^{\frac{\beta}{2}} ) dt = \frac{2}{\beta}C_{\beta, 1}K_4^2 \int_{u_0}^{\infty} e^{-u}\,u^{\frac{2}{\beta}-1} du, \quad \text{where} \,u_0 = \qty(\frac{M_{\psi\psi}}{C_{\beta, 1}K_4^2})^{\frac{\beta}{2}}.
\end{equation*}
By Lemma \ref{lem: incomplete gamma function upper bound}, with $\qty(\frac{M_{\psi\psi}}{C_{\beta, 1}K_4^2})^{\frac{\beta}{2}} > 2\qty(\frac{2}{\beta} - 1)$,
\begin{equation*}
    \int_{M_{\psi\psi}}^{\infty} \exp\qty( -\qty(\frac{t}{C_{\beta, 1}K_4^2})^{\frac{\beta}{2}} ) dt \leq \frac{4}{\beta} C_{\beta, 1}K_4^2 \qty(\frac{M_{\psi\psi}}{C_{\beta, 1}K_4^2})^{1 - \frac{\beta}{2}} 
   \exp\qty(-\qty(\frac{M_{\psi\psi}}{C_{\beta, 1}K_4^2})^{\frac{\beta}{2}})
\end{equation*}
Define
\begin{multline*}
     \Omega_{22} = \Bigg\{\operatorname*{sup}_{\alpha(x, w_1, w_2) \in \gF_2}
  \norm{ \qty(\check{\P}_{\gD_{r_1^*,2}^\0} - \check{\P})\qty[\frac{G_{\psi\psi}}{\alpha_1}] }_{\infty,\infty} \leq \\\frac{1}{\underline{\alpha}}\qty(M_{\psi\psi} + \frac{8}{\beta} C_{\beta, 1}K_4^2 \qty(\frac{M_{\psi\psi}}{C_{\beta, 1}K_4^2})^{1 - \frac{\beta}{2}} )
   \exp\qty(-\qty(\frac{M_{\psi\psi}}{C_{\beta, 1}K_4^2})^{\frac{\beta}{2}})\Bigg\}
\end{multline*}
Then
\begin{equation*}
    \P\qty(\Omega_{22}) \geq 1 - 2 d^2 \exp\qty( -\qty(\frac{M_{\psi\psi}}{2C_{\beta, 2}K_4^2})^{\frac{\beta}{2}} )
\end{equation*}
Suppose $\Omega_{21}\cap \Omega_{22}$, denoting $\delta_{M_{\psi\psi}} = \frac{1}{\underline{\alpha}}\qty(M_{\psi\psi} + \frac{8}{\beta} C_{\beta, 1}K_4^2 \qty(\frac{M_{\psi\psi}}{C_{\beta, 1}K_4^2})^{1 - \frac{\beta}{2}} )
   \exp\qty(-\qty(\frac{M_{\psi\psi}}{C_{\beta, 1}K_4^2})^{\frac{\beta}{2}})$, we have
\begin{equation}
\label{eq: i1 upper bound}
I_1 \leq d \qty(2 \gR_{\frac{n_0}{2}}(\gF_2) + \delta + \delta_{M_{\psi\psi}}).
\end{equation}
\textbf{For $I_2$,} we seek to use Lemma \ref{lem: ABI-op-relaxed}, which means we need to control
\begin{equation*}
    \operatorname*{sup}_{\alpha(x, w_1, w_2) \in \gF_2} \norm{\qty(\check{\P}_{\gD_{r_1^*,2}^\0} - \check{\P})\qty[\sum_{j=1}^3 \alpha_j\, \tilde{\phi}_j \tilde{\phi}_j^\top]}_{\mathrm{sp}}, \quad \norm{\qty(\check{\P}_{\gD_{r_1^*,2}^\0} - \check{\P})\qty[\psi_1 \tilde{\phi}_1^\top]}_{\mathrm{sp}}.
\end{equation*}
Since for any matrix $A$ we have 
$\norm{A}_{\mathrm{sp}} \leq \norm{A}_{\infty,\infty}$, 
it suffices to control the $\infty,\infty$-norm, i.e.,
\begin{equation*}
    \operatorname*{sup}_{\alpha(x, w_1, w_2) \in \gF_2} \norm{\qty(\check{\P}_{\gD_{r_1^*,2}^\0} - \check{\P})\qty[\sum_{j=1}^3 \alpha_j\, \tilde{\phi}_j \tilde{\phi}_j^\top]}_{\infty,\infty}, \quad \norm{\qty(\check{\P}_{\gD_{r_1^*,2}^\0} - \check{\P})\qty[\psi_1 \tilde{\phi}_1^\top]}_{\infty,\infty}.
\end{equation*}
For $\sum_{j=1}^3 \alpha_j\, \tilde{\phi}_j \tilde{\phi}_j^\top$, using the same truncation technique:
\begin{equation*}
F_{\tilde{\phi}_j \tilde{\phi}_j} = \tilde{\phi}_j \tilde{\phi}_j^\top \cdot \mathbf{1}\qty{\sum_{j=1}^3 \norm{\tilde{\phi}_j \tilde{\phi}_j^\top}_{\infty,\infty} \leq M_{\tilde{\phi} \tilde{\phi}} }, 
\qquad 
G_{\tilde{\phi}_j \tilde{\phi}_j} = \tilde{\phi}_j \tilde{\phi}_j^\top - F_{\tilde{\phi}_j \tilde{\phi}_j}.  
\end{equation*}
We continue the analysis entry-wise:
\begin{multline*}
    \operatorname*{sup}_{\alpha(x, w_1, w_2) \in \gF_2} \abs{\qty(\check{\P}_{\gD_{r_1^*,2}^\0} - \check{\P})\qty[\sum_{j=1}^3 \alpha_j\, \tilde{\phi}_{jk} \tilde{\phi}_{jl}^\top]} \leq \operatorname*{sup}_{\alpha(x, w_1, w_2) \in \gF_2} \\
     \abs{\qty(\check{\P}_{\gD_{r_1^*,2}^\0} - \check{\P})\qty[\sum_{j=1}^3 \alpha_j\, \qty(F_{\tilde{\phi}_j \tilde{\phi}_j})_{kl}]} + \operatorname*{sup}_{\alpha(x, w_1, w_2) \in \gF_2} \abs{\qty(\check{\P}_{\gD_{r_1^*,2}^\0} - \check{\P})\qty[\sum_{j=1}^3 \alpha_j\, \qty(G_{\tilde{\phi}_j \tilde{\phi}_j})_{kl}]}
\end{multline*}
Notice that
\begin{equation*}
\abs{\sum_{j=1}^3 \alpha_j(x, w_1, w_2)\, \tilde{\phi}_{jk} \tilde{\phi}_{jl}^\top}
\leq M_{\phi\phi},
\end{equation*}
and hence by using Lemma \ref{lem: wainwright2019high uniform bound}, similar to the analysis of $F_{\psi\psi}$, we have
\begin{multline*}
\P\qty(\operatorname*{sup}_{\alpha(x, w_1, w_2) \in \gF_2}\abs{\qty(\check{\P}_{\gD_{r_1^*,2}^\0} - \check{\P})\qty[\sum_{j=1}^3 \alpha_j\, \qty(F_{\tilde{\phi}_j \tilde{\phi}_j})_{kl}]} \leq 2 \gR_{\frac{n_0}{2}}(\gF_2) + \delta) \geq \\1 - \exp\qty(-\frac{n_0 \delta^2}{4 M_{\tilde{\phi} \tilde{\phi}}^2}).    
\end{multline*}
Now we define
\begin{equation*}
    \Omega_{23} = \qty{\operatorname*{sup}_{\alpha(x, w_1, w_2) \in \gF_2}\norm{ \qty(\check{\P}_{\gD_{r_1^*,2}^\0} - \check{\P})\qty[\sum_{j=1}^3 \alpha_j F_{\tilde{\phi}_j \tilde{\phi}_j}] }_{\infty,\infty} \leq 2 \gR_{\frac{n_0}{2}}(\gF_2) + \delta}.    
\end{equation*}
And we obtain that
\begin{equation*}
\P\qty(\Omega_{23}) \geq 1 - d^2\exp\qty(-\frac{n_0\delta^2}{4 M_{\tilde{\phi} \tilde{\phi}}^2}).
\end{equation*}
For $G_{\tilde{\phi}_j \tilde{\phi}_j}$ term, we have
\begin{equation*}
\begin{aligned}
&\operatorname*{sup}_{\alpha(x, w_1, w_2) \in \gF_2} \abs{\qty(\check{\P}_{\gD_{r_1^*,2}^\0} - \check{\P})\qty[\sum_{j=1}^3 \alpha_j\, \qty(G_{\tilde{\phi}_j \tilde{\phi}_j})_{kl}]}\\
&\leq \operatorname*{sup}_{\alpha(x, w_1, w_2) \in \gF_2} \abs{\check{\P}_{\gD_{r_1^*,2}^\0}\qty[\sum_{j=1}^3 \alpha_j\, \qty(G_{\tilde{\phi}_j \tilde{\phi}_j})_{kl}]}  + \operatorname*{sup}_{\alpha(x, w_1, w_2) \in \gF_2} \abs{\check{\P}\qty[\sum_{j=1}^3 \alpha_j\, \qty(G_{\tilde{\phi}_j \tilde{\phi}_j})_{kl}]}\\
&\leq \check{\P}_{\gD_{r_1^*,2}^\0}\qty[\sum_{j=1}^3\abs{\qty(G_{\tilde{\phi}_j \tilde{\phi}_j})_{kl}}]  + \check{\P}\qty[\sum_{j=1}^3 \abs{\qty(G_{\tilde{\phi}_j \tilde{\phi}_j})_{kl}}].
\end{aligned}
\end{equation*}
By Assumption \ref{asm: subweibull for phi and psi}, $\tilde{\phi}_j, j \in [3]$ are entrywise sub-Weibull($\beta$) with Orlicz norm $K_4$. Hence $\sum_{j=1}^3\abs{(\tilde{\phi}_j \tilde{\phi}_j^\top)_{kl}}$ is sub-Weibull($\frac{\beta}{2}$) with norm bounded by $3C_{\beta, 3}K_4^2$. Applying Lemma \ref{lem: vladimirova2020sub sum of iid subweibull}, it follows that $\check{\P}_{\gD_{r_1^*,2}^\0}\qty[\sum_{j=1}^3\abs{(\tilde{\phi}_j \tilde{\phi}_j^\top)_{kl}}]$ is sub-Weibull($\frac{\beta}{2}$) with
\begin{equation*}
\P\qty( \check{\P}_{\gD_{r_1^*,2}^\0}\qty[\sum_{j=1}^3\abs{(\tilde{\phi}_j \tilde{\phi}_j^\top)_{kl}}] \geq M_{\tilde{\phi}\tilde{\phi}} )
\leq 2 \exp\qty( -\qty(\frac{M_{\tilde{\phi}\tilde{\phi}}}{6C_{\beta,4}K_4^2})^{\frac{\beta}{2}} ),  
\end{equation*}
where $C_{\beta, 3}, C_{\beta,4}$ are some constants depending only on $\beta$. Consequently,
\begin{equation*}
\P\qty( \check{\P}_{\gD_{r_1^*,2}^\0}\qty[\sum_{j=1}^3\abs{(G_{\tilde{\phi}_j \tilde{\phi}_j})_{kl}}] > 0 )
\leq 2 \exp\qty( -\qty(\frac{M_{\tilde{\phi}\tilde{\phi}}}{6C_{\beta,4}K_4^2})^{\frac{\beta}{2}} ). 
\end{equation*}
Moreover,
\begin{multline*}
 \check{\P}\qty[\sum_{j=1}^3\abs{(G_{\tilde{\phi}_j \tilde{\phi}_j})_{kl}}]
\le M_{\tilde{\phi}\tilde{\phi}} \exp\qty( -\qty(\frac{M_{\tilde{\phi}\tilde{\phi}}}{3C_{\beta, 3}K_4^2})^{\frac{\beta}{2}} ) + 2 \int_{M_{\tilde{\phi}\tilde{\phi}}}^{\infty}
\exp\qty( -\qty(\frac{t}{3C_{\beta, 3}K_4^2})^{\frac{\beta}{2}} ) dt\\
= M_{\tilde{\phi}\tilde{\phi}} \exp\qty( -\qty(\frac{M_{\tilde{\phi}\tilde{\phi}}}{3C_{\beta, 3}K_4^2})^{\frac{\beta}{2}} ) + \frac{12}{\beta}C_{\beta, 3}K_4^2 \int_{u_0}^{\infty} e^{-u} u^{\frac{2}{\beta}-1} du,
~\text{where}
\, u_0 := \qty(\frac{M_{\tilde{\phi}\tilde{\phi}}}{3C_{\beta, 3}K_4^2})^{\frac{\beta}{2}}.
\end{multline*}
By Lemma \ref{lem: incomplete gamma function upper bound}, for $\qty(\frac{M_{\tilde{\phi}\tilde{\phi}}}{3C_{\beta, 3}K_4^2})^{\frac{\beta}{2}} > 2\qty(\frac{2}{\beta} - 1)$,
\begin{equation*}
\int_{u_0}^{\infty} e^{-u} u^{\frac{2}{\beta}-1} du
\le 2 \qty(\frac{M_{\tilde{\phi}\tilde{\phi}}}{3C_{\beta, 3}K_4^2})^{1-\frac{\beta}{2}}
\exp\qty( -\qty(\frac{M_{\tilde{\phi}\tilde{\phi}}}{3C_{\beta, 3}K_4^2})^{\frac{\beta}{2}} ).
\end{equation*}
Thus,
\begin{equation*}
\check{\P}\qty[\sum_{j=1}^3\abs{(G_{\tilde{\phi}_j \tilde{\phi}_j})_{kl}}] \leq \qty(M_{\tilde{\phi}\tilde{\phi}} + \frac{24}{\beta}C_{\beta, 3}K_4^2 \qty(\frac{M_{\tilde{\phi}\tilde{\phi}}}{3C_{\beta, 3}K_4^2})^{1-\frac{\beta}{2}})
\exp\qty( -\qty(\frac{M_{\tilde{\phi}\tilde{\phi}}}{3C_{\beta, 3}K_4^2})^{\frac{\beta}{2}} ).
\end{equation*}
Define
\begin{multline*}
     \Omega_{24} = \Bigg\{\operatorname*{sup}_{\alpha(x, w_1, w_2) \in \gF_2}\norm{ \qty(\check{\P}_{\gD_{r_1^*,2}^\0} - \check{\P})\qty[\sum_{j=1}^3 \alpha_j G_{\tilde{\phi}_j \tilde{\phi}_j}] }_{\infty,\infty} \leq \\\qty(M_{\tilde{\phi}\tilde{\phi}} + \frac{24}{\beta}C_{\beta, 3}K_4^2 \qty(\frac{M_{\tilde{\phi}\tilde{\phi}}}{3C_{\beta, 3}K_4^2})^{1-\frac{\beta}{2}})
\exp\qty( -\qty(\frac{M_{\tilde{\phi}\tilde{\phi}}}{3C_{\beta, 3}K_4^2})^{\frac{\beta}{2}} )\Bigg\}\,.
\end{multline*}
Then
\begin{equation*}
    \P\qty(\Omega_{24}) \geq 1 - 2 d^2 \exp\qty( -\qty(\frac{M_{\tilde{\phi}\tilde{\phi}}}{6C_{\beta,4}K_4^2})^{\frac{\beta}{2}} )\,.
\end{equation*}
Now we control $\norm{\qty(\check{\P}_{\gD_{r_1^*,2}^\0} - \check{\P})\qty[\psi_1 \tilde{\phi}_1^\top]}_{\infty,\infty}$. We still analyze it entry-wise.
By Assumption \ref{asm: subweibull for phi and psi}, each entry of $\psi_1, \tilde{\phi}_1$ is sub-Weibull($\beta$) with Orlicz norm $K_4$. Hence, the product $\psi_{1k}\tilde{\phi}_{1l}$ is sub-Weibull($\frac{\beta}{2}$) with Orlicz norm bounded by $C_{\beta, 1}K_4^2$. Applying Lemma \ref{lem: vladimirova2020sub sum of iid subweibull}, we obtain that $\qty(\check{\P}_{\gD_{r_1^*,2}^\0} - \check{\P})\qty[\psi_{1k} \tilde{\phi}_{1l}]$ is sub-Weibull($\frac{\beta}{2}$) with
\begin{equation*}
\P\qty( \abs{\qty(\check{\P}_{\gD_{r_1^*,2}^\0} - \check{\P})\qty[\psi_{1k} \tilde{\phi}_{1l}]} \geq \delta )
\leq 2 \exp\qty( -\qty(\frac{n_0\delta}{2C_{\beta, 2}K_4^2})^{\frac{\beta}{2}} ).   
\end{equation*}
Define
\begin{equation}
\label{eq: omega set 25 psiphi}
    \Omega_{25} = \qty{\norm{\qty(\check{\P}_{\gD_{r_1^*,2}^\0} - \check{\P})\qty[\psi_1 \tilde{\phi}_1^\top]}_{\infty,\infty} \leq\delta}.
\end{equation}
We have
\begin{equation*}
\P\qty(\Omega_{25}) \geq 1 - 2 d^2 \exp\qty( -\qty(\frac{n_0\delta}{2C_{\beta, 2}K_4^2})^{\frac{\beta}{2}} ).
\end{equation*}
Similarly,
Define
\begin{equation}
\label{eq: omega set 26 phiphi}
    \Omega_{26} = \qty{\norm{\qty(\check{\P}_{\gD_{r_1^*,2}^\0} - \check{\P})\qty[\tilde{\phi}_1 \tilde{\phi}_1^\top]}_{\infty,\infty} \leq\delta}.
\end{equation}
We have
\begin{equation*}
\P\qty(\Omega_{26}) \geq 1 - 2 d^2 \exp\qty( -\qty(\frac{n_0\delta}{2C_{\beta, 2}K_4^2})^{\frac{\beta}{2}} ).
\end{equation*}
Suppose $\Omega_{23} \cap \Omega_{24}\cap \Omega_{25} \cap \Omega_{26}$ happens,
\begin{equation*}
    \Delta_{\psi\phi,1} =  \norm{\qty(\check{\P}_{\gD_{r_1^*,2}^\0} - \check{\P})\qty[\psi_1 \tilde{\phi}_1^\top]}_{\mathrm{sp}} \leq  \norm{\qty(\check{\P}_{\gD_{r_1^*,2}^\0} - \check{\P})\qty[\psi_1 \tilde{\phi}_1^\top]}_{\infty,\infty} \leq \delta,
\end{equation*}
Denote $\delta_{M_{\tilde{\phi}\tilde{\phi}}} = \qty(M_{\tilde{\phi}\tilde{\phi}} + \frac{24}{\beta}C_{\beta, 3}K_4^2 \qty(\frac{M_{\tilde{\phi}\tilde{\phi}}}{3C_{\beta, 3}K_4^2})^{1-\frac{\beta}{2}})
\exp\qty( -\qty(\frac{M_{\tilde{\phi}\tilde{\phi}}}{3C_{\beta, 3}K_4^2})^{\frac{\beta}{2}} )$.
\begin{multline*}
    \Delta_{\phi\phi,1} = \operatorname*{sup}_{\alpha(x, w_1, w_2) \in \gF_2} \norm{\qty(\check{\P}_{\gD_{r_1^*,2}^\0} - \check{\P})\qty[\sum_{j=1}^3 \alpha_j\, \tilde{\phi}_j \tilde{\phi}_j^\top]}_{\mathrm{sp}} \leq \\   \operatorname*{sup}_{\alpha(x, w_1, w_2) \in \gF_2} \norm{\qty(\check{\P}_{\gD_{r_1^*,2}^\0} - \check{\P})\qty[\sum_{j=1}^3 \alpha_j \tilde{\phi}_j \tilde{\phi}_j^\top]}_{\infty,\infty} \leq 2 \gR_{\frac{n_0}{2}}(\gF_2) + \delta + \delta_{M_{\tilde{\phi}\tilde{\phi}}}.
\end{multline*}
Moreover,
\begin{equation*}
    \sigma_{\min} \qty(\check{\P}\qty[ \sum_{j=1}^3 \alpha_j \tilde{\phi}_j\tilde{\phi}_j^\top]) \geq \underline{\alpha} \sigma_{\min} \qty(\check{\P}\qty[\tilde{\phi}_1\tilde{\phi}_1^\top]) \geq \underline{\alpha} K_2. 
\end{equation*}
\begin{multline}
\label{eq: sigam min pd phiphi}
    \sigma_{\min} \qty(\check{\P}_{\gD_{r_1^*,2}^\0}\qty[ \sum_{j=1}^3 \alpha_j \tilde{\phi}_j\tilde{\phi}_j^\top]) \geq \underline{\alpha} \sigma_{\min} \qty(\check{\P}_{\gD_{r_1^*,2}^\0}\qty[\tilde{\phi}_1\tilde{\phi}_1^\top]) \\
    \geq \underline{\alpha}\qty(\sigma_{\min} \qty(\check{\P}\qty[\tilde{\phi}_1\tilde{\phi}_1^\top]) - \norm{\qty(\check{\P}_{\gD_{r_1^*,2}^\0} - \check{\P})\qty[ \tilde{\phi}_1 \tilde{\phi}_1^\top]}_{\mathrm{sp}}) \geq \underline{\alpha}\qty( K_2 - \delta),
\end{multline}
where the last inequality follows from Equation \eqref{eq: omega set 26 phiphi}.
Now we summary the bounds as following:
\begin{multline*}
\norm{\check{\P}\qty[\psi_1\tilde{\phi}_1^\top]}_{\mathrm{sp}} \leq K_3,
\quad
\Delta_{\psi\phi,1}\le \delta,
\quad
\Delta_{\phi\phi,1}\le 2 \gR_{\frac{n_0}{2}}(\gF_2) + \delta + \delta_{M_{\tilde{\phi}\tilde{\phi}}},
\\
\sigma_{\min} \qty(\check{\P}\qty[ \sum_{j=1}^3 \alpha_j \tilde{\phi}_j\tilde{\phi}_j^\top])\ge \underline{\alpha} K_2,
\quad
\sigma_{\min} \qty(\check{\P}_{\gD_{r_1^*,2}^\0}\qty[ \sum_{j=1}^3 \alpha_j \tilde{\phi}_j\tilde{\phi}_j^\top])\ge \underline{\alpha}(K_2-\delta).
\end{multline*}
Then apply Lemma \ref{lem: ABI-op-relaxed},  when $\delta < \min \{\frac{K_2}{4}, K_3\}$,
\begin{multline}
\label{eq: i2 upper bound}
    I_2  \leq d\qty(\frac{(K_3 + \delta)\delta}{\underline{\alpha}\qty( K_2 - \delta)}
+\frac{K_3(K_3 + \delta)}{\underline{\alpha}K_2 \underline{\alpha}\qty( K_2 - \delta)}\qty(2 \gR_{\frac{n_0}{2}}(\gF_2) + \delta + \delta_{M_{\tilde{\phi}\tilde{\phi}}})
+\frac{K_3 \delta}{\underline{\alpha}K_2}) \leq \\
\qty(4\frac{K_3}{K_2} +  3\frac{K_3^2}{K_2^2}) \frac{d\delta}{\underline{\alpha}} + \frac{3dK_3^2}{\underline{\alpha}K_2^2}\qty(2 \gR_{\frac{n_0}{2}}(\gF_2) + \delta_{M_{\tilde{\phi}\tilde{\phi}}}).
\end{multline}
Define $\Omega_2 = \bigcap_{j=1}^6 \Omega_{2j}$. Suppose the event $\bigcap_{j=1}^2 \Omega_j$ occurs. By combining Equation \eqref{eq: i1 upper bound} and Equation \eqref{eq: i2 upper bound}, we obtain:
\begin{multline}
\label{eq: term 1 upper bound}
\operatorname*{sup}_{\alpha(x, w_1, w_2) \in \gF_2}
\abs{ \tilde{\gL}\qty(\alpha(x, w_1, w_2)) - \gL\qty(\alpha(x, w_1, w_2)) } \leq I_1 + I_2 \leq \\
\qty(1 + 4\frac{K_3}{K_2} +  3\frac{K_3^2}{K_2^2}) \frac{d\delta}{\underline{\alpha}} + \qty(2 + \frac{6K_3^2}{\underline{\alpha}K_2^2})d\gR_{\frac{n_0}{2}}(\gF_2) + \frac{3dK_3^2}{\underline{\alpha}K_2^2}\delta_{M_{\tilde{\phi}\tilde{\phi}}} + d \delta_{M_{\psi \psi}}.
\end{multline}

\textbf{Secondly we control $\operatorname*{sup}_{\alpha(x, w_1, w_2) \in \gF_2} 
\abs{ \hat{\gL}\qty(\alpha(x, w_1, w_2)) - \tilde{\gL}\qty(\alpha(x, w_1, w_2)) }$.} It can be split to the following two terms:
\begin{multline}
\label{eq: defn of i3}
I_3 = \operatorname*{sup}_{\alpha(x, w_1, w_2) \in \gF_2}\abs{\Tr\qty(
  \check{\P}_{\gD_{r_1^*,2}^\0}\qty[ \frac{\hat{\psi}_1 \hat{\psi}_1^\top - \psi_1 \psi_1^\top}{\alpha_1} ]
) }
\leq \\d \operatorname*{sup}_{\alpha(x, w_1, w_2) \in \gF_2}
  \norm{ \check{\P}_{\gD_{r_1^*,2}^\0}\qty[ \frac{\hat{\psi}_1 \hat{\psi}_1^\top - \psi_1 \psi_1^\top}{\alpha_1} ] }_{\mathrm{sp}} 
\leq \frac{d}{\underline{\alpha}}
  \check{\P}_{\gD_{r_1^*,2}^\0}\qty[ \norm{ \hat{\psi}_1 \hat{\psi}_1^\top - \psi_1 \psi_1^\top }_{\mathrm{sp}} ].
\end{multline}
\begin{equation}
\begin{aligned}
\label{eq: defn of i4}
I_4 = &\operatorname*{sup}_{\alpha(x, w_1, w_2) \in \gF_2} \left|
\Tr\Bigg(
  \check{\P}_{\gD_{r_1^*,2}^\0}\left[\hat{\psi}_1\hat{\phi}_1^\top\right]
  \left(\check{\P}_{\gD_{r_1^*,2}^\0}\left[\sum_{j=1}^3 
      \alpha_j \hat{\phi}_j\hat{\phi}_j^\top\right]\right)^{-1}
  \check{\P}_{\gD_{r_1^*,2}^\0}\left[\hat{\phi}_1\hat{\psi}_1^\top\right]
  \right.
\\
&\hspace{7em}\left.
  -\check{\P}_{\gD_{r_1^*,2}^\0}\left[\psi_1\tilde{\phi}_1^\top\right]
  \left(\check{\P}_{\gD_{r_1^*,2}^\0}\left[\sum_{j=1}^3 
      \alpha_j \tilde{\phi}_j\tilde{\phi}_j^\top\right]\right)^{-1}
  \check{\P}_{\gD_{r_1^*,2}^\0}\left[\tilde{\phi}_1\psi_1^\top\right]
\Bigg)
\right|
\\
&\leq d \operatorname*{sup}_{\alpha(x, w_1, w_2) \in \gF_2}
\left\|
  \check{\P}_{\gD_{r_1^*,2}^\0}\left[\hat{\psi}_1\hat{\phi}_1^\top\right]
  \left(\check{\P}_{\gD_{r_1^*,2}^\0}\left[\sum_{j=1}^3 
      \alpha_j \hat{\phi}_j\hat{\phi}_j^\top\right]\right)^{-1}
  \check{\P}_{\gD_{r_1^*,2}^\0}\left[\hat{\phi}_1\hat{\psi}_1^\top\right]
\right.
\\
&\hspace{7em}\left.
  -\check{\P}_{\gD_{r_1^*,2}^\0}\left[\psi_1\tilde{\phi}_1^\top\right]
  \left(\check{\P}_{\gD_{r_1^*,2}^\0}\left[\sum_{j=1}^3 
      \alpha_j \tilde{\phi}_j\tilde{\phi}_j^\top\right]\right)^{-1}
  \check{\P}_{\gD_{r_1^*,2}^\0}\left[\tilde{\phi}_1\psi_1^\top\right]
\right\|_{\mathrm{sp}}.
\end{aligned}
\end{equation}
To control the above two terms, we introduce a sequence of events.
\begin{equation*}
\Omega_{31} = \qty{ 
   \abs{ \check{\P}_{\gD_{r_1^*,2}^\0}\qty[ \norm{ \hat{\psi}_1 \hat{\psi}_1^\top - \psi_1 \psi_1^\top }_{\mathrm{sp}} ]
   - \check{\P}\qty[ \norm{ \hat{\psi}_1 \hat{\psi}_1^\top - \psi_1 \psi_1^\top }_{\mathrm{sp}} ] }
   \leq \delta }.
\end{equation*}
By Assumption \ref{asm: stronger stability},
\begin{equation*}
    \P\qty(\Omega_{31}) 
\geq 1 - 2 \exp\qty( - \frac{n_0 \delta^2}{2 K_1^2} ).
\end{equation*}
Define
\begin{equation*}
\Omega_{32} = \qty{ 
   \check{\P}\qty[ \norm{ \hat{\psi}_1 \hat{\psi}_1^\top - \psi_1 \psi_1^\top }_{\mathrm{sp}} ] \leq \delta}.
\end{equation*}
And apply Assumption \ref{asm: stronger stability} again, using Markov's inequality, we have
\begin{equation*}
\P\qty( 
   \Omega_{32}) \geq 1 - \frac{\epsilon_{n_0}}{\delta}.
\end{equation*}
Similarly, for each $j \in \{1,2,3\}$ we define
\begin{equation*}
\Omega_{33}^{(j)} = \qty{
   \abs{ \check{\P}_{\gD_{r_j^*,2}^\0}\qty[ \norm{ \hat{\phi}_j \hat{\phi}_j^\top - \tilde{\phi}_j \tilde{\phi}_j^\top }_{\mathrm{sp}} ]
   - \check{\P}\qty[ \norm{ \hat{\phi}_j \hat{\phi}_j^\top - \tilde{\phi}_j \tilde{\phi}_j^\top }_{\mathrm{sp}} ] }
   \leq \delta }.
\end{equation*}
By Assumption~\ref{asm: stronger stability},
\begin{equation*}
\P\qty(\Omega_{33}^{(j)}) \geq 1 - 2 \exp\qty( - \frac{n_0 \delta^2}{2 K_1^2} ).
\end{equation*}

\begin{equation*}
\Omega_{34}^{(j)} = \qty{
   \check{\P}\qty[ \norm{ \hat{\phi}_j \hat{\phi}_j^\top - \tilde{\phi}_j \tilde{\phi}_j^\top }_{\mathrm{sp}} ] \leq \delta }.
\end{equation*}
By Assumption~\ref{asm: stronger stability} and Markov's inequality,
\begin{equation*}
\P\qty(\Omega_{34}^{(j)}) \geq 1 - \frac{\epsilon_{n_0}}{\delta}.
\end{equation*}
For the cross term we define
\begin{equation*}
\Omega_{35} = \qty{
   \abs{ \check{\P}_{\gD_{r_1^*,2}^\0}\qty[ \norm{ \hat{\psi}_1 \hat{\phi}_1^\top - \psi_1 \tilde{\phi}_1^\top }_{\mathrm{sp}} ]
   - \check{\P}\qty[ \norm{ \hat{\psi}_1 \hat{\phi}_1^\top - \psi_1 \tilde{\phi}_1^\top }_{\mathrm{sp}} ] }
   \leq \delta }.
\end{equation*}
By Assumption~\ref{asm: stronger stability},
\begin{equation*}
\P\qty(\Omega_{35}) \geq 1 - 2 \exp\qty( - \frac{n_0 \delta^2}{2 K_1^2} ).
\end{equation*}

\begin{equation*}
\Omega_{36} = \qty{
   \check{\P}\qty[ \norm{ \hat{\psi}_1 \hat{\phi}_1^\top - \psi_1 \tilde{\phi}_1^\top }_{\mathrm{sp}} ] \leq \delta }.
\end{equation*}
By Assumption~\ref{asm: stronger stability} and Markov's inequality,
\begin{equation*}
\P\qty(\Omega_{36}) \geq 1 - \frac{\epsilon_{n_0}}{\delta}.
\end{equation*}
Define the master event as the intersection of all the above:
\begin{equation*}
\Omega_3
= \Omega_{31}\cap \Omega_{32}
\cap \bigcap_{j=1}^3 \qty(\Omega_{33}^{(j)}\cap \Omega_{34}^{(j)})
\cap \Omega_{35}\cap \Omega_{36}.
\end{equation*}
By the union bound, we have
\begin{multline*}
\P\qty(\Omega_3)
\ge 1
- \Big(2 + 3\times 2 + 2\Big)\exp\qty(-\frac{n_0\delta^2}{2 K_1^2}) \\
- \Big(1 + 3 + 1\Big)\,\frac{\epsilon_{n_0}}{\delta}
= 1 - 10 \exp\qty(-\frac{n_0\delta^2}{2 K_1^2}) - 5 \frac{\epsilon_{n_0}}{\delta}.
\end{multline*}
Now suppose $\bigcap_{j=1}^3 \Omega_j$ happens. Then
\begin{equation*}
     \check{\P}_{\gD_{r_1^*,2}^\0}\qty[ \norm{ \hat{\psi}_1 \hat{\psi}_1^\top - \psi_1 \psi_1^\top }_{\mathrm{sp}} ] \leq 2 \delta.
\end{equation*}
Hence according to Equation \eqref{eq: defn of i3}, we have
\begin{equation*}
    I_3 \leq \frac{2d \delta}{\underline{\alpha}}.
\end{equation*}
For $I_4$, still assuming $\bigcap_{j=1}^3 \Omega_j$ holds, we seek to use Lemma \ref{lem: ABI-op-relaxed}. Introduce
\begin{equation*}
    \Delta_{\psi\phi,2} =  \norm{ \check{\P}_{\gD_{r_1^*,2}^\0}\qty[ \hat{\psi}_1 \hat{\phi}_1^\top - \psi_1 \tilde{\phi}_1^\top]}_{\mathrm{sp}} \leq  \check{\P}_{\gD_{r_1^*,2}^\0}\qty[ \norm{ \hat{\psi}_1 \hat{\phi}_1^\top - \psi_1 \tilde{\phi}_1^\top }_{\mathrm{sp}} ]\leq 2\delta ,
\end{equation*}
\begin{equation*}
    \Delta_{\phi\phi,2} = \norm{ \check{\P}_{\gD_{r_1^*,2}^\0}\qty[ \sum_{j=1}^3 \alpha_j \hat{\phi}_j\hat{\phi}_j^\top - \sum_{j=1}^3 \alpha_j \tilde{\phi}_j\tilde{\phi}_j^\top]}_{\mathrm{sp}} \leq  \sum_{j=1}^3\check{\P}_{\gD_{r_1^*,2}^\0}\qty[ \norm{ \hat{\phi}_j \hat{\phi}_j^\top - \tilde{\phi}_j\tilde{\phi}_j^\top }_{\mathrm{sp}} ] \leq 6\delta.
\end{equation*}
Moreover, using Equation \eqref{eq: sigam min pd phiphi},
\begin{equation*}
    \sigma_{\min} \qty(\check{\P}_{\gD_{r_1^*,2}^\0}\qty[ \sum_{j=1}^3 \alpha_j \tilde{\phi}_j\tilde{\phi}_j^\top])\geq \underline{\alpha}\qty( K_2 - \delta).   
\end{equation*}
Moreover,
\begin{multline*}
    \sigma_{\min} \qty(\check{\P}_{\gD_{r_1^*,2}^\0}\qty[ \sum_{j=1}^3 \alpha_j \hat{\phi}_j\hat{\phi}_j^\top]) \geq \underline{\alpha} \sigma_{\min} \qty(\check{\P}_{\gD_{r_1^*,2}^\0}\qty[\hat{\phi}_1\hat{\phi}_1^\top]) \\
    \geq \underline{\alpha}\qty(\sigma_{\min} \qty(\check{\P}_{\gD_{r_1^*,2}^\0}\qty[\tilde{\phi}_1\tilde{\phi}_1^\top]) - \norm{\check{\P}_{\gD_{r_1^*,2}^\0}\qty[\hat{\phi}_1\hat{\phi}_1^\top - \tilde{\phi}_1 \tilde{\phi}_1^\top]}_{\mathrm{sp}}) \geq \underline{\alpha}\qty( K_2 - 2\delta).
\end{multline*}
For $\check{\P}_{\gD_{r_1^*,2}^\0}\qty[\psi_1\tilde{\phi}_1^\top]$, combining with Equation \eqref{eq: omega set 25 psiphi},
\begin{equation*}
\norm{\check{\P}_{\gD_{r_1^*,2}^\0}\qty[\psi_1\tilde{\phi}_1^\top]}_{\mathrm{sp}} \leq \norm{\check{\P}\qty[\psi_1\tilde{\phi}_1^\top]}_{\mathrm{sp}} + \norm{\qty(\check{\P}_{\gD_{r_1^*,2}^\0} - \check{\P})\qty[\psi_1\tilde{\phi}_1^\top]}_{\mathrm{sp}} \leq K_3 + \delta.
\end{equation*}
Summarizing the bounds established above, we have
\begin{multline*}
\norm{\check{\P}_{\gD_{r_1^*,2}^\0}\qty[\psi_1\tilde{\phi}_1^\top]}_{\mathrm{sp}} \leq K_3 + \delta,
\quad
\Delta_{\psi\phi,2}\le 2\delta,
\quad
\Delta_{\phi\phi,2}\le 6\delta,
\\
\sigma_{\min}\qty(\check{\P}_{\gD_{r_1^*,2}^\0}\qty[ \sum_{j=1}^3 \alpha_j \tilde{\phi}_j\tilde{\phi}_j^\top])\ge \underline{\alpha}(K_2-\delta),
\quad
\sigma_{\min}\qty(\check{\P}_{\gD_{r_1^*,2}^\0}\qty[ \sum_{j=1}^3 \alpha_j \hat{\phi}_j\hat{\phi}_j^\top])\ge \underline{\alpha}(K_2-2\delta).
\end{multline*}
Therefore, when $\delta < \min \{\frac{K_2}{4}, K_3\}$, apply Lemma \ref{lem: ABI-op-relaxed} to Equation \eqref{eq: defn of i4}, plug in the bound we obtained above, 
\begin{multline*}
    I_4  \leq d\qty(\frac{(K_3 + \delta + \delta)\delta}{\underline{\alpha}\qty( K_2 - 2\delta)}
+6\frac{(K_3 + \delta)(K_3 + \delta + \delta)}{\underline{\alpha}\qty( K_2 - 2\delta)\underline{\alpha}\qty( K_2 - \delta)}\delta
+\frac{(K_3 + \delta)\delta}{\underline{\alpha}\qty( K_2 - \delta)}) \leq \\
\qty(9\frac{K_3}{K_2} +  96\frac{K_3^2}{K_2^2}) \frac{d\delta}{\underline{\alpha}}.
\end{multline*}
Then we have
\begin{equation}
\label{eq: term 2 upper bound}
    \operatorname*{sup}_{\alpha(x, w_1, w_2) \in \gF_2} 
\abs{ \hat{\gL}\qty(\alpha(x, w_1, w_2)) - \tilde{\gL}\qty(\alpha(x, w_1, w_2)) } \leq I_3 + I_4 \leq \qty(2 + 9\frac{K_3}{K_2} +  96\frac{K_3^2}{K_2^2}) \frac{d\delta}{\underline{\alpha}}
\end{equation}
With Equation \eqref{eq: term 1 upper bound} and Equation \eqref{eq: term 2 upper bound}, when $\bigcap_{j=1}^3 \Omega_j$ happens, we have
\begin{multline*}
\gL\qty(\hat{\alpha}^{\text{\method}}(x, w_1, w_2)) - \gL\qty(\alpha^*(x, w_1, w_2)) \leq  \\
\qty(6 + 26\frac{K_3}{K_2} +  198\frac{K_3^2}{K_2^2}) \frac{d\delta}{\underline{\alpha}} + \qty(4 + \frac{12K_3^2}{\underline{\alpha}K_2^2})d\gR_{\frac{n_0}{2}}(\gF_2) + \frac{6dK_3^2}{\underline{\alpha}K_2^2}\delta_{M_{\tilde{\phi}\tilde{\phi}}} + 2d \delta_{M_{\psi \psi}},
\end{multline*}
where
\begin{align*}
&\delta_{M_{\psi\psi}} =\frac{1}{\underline{\alpha}}\qty(M_{\psi\psi} + \frac{8}{\beta} C_{\beta, 1}K_4^2 \qty(\frac{M_{\psi\psi}}{C_{\beta, 1}K_4^2})^{1 - \frac{\beta}{2}} )
   \exp\qty(-\qty(\frac{M_{\psi\psi}}{C_{\beta, 1}K_4^2})^{\frac{\beta}{2}}),\\
&\delta_{M_{\tilde{\phi}\tilde{\phi}}} = \qty(M_{\tilde{\phi}\tilde{\phi}} + \frac{24}{\beta}C_{\beta, 3}K_4^2 \qty(\frac{M_{\tilde{\phi}\tilde{\phi}}}{3C_{\beta, 3}K_4^2})^{1-\frac{\beta}{2}})
\exp\qty( -\qty(\frac{M_{\tilde{\phi}\tilde{\phi}}}{3C_{\beta, 3}K_4^2})^{\frac{\beta}{2}} ).
\end{align*}
Based on the previous analysis,
\begin{align*}
    &\P\qty(\bigcap_{j=1}^3 \Omega_j) \geq 1 - 2 \exp\qty( - \frac{n_0 \delta^2}{2(\rho+1)^2} ) - d^2\exp\qty(-\frac{n_0 \underline{\alpha}^2\delta^2}{4 M_{\psi\psi}^2}) -\\& 2 d^2 \exp\qty( -\qty(\frac{M_{\psi\psi}}{2C_{\beta, 2}K_4^2})^{\frac{\beta}{2}} ) - d^2\exp\qty(-\frac{n_0\delta^2}{4 M_{\tilde{\phi} \tilde{\phi}}^2}) - 2 d^2 \exp\qty( -\qty(\frac{M_{\tilde{\phi}\tilde{\phi}}}{6C_{\beta,4}K_4^2})^{\frac{\beta}{2}} ) - \\&4 d^2 \exp\qty( -\qty(\frac{n_0\delta}{2C_{\beta, 2}K_4^2})^{\frac{\beta}{2}} ) - 10 \exp\qty(-\frac{n_0\delta^2}{2 K_1^2}) - 5 \frac{\epsilon_{n_0}}{\delta}.   
\end{align*}
To choose the truncation parameters $M_{\psi\psi}$ and $M_{\tilde{\phi}\tilde{\phi}}$, recall that we require $\qty(\frac{M_{\psi\psi}}{C_{\beta, 1}K_4^2})^{\frac{\beta}{2}} > 2\qty(\frac{2}{\beta} - 1)$ and $\qty(\frac{M_{\tilde{\phi}\tilde{\phi}}}{3C_{\beta, 3}K_4^2})^{\frac{\beta}{2}} > 2\qty(\frac{2}{\beta} - 1)$ (from the application of Lemma \ref{lem: incomplete gamma function upper bound} above). For sufficiently large $n_0$, we take $M_{\psi\psi} = M_{\tilde{\phi} \tilde{\phi}} = C_{\beta,5} n_0^{\frac{1}{4}}$, where $C_{\beta,5} = \max\{C_{\beta, 1}, 3C_{\beta, 3}\} K_4^2 \cdot [2(\frac{2}{\beta} - 1)]^{\frac{2}{\beta}} + 1$ is chosen to satisfy both constraints. This yields
\begin{align*}
    &\P\qty(\bigcap_{j=1}^3 \Omega_j) \geq 1 - 2 \exp\qty( - \frac{n_0 \delta^2}{2(\rho+1)^2} ) - d^2\exp\qty(-\frac{\sqrt{n_0} \underline{\alpha}^2\delta^2}{4 C_{\beta,5}^2}) -\\& 2 d^2 \exp\qty( -\qty(\frac{C_{\beta,5} n_0^{\frac{1}{4}}}{2C_{\beta, 2}K_4^2})^{\frac{\beta}{2}} ) - d^2\exp\qty(-\frac{\sqrt{n_0}\delta^2}{4 C_{\beta,5}^2}) - \\&2 d^2 \exp\qty( -\qty(\frac{C_{\beta,5} n_0^{\frac{1}{4}}}{6C_{\beta,4}K_4^2})^{\frac{\beta}{2}} ) - 4 d^2 \exp\qty( -\qty(\frac{n_0\delta}{2C_{\beta, 2}K_4^2})^{\frac{\beta}{2}} )\\& - 10 \exp\qty(-\frac{n_0\delta^2}{2 K_1^2}) - 5 \frac{\epsilon_{n_0}}{\delta}.
\end{align*}
The dominant (slowest decay) rate is determined by $\min\{n_0^{\beta/8}, \sqrt{n_0}\delta^2\}$. Counting the coefficients: we have $12$ exponential terms, which can all be bounded by the weakest decay rate. Thus, ignoring lower-order terms and consolidating all constants, with probability at least $1 - 12 d^2 \exp\qty( - C\qty(\rho, \underline{\alpha}, \beta, K_1, K_4) \min\{n_0^{\frac{\beta}{8}}, \sqrt{n_0} \delta^2\}) - 5 \frac{\epsilon_{n_0}}{\delta}$, we have:
\begin{multline*}
    \gL\qty(\hat{\alpha}^{\text{\method}}(x, w_1, w_2)) - \gL\qty(\alpha^*(x, w_1, w_2)) \leq\\
    C\qty( \underline{\alpha}, \beta, K_2, K_3) d \qty( \delta + \gR_{\frac{n_0}{2}}(\gF_2) + n_0^{\frac{1}{4}} \exp\qty(- C\qty(K_4, \beta) n_0^{\frac{\beta}{8}})).
\end{multline*}
\end{proof}

\subsection{Proof of Corollary \ref{cor: consistency get efficiency}}
\label{sec: proof consistency get efficiency}
\begin{proof}
The proof is divided into two steps.

\textbf{Step 1.}
We first show that given any propensity score $\alpha(x, w_1, w_2)$,
our \method~estimator's asymptotic variance achieves the semiparametric efficiency lower bound.
Equivalently, we will show that the asymptotic variance of our estimator is no larger than that of any regular estimator $\hat{\theta}^{\text{any}, \alpha}$.
Here, we clarify the notational convention: the superscript $\alpha$ in $\hat{\theta}^{\text{any}, \alpha}$ and $\Sigma^{\text{any}, \alpha}$ denotes quantities constructed for the given, fixed propensity score $\alpha$.
To this end, note that for any regular estimator, we have
\begin{equation*}
\sqrt{n}\qty(\hat{\theta}^{\text{any}, \alpha} - \theta^*) \xrightarrow{d} \mathcal{N}\qty(0, \Sigma^{\text{any}, \alpha}),
\end{equation*}
and by Lemma \ref{lem: van2000asymptotic convolution} combined with Proposition \ref{prop: EIF MCAR monotone}, since $\psi\qty(\qty(X, Y, W_1, W_2, V)^R, R)$ is the EIF, we have:
\begin{equation*}
\Sigma^{\text{any}, \alpha} \geq \Var\qty(\psi\qty(\qty(X, Y, W_1, W_2, V)^R, R)).
\end{equation*}
On the other hand, for the \method~estimator, let $\hat{\theta}^{\text{\method}, \alpha}$ and $\Sigma^{\text{\method}, \alpha}$ denote the estimator and its asymptotic variance when using the given propensity score $\alpha(x, w_1, w_2)$ (following the same notational convention). By Theorem \ref{thm: asymptotic normality}, we have
\begin{equation*}
\sqrt{n}\qty(\hat{\theta}^{\text{\method}, \alpha} - \theta^*) \xrightarrow{d} \mathcal{N}\qty(0, \Sigma^{\text{\method}, \alpha}).
\end{equation*}
When $\tilde{\phi}_j = \phi_j$, we denote
\begin{multline*}
    \psi^{\text{\method}}\qty(\qty(X, Y, W_1, W_2, V)^R, R) = \\\mathbbm{1}\{R = r_1^*\} \frac{\psi_1(X, Y, W_1, W_2, V)}{\alpha_1(X, W_1, W_2)}
+ M \sum_{j=1}^3 \mathbbm{1}\{R = r_j^*\} \phi_j\qty(\qty(X, Y, W_1, W_2, V)^{r_j^*}).
\end{multline*}
It follows that
\begin{equation*}
\Sigma^{\text{\method}, \alpha} =\Cov\qty(
\psi^{\text{\method}}\qty(\qty(X, Y, W_1, W_2, V)^R, R)
).
\end{equation*}
To establish efficiency, it suffices to prove that $\Sigma^{\text{\method}, \alpha} = \Var\qty(\psi\qty(\qty(X, Y, W_1, W_2, V)^R, R))$.
By the form of the EIF in \eqref{eq: EIF MAR monotone}, we see that verifying $M = I$ would imply
\begin{equation*}
\psi^{\text{\method}}\qty(\qty(X, Y, W_1, W_2, V)^R, R) = \psi\qty(\qty(X, Y, W_1, W_2, V)^R, R),
\end{equation*} which completes the proof. Thus, we now focus on proving that $M = I$.

Recall that
\begin{equation*}
    M = -\check{\P}\qty[\mathbbm{1}\{R = r_1^*\}\frac{\psi_1\tilde{\phi}_1^\top}{\alpha_1}] \qty( \check{\P}\qty[\sum_{j = 1}^3 \mathbbm{1}\{R = r_j^*\} \tilde{\phi}_j\tilde{\phi}_j^\top])^{-1}.
\end{equation*}
We first observe that
\begin{equation*}
\check{\P}\qty[\mathbbm{1}\{R = r^*\} \frac{\psi_1 \phi_1^\top}{\alpha_1}]
= \E\qty[ \E\qty[ \mathbbm{1}\{R = r^*\} \mid X, W_1, W_2, V, Y ] 
\frac{\psi_1 \phi_1^\top}{\alpha_1(X, W_1, W_2)} ]
= \check{\P}\qty[\psi_1 \phi_1^\top].
\end{equation*}
Note that the following relations hold:
\begin{equation*}
\check{\P}\qty[\psi_1 \psi_2^\top] = \check{\P}\qty[\psi_2 \psi_2^\top], \qquad 
\check{\P}\qty[\psi_1 \psi_3^\top] = \check{\P}\qty[\psi_3 \psi_3^\top].
\end{equation*}
Substituting $\phi_j$ from Equation \eqref{eq: phi defn}, we obtain
\begin{equation}
\label{eq: psi1phi1}
\check{\P}\qty[\psi_1 \phi_1^\top] = \check{\P}\qty[
- \frac{\alpha_2}{(\alpha_1 + \alpha_2)\alpha_1} \psi_2 \psi_2^\top 
- \frac{\alpha_3}{\alpha_1 + \alpha_2} \psi_3 \psi_3^\top
].
\end{equation}
By similar calculations, we can derive the following:
\begin{equation*}
\check{\P}\qty[ \alpha_1 \phi_1 \phi_1^\top ] =
\check{\P}\qty[
\frac{\alpha_2^2}{(\alpha_1 + \alpha_2)^2 \alpha_1} \psi_2 \psi_2^\top
+ \frac{\alpha_3^2\alpha_1}{(\alpha_1 + \alpha_2)^2} \psi_3 \psi_3^\top
+ \frac{2\alpha_2 \alpha_3}{(\alpha_1 + \alpha_2)^2} \psi_2 \psi_2^\top
],
\end{equation*}
and
\begin{equation*}
\check{\P}\qty[ \alpha_2 \phi_2 \phi_2^\top ] =
\check{\P}\qty[
\frac{\alpha_2}{(\alpha_1 + \alpha_2)^2} \psi_2 \psi_2^\top
+ \frac{\alpha_2 \alpha_3^2}{(\alpha_1 + \alpha_2)^2} \psi_3 \psi_3^\top
- \frac{2 \alpha_2 \alpha_3}{(\alpha_1 + \alpha_2)^2} \psi_3 \psi_3^\top
],
\end{equation*}
and
\begin{equation*}
\check{\P}\qty[ \alpha_3 \phi_3 \phi_3^\top ] = \check{\P}\qty[ \alpha_3 \psi_3 \psi_3^\top ].
\end{equation*}
Combining these expressions, we arrive at
\begin{equation}
\label{eq: sum phijphij}
\begin{aligned}
\check{\P}\qty[ \sum_{j=1}^3 \mathbbm{1}\{R = r_j^*\} \phi_j \phi_j^\top ] &=
\check{\P}\qty[
\qty( \frac{\alpha_2^2 + \alpha_1 \alpha_2}{(\alpha_1 + \alpha_2)^2 \alpha_1} ) \psi_2 \psi_2^\top
+ \qty( \frac{\alpha_3^2 (\alpha_1 + \alpha_2)}{(\alpha_1 + \alpha_2)^2} + \alpha_3 ) \psi_3 \psi_3^\top
]\\
&= - \check{\P}\qty[ \mathbbm{1}\{R = r^*\} \frac{\psi_1 \phi_1^\top}{\alpha_1} ].
\end{aligned}
\end{equation}
This equation immediately implies that $M = I$. Consequently, we have
\begin{equation*}
\Sigma^{\text{\method}, \alpha} = \Cov\qty( \psi(X, Y, W_1, W_2, V)^R, R),
\end{equation*}
which yields
\begin{equation*}
\Tr\qty(\Sigma^{\text{\method}, \alpha}) \le \Tr\qty(\Sigma^{\text{any},\alpha}).
\end{equation*}
This completes Step 1.

\textbf{Step 2.}
We now leverage the result from Theorem \ref{thm: excess risk upper bound}, which provides the following bound with probability at least $1 - 12 d^2 \exp\qty( - C\qty(\rho, \underline{\alpha}, \beta, K_1, K_4)\min \{ n_0^{\frac{\beta}{8}}, \sqrt{n_0} \delta^2\}) - 5 \frac{\epsilon_{n_0}}{\delta}$:
\begin{multline*}
   \Tr\qty(\Sigma^{\text{\method}, \hat{\alpha}^{\text{\method}}}) - \Tr\qty(\Sigma^{\text{\method}, \alpha^*}) \leq\\
    C\qty( \underline{\alpha}, \beta, K_2, K_3) d \qty( \delta + \gR_{\frac{n_0}{2}}(\gF_2) + n_0^{\frac{1}{4}} \exp\qty(- C\qty(K_4, \beta) n_0^{\frac{\beta}{8}})).
\end{multline*}
Combining this with Step 1, we observe that for any regular estimator $\hat{\theta}_{\text{any}, \alpha}$ based on the propensity score $\alpha(x, w_1, w_2)$, the following chain of inequalities holds:
\begin{equation*}
    \Tr\qty(\Sigma^{\text{\method}, \alpha^*}) \leq \Tr\qty(\Sigma^{\text{\method}, \alpha}) \leq \Tr\qty(\Sigma^{\text{any}, \alpha})
\end{equation*}
where the first inequality follows from the optimality of $\alpha^*$ (which minimizes the asymptotic variance among all propensity scores), and the second inequality follows from Step 1 (our estimator achieves the semiparametric efficiency bound).
Therefore, with probability at least $1 - 12 d^2 \exp\qty( - C\qty(\rho, \underline{\alpha}, \beta, K_1, K_4)\min \{ n_0^{\frac{\beta}{8}}, \sqrt{n_0} \delta^2\}) - 5 \frac{\epsilon_{n_0}}{\delta}$, we obtain:
\begin{multline*}
   \Tr\qty(\Sigma^{\text{\method}, \hat{\alpha}^{\text{\method}}}) - \Tr\qty(\Sigma^{\text{any}, \alpha}) \leq\\
    C\qty( \underline{\alpha}, \beta, K_2, K_3) d \qty( \delta + \gR_{\frac{n_0}{2}}(\gF_2) + n_0^{\frac{1}{4}} \exp\qty(- C\qty(K_4, \beta) n_0^{\frac{\beta}{8}})).
\end{multline*}
This completes the proof.
\end{proof}

\subsection{Proof of Corollary \ref{cor: keep safe}}
\label{sec: proof keep safe}
\begin{proof}
\textbf{First part.} Consider the estimator $\hat{\theta}^{\text{label-unlabel}}$ corresponding to the policy
\begin{equation*}
\alpha^{\text{label}}(x, w_1, w_2) = \qty( \tau/\rho,\ 0, 1 - \tau/\rho ) \in \gF_1,   
\end{equation*}
which allocates budget only to labeled data and pseudo-labels. Since $\alpha^{\text{label}} \in \gF_1$, by optimality of $\alpha^*(x, w_1, w_2)$, we have
\begin{equation*}
\gL(\alpha^{\text{label}}(x, w_1, w_2)) \geq \gL(\alpha^*(x, w_1, w_2))
\end{equation*}
By applying Theorem \ref{thm: excess risk upper bound}, with probability at least $1 - 12 d^2 \exp\qty( - C\qty(\rho, \underline{\alpha}, \beta, K_1, K_4)\min \{ n_0^{\frac{\beta}{8}}, \sqrt{n_0} \delta^2\}) - 5 \frac{\epsilon_{n_0}}{\delta}$, we obtain
\begin{multline*}
    \gL\qty(\hat{\alpha}^{\text{\method}}(x, w_1, w_2)) - \gL\qty(\alpha^{\text{label}}(x, w_1, w_2)) \leq \gL\qty(\hat{\alpha}^{\text{\method}}(x, w_1, w_2)) - \gL\qty(\alpha^*(x, w_1, w_2))\\
    \leq C\qty( \underline{\alpha}, \beta, K_2, K_3) d \qty( \delta + \gR_{\frac{n_0}{2}}(\gF_2) + n_0^{\frac{1}{4}} \exp\qty(- C\qty(K_4, \beta) n_0^{\frac{\beta}{8}})).
\end{multline*}
Converting this excess risk bound to the asymptotic variance, we obtain the following equivalent statement:
\begin{multline*}
    \lim_{n_1 \to \infty} \Tr(\Cov(\hat{\theta}^{\text{\method}})) - \lim_{n_1 \to \infty} \Tr(\Cov(\hat{\theta}^{\text{label-unlabel}})) \leq \\
    C\qty( \underline{\alpha}, \beta, K_2, K_3) d \qty( \delta + \gR_{\frac{n_0}{2}}(\gF_2) + n_0^{\frac{1}{4}} \exp\qty(- C\qty(K_4, \beta) n_0^{\frac{\beta}{8}})).
\end{multline*}

\textbf{Second part.} We now compare the asymptotic variance of $\hat{\theta}^{\text{label-unlabel}}$ with the label-only estimator $\hat{\theta}^{\text{label-only}}$. Under the policy $\alpha^{\text{label}}(x, w_1, w_2)$, the asymptotic covariances are given by
\begin{equation*}
\lim_{n \to \infty} \Cov\qty( \hat{\theta}^{\text{label-unlabel}} ) 
= \Sigma_{\text{\method}, \alpha} 
= \Cov\qty( \psi\qty((X, Y, W_1, W_2, V)^R, R))
\end{equation*}
\begin{equation*}
\lim_{n_1 \to \infty} \Cov\qty( \hat{\theta}^{\text{label-only}} ) 
= \Cov\qty(\psi_1(X, Y, W_1, W_2, V) ).
\end{equation*}
To show that incorporating pseudo-labels reduces variance, we examine the trace of the covariance. By the structure of the influence function $\psi\qty((X, Y, W_1, W_2, V)^R, R)$, it follows that
\begin{equation*}
\begin{aligned}
&\Tr\qty(\Cov\qty( \psi\qty((X, Y, W_1, W_2, V)^R, R))) \\
=  &\Tr\qty(\Cov\qty(
\mathbbm{1}\{R = r_1^*\} \frac{\psi_1(X, Y, W_1, W_2, V)}{\alpha_1(X, W_1, W_2)})) \\
&-\Tr\qty(\Cov\qty( M \sum_{j=1}^3 \mathbbm{1}\{R = r_j^*\} \tilde{\phi}_j\qty(\qty(X, Y, W_1, W_2, V)^{r_j^*})))\\
=  &\Tr\qty(\Cov\qty(\psi_1(X, Y, W_1, W_2, V) )) \\
&- \Tr\qty(\Cov\qty( M \sum_{j=1}^3 \mathbbm{1}\{R = r_j^*\} \tilde{\phi}_j\qty(\qty(X, Y, W_1, W_2, V)^{r_j^*})))\\
\leq &\Tr\qty(\Cov\qty(\psi_1(X, Y, W_1, W_2, V) )).
\end{aligned}
\end{equation*}
Combining the results from both parts, we conclude that
\begin{equation*}
    \lim_{n_1 \to \infty} \Tr(\Cov(\hat{\theta}^{\text{label-unlabel}}))  \leq  \lim_{n_1 \to \infty} \Tr(\Cov(\hat{\theta}^{\text{label-only}})).
\end{equation*}
\end{proof}

\subsection{Proof of Corollary \ref{cor: need label data}}
\begin{proof}
We show that the asymptotic variance diverges as the budget allocation to labeled data vanishes. 
When we can estimate the nuisance well (i.e., $\tilde{\phi}_j = \phi_j$), for any policy $\alpha(x, w_1, w_2)$, based on Equation \eqref{eq: variance population}, Equation \eqref{eq: psi1phi1} and Equation \eqref{eq: sum phijphij}, the asymptotic loss can be decomposed as
\begin{equation}
\label{eq: loss based on phij}
\begin{aligned}
\gL\qty(\alpha(x, w_1, w_2)) &= \Tr\qty(\check{\P}\qty[\frac{1}{\alpha_1} \psi_1 \psi_1^\top]
+ \check{\P}\qty[
- \frac{\alpha_2}{(\alpha_1 + \alpha_2)\alpha_1} \psi_2 \psi_2^\top
- \frac{1}{\alpha_1 + \alpha_2} \psi_3 \psi_3^\top
+ \psi_3 \psi_3^\top])\\
&= \Tr\qty(\check{\P}\qty[
\frac{1}{\alpha_1} \psi_1 \psi_1^\top
- \frac{\alpha_2}{\alpha_1 + \alpha_2} \psi_2 \psi_2^\top
- \frac{\alpha_1}{\alpha_1 + \alpha_2} \psi_3 \psi_3^\top
]) + \Tr\qty(\check{\P}\qty[\psi_3 \psi_3^\top])
\end{aligned}
\end{equation}
Since the term $\Tr\qty(\check{\P}[\psi_3 \psi_3^\top])$ is finite and does not depend on $\alpha(x, w_1, w_2)$, we focus on the remaining terms. Taking the trace, we obtain
\begin{equation*}
\begin{aligned}
&\Tr\qty(
\check{\P}\qty[
\frac{1}{\alpha_1} \psi_1 \psi_1^\top
- \frac{\alpha_2}{\alpha_1 + \alpha_2} \psi_2 \psi_2^\top
- \frac{\alpha_1}{\alpha_1 + \alpha_2} \psi_3 \psi_3^\top
])\\
=& \check{\P}\qty[\frac{1}{\alpha_1}\Tr\qty(
\psi_1 \psi_1^\top
- \frac{\alpha_2}{\alpha_1 + \alpha_2} \psi_2 \psi_2^\top
- \frac{\alpha_1}{\alpha_1 + \alpha_2} \psi_3 \psi_3^\top
)]\\
\geq& \check{\P}\qty[\frac{1}{\alpha_1}\Tr\qty(
\psi_1 \psi_1^\top - \psi_2 \psi_2^\top
)],
\end{aligned}
\end{equation*}
where the last inequality uses the inequality $ \check{\P}[\norm{\psi_2}^2] \geq \check{\P}[\norm{\psi_3}^2]$. 
Now, define the gap between the norms of $\psi_1$ and $\psi_2$ as
\begin{equation*}
\check{\P}\qty[\Tr\qty(\psi_1 \psi_1^\top - \psi_2 \psi_2^\top)]
= \check{\P}\qty[\norm{\psi_1}^2] - \check{\P}\qty[\norm{\psi_2}^2]
\triangleq \delta_{\psi_{12}} > 0
\end{equation*}
where the strict inequality holds because $\psi_1 \neq \psi_2$.

To show that $\gL\qty(\alpha(x, w_1, w_2)) \rightarrow \infty$ as $\alpha_1(x, w_1, w_2) \rightarrow 0$, let $C > 0$ be an arbitrary constant. Choose $\alpha_1(x, w_1, w_2)$ small enough such that 
$\alpha_1(x, w_1, w_2) \leq \frac{\delta_{\psi_{12}}}{C}$.
Then, we have
\begin{equation*}
\mathcal{L}(\alpha(x, w_1, w_2)) 
\geq \frac{C}{\delta_{\psi_{12}}} \delta_{\psi_{12}} 
= C.
\end{equation*}
Since $C$ is arbitrary, this shows that $\gL\qty(\alpha(x, w_1, w_2)) \rightarrow \infty$ as $\alpha_1(x, w_1, w_2) \rightarrow 0$, completing the proof.
\end{proof}

\section{Proofs for Covariate-Agnostic Case }
\label{sec: app proof covariate agnostic}
\subsection{Proof of Theorem \ref{thm: mcar excess risk upper bound}}
\begin{proof}
The proof follows a similar strategy to that of Theorem \ref{thm: excess risk upper bound}, but adapted to the constrained setting by working with the propensity score class $\gF_3$.
Recall that
\begin{equation*}
\gF_3 = \qty{ \alpha \middle| \,
   \alpha_j = \alpha_j, j \in [3],
   \alpha_1 + \alpha_2 + \alpha_3 = 1,  \rho\alpha_1 + \alpha_2 = \tau, \alpha_1 \geq \underline{\alpha} }.
\end{equation*}
For this constrained class, the empirical optimal design is obtained by minimizing the empirical loss:
\begin{equation}
\label{eq: optimization problem for alpha scalar empirical}
\hat{\alpha}^{\text{\method-CA}}
= \argmin_{\alpha \in \gF_3} 
\hat{\gL}(\alpha),
\end{equation}
where the empirical loss function is given by
\begin{multline}
\label{eq: empirical loss function}
\hat{\gL}(\alpha) 
:= \Tr\left(
\check{\P}_{\gD_{r_1^*, 2}^\0}\qty[\frac{\hat{\psi}_1 \hat{\psi}_1^\top}{\alpha_1}] \right. \\
\left.
- \check{\P}_{\gD_{r_1^*, 2}^\0}\qty[\hat{\psi}_1 \hat{\phi}_1^\top]
  \qty(\check{\P}_{\gD_{r_1^*, 2}^\0}\qty[\sum_{j=1}^3 \alpha_j\, \hat{\phi}_j \hat{\phi}_j^\top])^{-1}
  \check{\P}_{\gD_{r_1^*, 2}^\0}\qty[\hat{\phi}_1 \hat{\psi}_1^\top]
\right).
\end{multline}
We split the excess risk as
\begin{multline*}
\gL(\hat{\alpha}^{\text{\method-CA}}) - \gL(\alpha^*)
= \gL(\hat{\alpha}^{\text{\method-CA}}) - \hat{\gL}(\hat{\alpha}^{\text{\method-CA}})\\
+ \hat{\gL}(\hat{\alpha}^{\text{\method-CA}}) - \hat{\gL}(\alpha^*)
+ \hat{\gL}(\alpha^*) - \gL(\alpha^*).
\end{multline*}
By the optimality of $\hat{\alpha}^{\text{\method-CA}}$ over $\gF_3$, the middle term is non-positive, yielding
\begin{equation*}
\gL\qty(\hat{\alpha}^{\text{\method-CA}}) - \gL\qty(\alpha^*)
\leq 2 \operatorname*{sup}_{\alpha \in \gF_3}
\abs{ \hat{\gL}\qty(\alpha) - \gL\qty(\alpha) }.
\end{equation*}
Thus, it suffices to bound the uniform deviation $\operatorname*{sup}_{\alpha \in \gF_3}
\abs{ \hat{\gL}\qty(\alpha) - \gL\qty(\alpha) }$.

To analyze this uniform deviation, we introduce an intermediate loss function and apply the triangle inequality:
\begin{equation*}
\operatorname*{sup}_{\alpha \in \gF_3}
\abs{ \hat{\gL}\qty(\alpha) - \gL\qty(\alpha) } \leq
\operatorname*{sup}_{\alpha \in \gF_3}   \abs{ \tilde{\gL}\qty(\alpha) - \gL\qty(\alpha) }
+ \operatorname*{sup}_{\alpha \in \gF_2} 
\abs{ \hat{\gL}\qty(\alpha) - \tilde{\gL}\qty(\alpha) },
\end{equation*}
where
\begin{multline*}
\tilde{\gL}(\alpha) 
:= \Tr\Bigg(
\check{\P}_{\gD_{r_1^*, 2}^\0}\qty[\frac{\psi_1 \psi_1^\top}{\alpha_1}] \\
- \check{\P}_{\gD_{r_1^*, 2}^\0}\qty[\psi_1 \tilde{\phi}_1^\top]
  \qty(\check{\P}_{\gD_{r_1^*, 2}^\0}\qty[\sum_{j=1}^3 \alpha_j\, \tilde{\phi}_j \tilde{\phi}_j^\top])^{-1}
  \check{\P}_{\gD_{r_1^*, 2}^\0}\qty[\tilde{\phi}_1 \psi_1^\top]
\Bigg).
\end{multline*}
The intermediate loss $\tilde{\gL}(\alpha)$ uses the true influence function $\psi_1$ but the estimated control variates $\tilde{\phi}_j$.

Applying the same concentration and uniform convergence techniques as in Section \ref{sec: app proof excess risk}, we obtain that with probability at least
\begin{align*}
    &1 - d^2\exp\qty(-\frac{n_0 \underline{\alpha}^2\delta^2}{4 M_{\psi\psi}^2}) - 2 d^2 \exp\qty( -\qty(\frac{M_{\psi\psi}}{2C_{\beta, 2}K_4^2})^{\frac{\beta}{2}} ) - d^2\exp\qty(-\frac{n_0\delta^2}{4 M_{\tilde{\phi} \tilde{\phi}}^2}) -\\& 2 d^2 \exp\qty( -\qty(\frac{M_{\tilde{\phi}\tilde{\phi}}}{6C_{\beta,4}K_4^2})^{\frac{\beta}{2}} ) -4 d^2 \exp\qty( -\qty(\frac{n_0\delta}{2C_{\beta, 2}K_4^2})^{\frac{\beta}{2}} ) - 10 \exp\qty(-\frac{n_0\delta^2}{2 K_1^2}) - 5 \frac{\epsilon_{n_0}}{\delta}, 
\end{align*}
the following bound holds:
\begin{multline*}
\gL\qty(\hat{\alpha}^{\text{\method-CA}}) - \gL\qty(\alpha^*) \leq  \\
\qty(6 + 26\frac{K_3}{K_2} +  198\frac{K_3^2}{K_2^2}) \frac{d\delta}{\underline{\alpha}} + \qty(4 + \frac{12K_3^2}{\underline{\alpha}K_2^2})d\gR_{\frac{n_0}{2}}(\gF_3) + \frac{6dK_3^2}{\underline{\alpha}K_2^2}\delta_{M_{\tilde{\phi}\tilde{\phi}}} + 2d \delta_{M_{\psi \psi}},
\end{multline*}
where
\begin{align*}
&\delta_{M_{\psi\psi}} = \frac{1}{\underline{\alpha}}\qty(M_{\psi\psi} + 4 \qty(\frac{M_{\psi\psi}}{C_{\beta, 1}K_4^2})^{1 - \frac{\beta}{2}} )\exp\qty(-\qty(\frac{M_{\psi\psi}}{C_{\beta, 1}K_4^2})^{\frac{\beta}{2}}),\\
&\delta_{M_{\tilde{\phi}\tilde{\phi}}} = \qty(M_{\tilde{\phi}\tilde{\phi}} + 4 \qty(\frac{M_{\tilde{\phi}\tilde{\phi}}}{3C_{\beta, 3}K_4^2})^{1-\frac{\beta}{2}})\exp\qty( -\qty(\frac{M_{\tilde{\phi}\tilde{\phi}}}{3C_{\beta, 3}K_4^2})^{\frac{\beta}{2}} ).
\end{align*}

Finally, similar to the proof of Theorem \ref{thm: excess risk upper bound} in Section \ref{sec: app proof excess risk}, we conclude that with probability at least $1 - 12 d^2 \exp\qty( - C\qty(\rho, \underline{\alpha}, \beta, K_1, K_4)\min \{ n_0^{\frac{\beta}{8}}, \sqrt{n_0} \delta^2\}) - 5 \frac{\epsilon_{n_0}}{\delta}$, the excess risk satisfies
\begin{multline*}
    \gL\qty(\hat{\alpha}^{\text{\method-CA}}) - \gL\qty(\alpha^*) \leq\\
    C\qty( \underline{\alpha}, \beta, K_2, K_3) d \qty( \delta + \gR_{\frac{n_0}{2}}(\gF_3) + n_0^{\frac{1}{4}} \exp\qty(- C\qty(K_4, \beta) n_0^{\frac{\beta}{8}})).
\end{multline*}
\end{proof}

\subsection{Proof of Corollary \ref{cor: mcar keep safe}}
\begin{proof}
The proof follows a similar approach to Section \ref{sec: proof consistency get efficiency}. 

First, we establish the semi-parametric efficiency. Given any propensity score $\alpha$, for any regular estimator, we have
\begin{equation*}
\sqrt{n}\qty(\hat{\theta}^{\text{any}, \alpha} - \theta^*) \xrightarrow{d} \mathcal{N}\qty(0, \Sigma^{\text{any}, \alpha})
\end{equation*}
and for our \method-CA~estimator with the given propensity score $\alpha$,
\begin{equation*}
\sqrt{n}\qty(\hat{\theta}^{\text{\method-CA}, \alpha} - \theta^*) \xrightarrow{d} \mathcal{N}\qty(0, \Sigma^{\text{\method-CA}, \alpha}).
\end{equation*}
By the semi-parametric efficiency theory (as used in Section \ref{sec: proof consistency get efficiency}), our estimator achieves the efficiency bound, so
\begin{equation}
\label{eq: mcar paci trace lower bound}
\Tr\qty(\Sigma^{\text{\method-CA}, \alpha}) \le \Tr\qty(\Sigma^{\text{any},\alpha}).
\end{equation}
Next, we leverage theexcess risk bound from Theorem \ref{thm: mcar excess risk upper bound}. With probability at least $1 - 12 d^2 \exp\qty( - C\qty(\rho, \underline{\alpha}, \beta, K_1, K_4)\min \{ n_0^{\frac{\beta}{8}}, \sqrt{n_0} \delta^2\}) - 5 \frac{\epsilon_{n_0}}{\delta}$, we have
\begin{multline*}
   \Tr\qty(\Sigma^{\text{\method-CA}, \hat{\alpha}^{\text{\method-CA}}}) - \Tr\qty(\Sigma^{\text{\method-CA}, \alpha^*}) \leq\\
    C\qty( \underline{\alpha}, \beta, K_2, K_3) d \qty( \delta + \gR_{\frac{n_0}{2}}(\gF_3) + n_0^{\frac{1}{4}} \exp\qty(- C\qty(K_4, \beta) n_0^{\frac{\beta}{8}})).
\end{multline*}
To complete the proof, we combine this finite-sample bound with the efficiency lower bound \eqref{eq: mcar paci trace lower bound}. Note that for any regular estimator $\hat{\theta}_{\text{any}, \alpha}$ based on the propensity score $\alpha$, we have the following ordering:
\begin{equation*}
    \Tr\qty(\Sigma^{\text{\method-CA}, \alpha^*}) \leq \Tr\qty(\Sigma^{\text{\method-CA}, \alpha}) \leq \Tr\qty(\Sigma^{\text{any}, \alpha})
\end{equation*}
where the first inequality follows from the optimality of $\alpha^*$ and the second from \eqref{eq: mcar paci trace lower bound}.
Combining these inequalities with the excess risk bound, we conclude that with probability at least $1 - 12 d^2 \exp\qty( - C\qty(\rho, \underline{\alpha}, \beta, K_1, K_4)\min \{ n_0^{\frac{\beta}{8}}, \sqrt{n_0} \delta^2\}) - 5 \frac{\epsilon_{n_0}}{\delta}$,
\begin{multline*}
   \Tr\qty(\Sigma^{\text{\method-CA}, \hat{\alpha}^{\text{\method-CA}}}) - \Tr\qty(\Sigma^{\text{any}, \alpha}) \leq\\
    C\qty( \underline{\alpha}, \beta, K_2, K_3) d \qty( \delta + \gR_{\frac{n_0}{2}}(\gF_3) + n_0^{\frac{1}{4}} \exp\qty(- C\qty(K_4, \beta) n_0^{\frac{\beta}{8}})).
\end{multline*}
This completes the proof.
\end{proof}

\subsection{Proof of Corollary \ref{cor: mcar keep safe}}
\begin{proof}
The proof follows a similar approach to Section \ref{sec: proof consistency get efficiency}.
Notice that
\begin{equation*}
\alpha^{\text{label}} = \qty( \tau/\rho,\ 0, 1 - \tau/\rho ) \in \gF_3.
\end{equation*}
By optimality of $\alpha^*$, we have
\begin{equation*}
\gL(\alpha^{\text{label}}) \geq \gL(\alpha^*).
\end{equation*}
By applying Theorem \ref{thm: mcar excess risk upper bound}, with probability at least $1 - 12 d^2 \exp\qty( - C\qty(\rho, \underline{\alpha}, \beta, K_1, K_4)\min \{ n_0^{\frac{\beta}{8}}, \sqrt{n_0} \delta^2\}) - 5 \frac{\epsilon_{n_0}}{\delta}$, we obtain
\begin{multline*}
    \gL\qty(\hat{\alpha}^{\text{\method-CA}}) - \gL\qty(\alpha^{\text{label}}) \leq \gL\qty(\hat{\alpha}^{\text{\method-CA}}) - \gL\qty(\alpha^*)\\
    \leq C\qty( \underline{\alpha}, \beta, K_2, K_3) d \qty( \delta + \gR_{\frac{n_0}{2}}(\gF_3) + n_0^{\frac{1}{4}} \exp\qty(- C\qty(K_4, \beta) n_0^{\frac{\beta}{8}})).
\end{multline*}
Converting this excess risk bound to the asymptotic variance, we obtain the following equivalent statement:
\begin{multline*}
    \lim_{n_1 \to \infty} \Tr(\Cov(\hat{\theta}^{\text{\method-CA}})) - \lim_{n_1 \to \infty} \Tr(\Cov(\hat{\theta}^{\text{label-unlabel}})) \leq \\
    C\qty( \underline{\alpha}, \beta, K_2, K_3) d \qty( \delta + \gR_{\frac{n_0}{2}}(\gF_3) + n_0^{\frac{1}{4}} \exp\qty(- C\qty(K_4, \beta) n_0^{\frac{\beta}{8}})).
\end{multline*}
\end{proof}

\subsection{Proof of Corollary \ref{cor: need preference data}}
\begin{proof}
We show that under the condition $\rho > \frac{e_1 - e_3}{e_2 - e_3}$, allocating zero budget to preference data is suboptimal.

Starting from Equation \eqref{eq: loss based on phij}, the asymptotic loss can be expressed as
\begin{equation*}
\gL(\alpha) = 
\Tr\qty(
\check{\P} \qty[
\frac{1}{\alpha_1} \psi_1 \psi_1^\top
- \frac{\alpha_2}{(\alpha_1 + \alpha_2)\alpha_1} \psi_2 \psi_2^\top
+ \qty(1 - \frac{1}{\alpha_1 + \alpha_2}) \psi_3 \psi_3^\top
]).
\end{equation*}
To analyze this loss as a function of $\alpha_1$, we substitute $\alpha_2, \alpha_3$ using the budget constraint \eqref{eq: budget constraint simplified}. Let $e_j := \Tr\qty( \check{\P}[\psi_j \psi_j^\top] )$ denote the trace of the covariance of each influence function. This yields
\begin{equation*}
\gL(\alpha) 
= \frac{1}{\alpha_1} e_1 
- \frac{\alpha_2}{(\alpha_1 + \alpha_2)\alpha_1} e_2 
+ \qty(1 - \frac{1}{\alpha_1 + \alpha_2}) e_3,
\end{equation*}
\begin{equation*}
\gL_1(\alpha_1) 
= \frac{e_1}{\alpha_1} 
- \frac{\tau - \rho \alpha_1}{\alpha_1(\tau + (1 - \rho)\alpha_1)} e_2
+ \qty(1 - \frac{1}{\tau + (1 - \rho)\alpha_1}) e_3,
\end{equation*}
where $\gL_1(\alpha_1) := \gL(\alpha)$ denotes the loss as a univariate function of $\alpha_1$ after substituting the budget constraints.

To determine whether preference data should be used, we evaluate the derivative at the boundary point $\alpha_1 = \frac{\tau}{\rho}$, where $\alpha_2 = 0$ (no budget allocated to preference data). Computing this derivative gives
\begin{equation*}
\gL_1'\qty(\frac{\tau}{\rho})
= \frac{\rho^2}{\tau^2} \qty( \rho(e_2 - e_3) - (e_1 - e_3) )
\end{equation*}
Observe that when the cost parameter satisfies
\begin{equation*}
\rho > \frac{e_1 - e_3}{e_2 - e_3},
\end{equation*}
we have $\gL_1'\qty(\frac{\tau}{\rho}) > 0$, meaning the loss is increasing at the boundary.

Since the derivative is positive, decreasing $\alpha_1$ from $\frac{\tau}{\rho}$ (and thus increasing $\alpha_2$ from 0) reduces the loss. Therefore,
\begin{equation*}
\gL_1(\frac{\tau}{\rho}) > \min_{\alpha \in \gF_3} \gL(\alpha),
\end{equation*}
which implies that the optimal policy must allocate positive budget to preference data, i.e., $\alpha_2^* \neq 0$.
\end{proof}

\section{Additional Related Work}
\label{sec: app additional related work}
\paragraph{RLHF \& Pairwise-comparison.}
Data from pairwise comparisons are widely used in numerous fields, including for ranking search results \citep{burges2005learning, lee2024methods} and evaluating text generation quality \citep{stiennon2020learning}. This data structure has also become central to modern AI for capturing human preferences. 

This pairwise comparison data is critical, as LLMs can exhibit a range of unintended biased behaviors, such as fabricating facts ("hallucinations"), generating biased text, or failing to follow user instructions \citep{kenton2021alignment, bender2021dangers, gehman2020realtoxicityprompts}. Then, pairwise comparison data becomes the cornerstone of powerful techniques like Reinforcement Learning from Human Feedback (RLHF), where preferences between two model outputs are used to align large language models with human values \citep{ouyang2022traininglanguagemodelsfollow, christiano2017deep, stiennon2020learning, rafailov2023direct}.
\paragraph{Semi-parametric Inference \& Missing Data.} Our work builds upon the rich literature of semi-parametric inference, which provides a formal framework for improving statistical efficiency by leveraging auxiliary information, particularly in settings with missing data \citep{robins1994estimation, chernozhukov2018double, robins1995semiparametric, rubin1976inference, rubin1996multiple}. This classical approach has seen a modern resurgence in the context of Prediction-Powered Inference (PPI) \citep{angelopoulos2023prediction, angelopoulos2023ppi++}, where machine learning predictions are used as auxiliary information to enhance inference when true outcomes are scarce. Recent advancements in this area include the development of decorrelation matrices to construct robust estimators \citep{miao2023assumption, gan2024prediction}, alongside theoretical contributions that establish the connection between semiparametric frameworks and prediction-powered inference (PPI), demonstrating its statistical efficiency under well-defined conditions \citep{ji2025predictionssurrogatesrevisitingsurrogate, xu2025unified}.
\paragraph{Experiment Design \& Active Learning.} 
The fields of experiment design and active learning both address the challenge of optimal data acquisition, however, from different perspectives. Experiment design traditionally involves dynamically adjusting the allocation of treatments across populations to optimize the accuracy of estimated treatment effects \citep{lai1985asymptotically, hahn2011adaptive, li2024optimaladaptiveexperimentaldesign}. In contrast, active learning, closer to machine learning, seeks to improve the model performance by selectively querying labels for the most informative unlabeled data \citep{ash2019deep, settles2009active, gal2017deep}. While recent studies explore data selection under budget constraints \citep{zrnic2024active, angelopoulos2025costoptimalactiveaimodel, li2025robust}, their methods lack formal guarantees on the statistical efficiency of the final estimators. On the other hand, \citet{ao2024predictionguidedactiveexperiments} provides asymptotic efficiency under certain conditions. A critical limitation of this latter method, however, is its restricted focus on only two data types and its reliance on deriving a closed-form expression for the optimal strategy. \citet{zhang2025efficient} also considers only two data types; \citet{imberg2025active} lacks efficiency guarantees and relies on closed-form solutions. Other recent works \citep{goebel2025budgeted, morato2025conbatch} also consider active learning under budget constraints, but these methods focus on single data type acquisition (e.g., labels only) and lack formal guarantees on statistical efficiency or asymptotic optimality for estimating statistical functionals.
Our work addresses these limitations by pursuing a broader objective. Rather than simply identifying the optimal labeling strategy in a constrained setting, we develop a method to achieve minimum estimator variance within a more general monotone missing data framework \citep{wang2024maximinoptimalapproachsampling}. Our approach is designed to handle multiple data types and directly derives the optimal acquisition strategy by minimizing the asymptotic variance of the proposed estimator, removing the need for a closed-form solution.

\section{Additional Experiment Results}
\label{sec: app additional figures}

\subsection{Simulation Study}
\label{sec: app additional simulation figures}

This section provides additional simulation results for the linear regression experiment described in Section~\ref{sec: experiment}. Figures~\ref{fig:ci_cov_c5_appendix} and \ref{fig:ci_cov_c20_appendix} show the confidence interval lengths and coverage rates for cost settings $c=5$ and $c=20$, respectively. The patterns are consistent with those observed for $c=10$ in the main text: both \method~and \method-CA~produce substantially shorter confidence intervals than the baselines while maintaining proper coverage near the nominal 90\% level.

\begin{figure}[!ht]
  \centering
  \includegraphics[width=\linewidth]{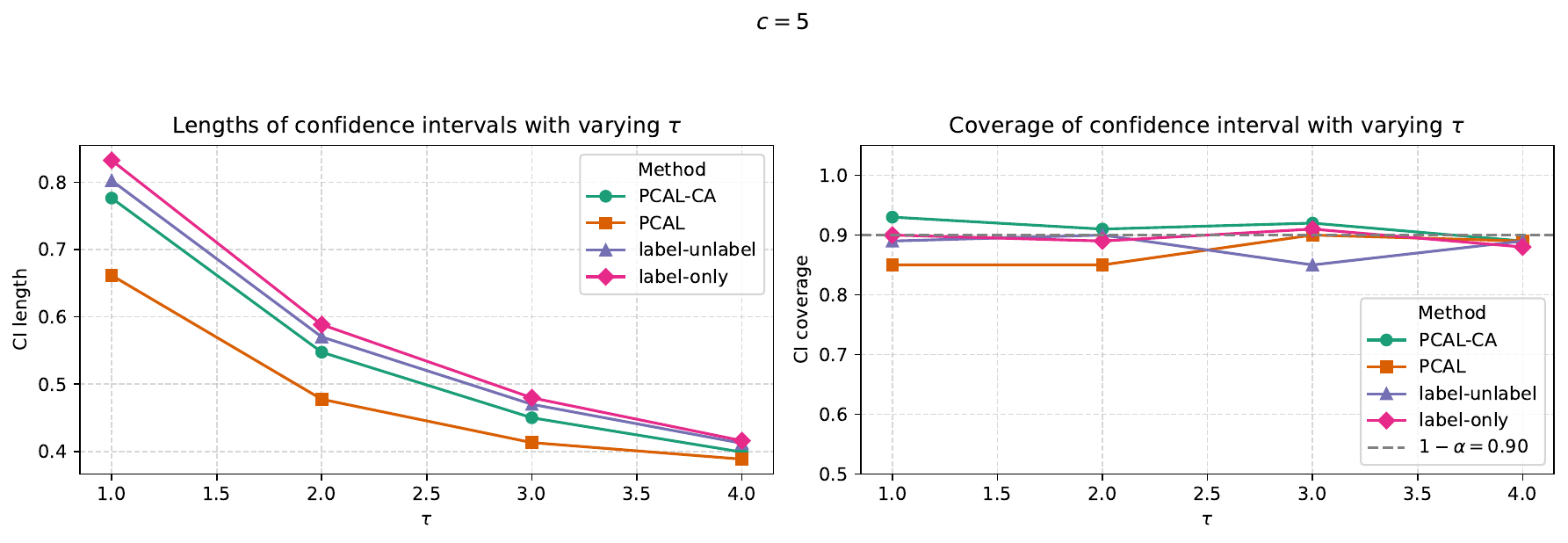}
  \caption{Confidence interval results of $\theta_1^*$ with $c=5$.}
  \label{fig:ci_cov_c5_appendix}
\end{figure}

\begin{figure}[!ht]
  \centering
  \includegraphics[width=\linewidth]{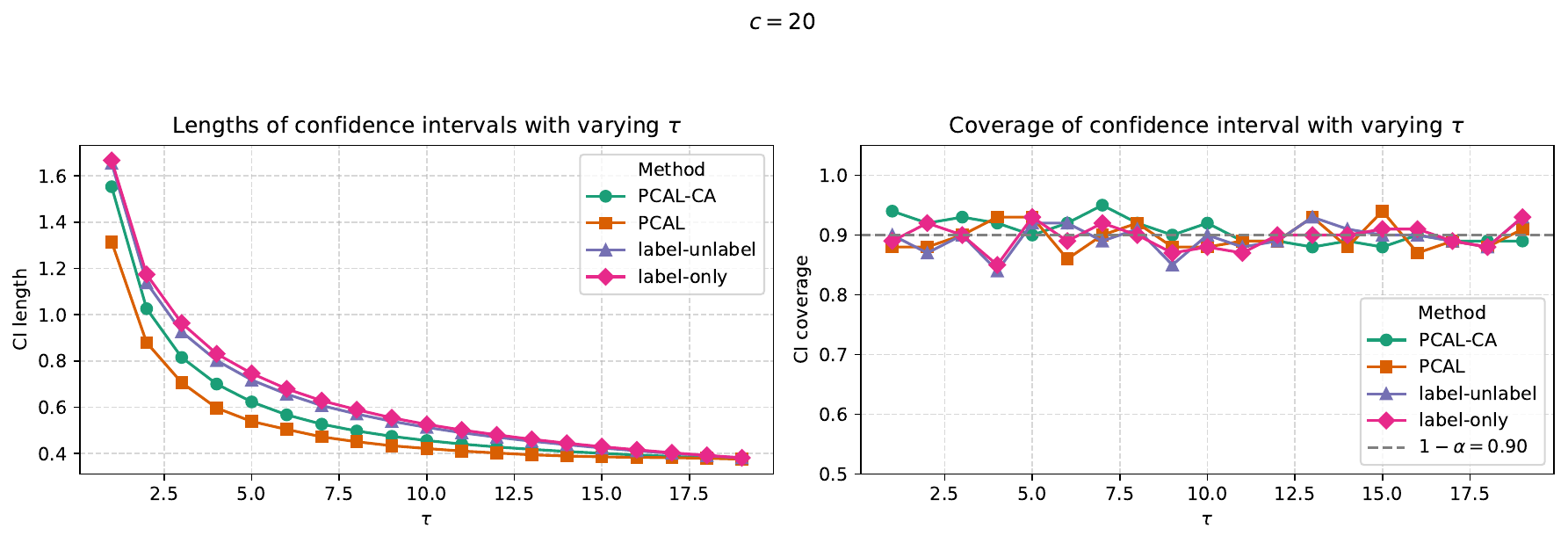}
  \caption{Confidence interval results of $\theta_1^*$ with $c=20$.}
  \label{fig:ci_cov_c20_appendix}
\end{figure}

\subsection{Real Data Analysis}
\label{sec: app additional realdata figures}

This section provides additional results for the politeness dataset analysis described in Section~\ref{sec: experiment}. Figures~\ref{fig:politeness_c20_appendix}--\ref{fig:politeness_c200_appendix} show results for cost settings $c \in \{20, 50, 200\}$. Across all settings, \method~consistently achieves the shortest confidence intervals while maintaining coverage close to the nominal 90\% level, confirming the robustness of our method across different cost structures.

\begin{figure}[!ht]
  \centering
  \includegraphics[width=\linewidth]{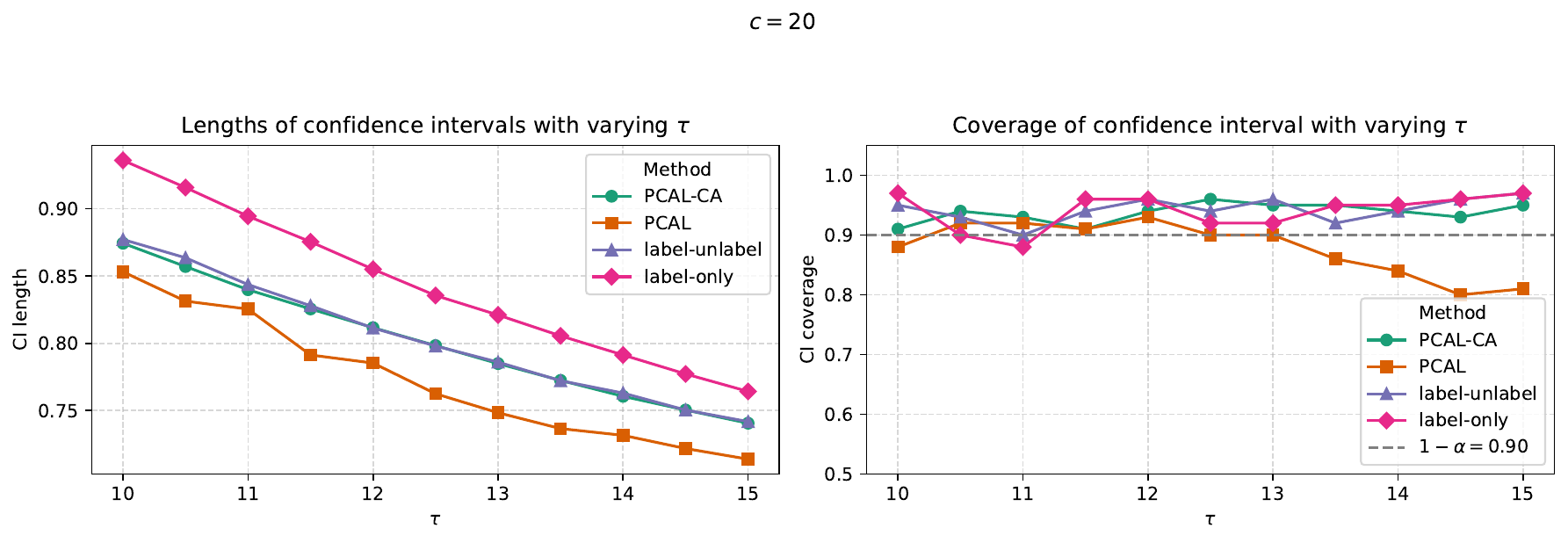}
  \caption{Confidence interval results of feature 3 with $c=20$.}
  \label{fig:politeness_c20_appendix}
\end{figure}

\begin{figure}[!ht]
  \centering
  \includegraphics[width=\linewidth]{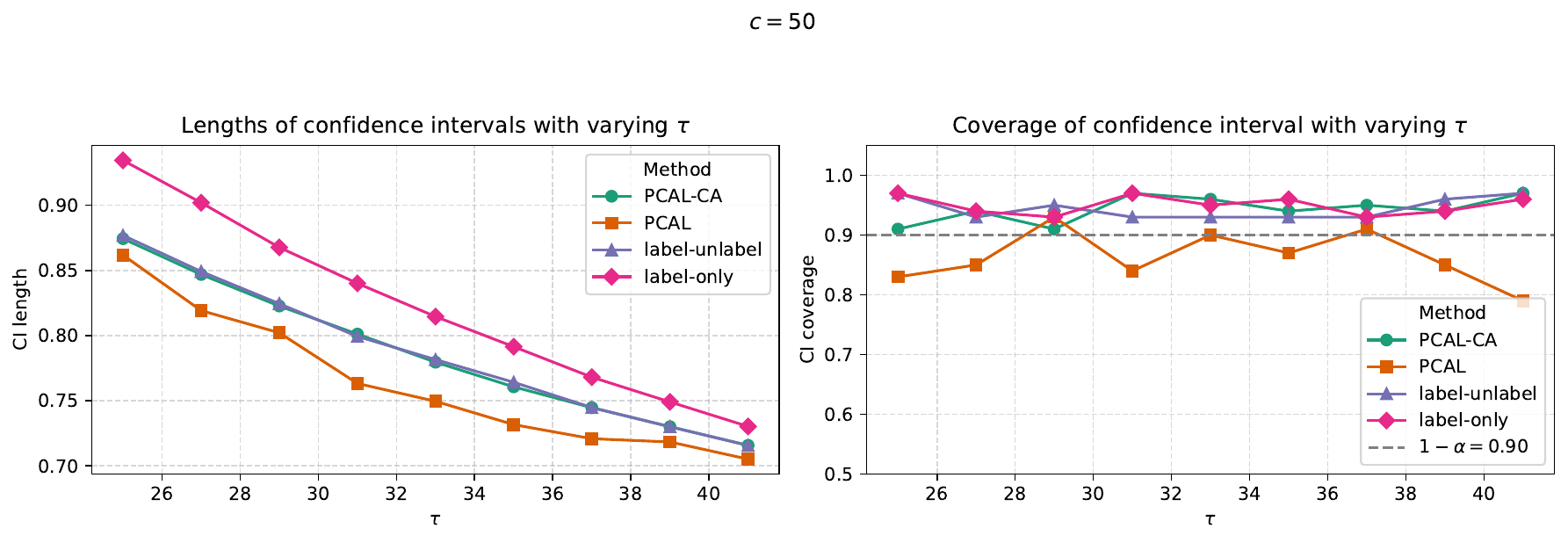}
  \caption{Confidence interval results of feature 3 with $c=50$.}
  \label{fig:politeness_c50_appendix}
\end{figure}

\begin{figure}[!ht]
  \centering
  \includegraphics[width=\linewidth]{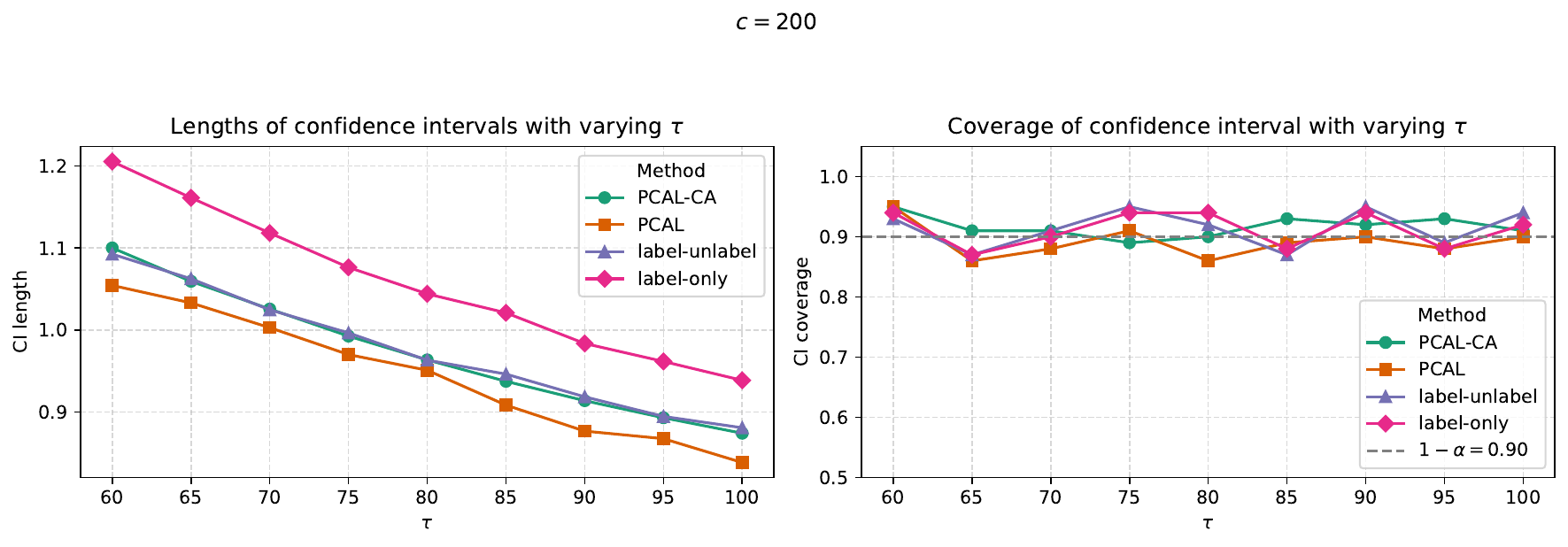}
  \caption{Confidence interval results of feature 3 with $c=200$.}
  \label{fig:politeness_c200_appendix}
\end{figure}
\end{document}